\documentclass[twoside,fleqn]{article}
\usepackage{ecj,palatino,epsfig,latexsym,natbib}
\usepackage{ifthen}
\usepackage{xcolor}
\usepackage[hidelinks = true]{hyperref}
\usepackage{rotating}
\newcommand{\markupdraft}[2]{% {#1: {color|display} command}{#2: desired color or text}
    \ifthenelse{\equal{#1}{display}}{#2}{}%                 % display only in draft version
    \ifthenelse{\equal{#1}{color}}{\color{#2}}{}%           % colored only in draft (for \new command)
}
\newcommand{\newcolored}[3][]{{\markupdraft{color}{#2}#3}%  % kept in the final print
    \ifthenelse{\equal{#1}{}}{}{\markupdraft{display}{{\color{yellow!70!black}[#1]}}}} 

\newcommand{\newalg}[1]{\,\text{\colorbox{red!15!white}{\ensuremath{\displaystyle#1}}}\,}

\usepackage{cmap}
\usepackage[utf8x]{inputenc}
\usepackage[T1]{fontenc}
\usepackage[english]{babel}
\usepackage{graphicx}
\graphicspath{{./}}
\usepackage{wrapfig}
\usepackage{xcolor}
\usepackage{algorithm}
\usepackage{algpseudocode}
\usepackage{subcaption}
\usepackage{booktabs}

\usepackage{latexsym}
\usepackage{amsmath,amssymb} % AMS
\usepackage{mathtools}       % many extension of amsmath
\usepackage{stmaryrd}        % \llbracket, \rrbracket
\usepackage{bm}              % bold font
\usepackage{algpseudocode}
\usepackage{algorithm}

\usepackage{amsthm}  
\newtheorem{theorem}{Theorem}[section]

\theoremstyle{definition}

\DeclareMathOperator{\Tr}{Tr}     % trace
\DeclareMathOperator{\Cond}{Cond} % condition number
\DeclareMathOperator{\diag}{diag} % diagonal
\providecommand{\R}{\mathbb{R}} % real field
\providecommand{\E}{\mathbb{E}} % expectation
\providecommand{\T}{\mathsf{T}} % transpose
\renewcommand{\geq}{\geqslant} % geq
\renewcommand{\leq}{\leqslant} % leq
\newcommand{\ve}[1]{\ensuremath{\mathchoice%
  {\mbox{\boldmath$\displaystyle#1$}}
  {\mbox{\boldmath$\textstyle#1$}} {\mbox{\boldmath$\scriptstyle#1$}}
  {\mbox{\boldmath$\scriptscriptstyle#1$}}}}

\DeclarePairedDelimiterX{\inner}[2]{\langle}{\rangle}{#1, #2}
\DeclarePairedDelimiter{\norm}{\lVert}{\rVert}
\DeclarePairedDelimiter{\abs}{\lvert}{\rvert}

\providecommand{\m}{\ensuremath{\ve{m}}}

\providecommand{\C}{\ensuremath{\mathbf{C}}}

\providecommand{\Ccorr}{\ensuremath{\mathrm{corr}(\C)}}
\providecommand{\Ccorr}{\ensuremath{\C_{\mathrm{corr}}}}
\providecommand{\Lam}{\ensuremath{\mathbf{\Lambda}}}
\providecommand{\pc}{\ensuremath{\ve{p}_{c}}}
\providecommand{\ps}{\ensuremath{\ve{p}_{\sigma}}}
\providecommand{\cm}{\ensuremath{c_m}}
\providecommand{\cone}{\ensuremath{c_1}}
\providecommand{\cmu}{\ensuremath{c_\mu}}
\providecommand{\cc}{\ensuremath{c_c}}
\providecommand{\cs}{\ensuremath{c_\sigma}}
\providecommand{\ds}{\ensuremath{d_\sigma}}

\providecommand{\mueff}{\ensuremath{\mu_{w}}}

\providecommand{\hsig}{h_\sigma}
\providecommand{\gs}{\gamma_\sigma}
\providecommand{\gc}{\gamma_c}
\newcommand{\bw}{w}

\providecommand{\NN}{\mathcal{N}}
\providecommand{\xx}{\ve{x}}
\providecommand{\yy}{\ve{y}}
\providecommand{\zz}{\ve{z}}

\providecommand{\D}{\ensuremath{\mathbf{D}}}
\providecommand{\eye}{\ensuremath{\mathbf{I}}}
\newcommand{\teig}{\ensuremath{t_\text{eig}}}

\newcommand{\DD}{dd-}
\newcommand{\DDCMA}{\DD{}CMA}

\begin{document}

\newcommand{\mytitle}{Diagonal Acceleration for Covariance Matrix Adaptation Evolution Strategies}
%\ecjHeader{x}{x}{xxx-xxx}{2019}{\mytitle}{Y. Akimoto and N. Hansen}
\title{\bf \mytitle}  

\author{\name{\bf Y. Akimoto} \hfill \addr{akimoto@cs.tsukuba.ac.jp}\\ 
        \addr{University of Tsukuba, 1-1-1 Tennodai, Tsukuba, JAPAN}
\AND
       \name{\bf N. Hansen} \hfill \addr{nikolaus.hansen@inria.fr}\\
        \addr{Inria, RandOpt Team, CMAP, Ecole Polytechnique, Palaiseau, FRANCE}
}

\maketitle

%%%%%%%%%%%%%%%%%%%%%%%%%%%
%\newpage
\begin{abstract}%
We introduce an acceleration for covariance matrix adaptation evolution strategies (CMA-ES) by means of \emph{adaptive diagonal decoding} (\DDCMA). This diagonal acceleration endows the default CMA-ES with the advantages of separable CMA-ES without inheriting its drawbacks.
Technically, we introduce a diagonal matrix $D$ that expresses coordinate-wise variances of the sampling distribution in $DCD$ form.
The diagonal matrix can learn a rescaling of the problem in the coordinates within linear number of function evaluations.
Diagonal decoding can also exploit separability of the problem, but, crucially, does not compromise the performance on non-separable problems.
The latter is accomplished by modulating the learning rate for the diagonal matrix based on the condition number of the underlying correlation matrix.
\DDCMA-ES not only combines the advantages of default and separable CMA-ES, but may achieve overadditive speedup: it improves the performance, and even the scaling, of the better of default and separable CMA-ES on classes of non-separable test functions that reflect, arguably, a landscape feature commonly observed in practice.

The paper makes two further secondary contributions:
we introduce two different approaches to guarantee positive definiteness of the covariance matrix with active CMA, which is valuable in particular with large population size;
we revise the default parameter setting in CMA-ES, proposing accelerated settings in particular for large dimension.

All our contributions can be viewed as independent improvements of CMA-ES, yet they are also complementary and can be seamlessly combined. In numerical experiments with \DDCMA-ES up to dimension 5120, we observe remarkable improvements over the original covariance matrix adaptation on functions with coordinate-wise ill-conditioning. The improvement is observed also for large population sizes up to about dimension squared.
\end{abstract}


%\tableofcontents

\section{Introduction}

In real world applications of continuous optimization involving simulations such as physics or chemical simulations, the input-output relation between a candidate solution and its objective function value is barely expressible in explicit mathematical formula. The objective function value is computed through a complex simulation with a candidate solution as an input. In such scenarios, we gain the information of the problem only through the evaluation of the objective function value of a given candidate solution. A continuous optimization of $f:\R^{n} \to \R$ is referred to as black-box continuous optimization if we gain the information of the problem only through the evaluation $\xx \mapsto f(\xx)$ of a given candidate solution $\xx \in \R^{n}$.

Black-box continuous optimization arises widely in real world applications such as model parameter calibration and design of robot controller. It often involves computationally expensive simulation to evaluate the quality of candidate solutions. The search cost of black-box continuous optimization is therefore the number of simulations, i.e., the number of objective function calls, and a search algorithm is desired to locate good quality solutions with as few $f$-calls as possible. Practitioners need to choose one or a few search algorithms to solve their problems and tune their hyper-parameters based on the prior knowledge into their problems. However, prior knowledge is often limited in the black-box situation due to the black-box relation between $\xx$ and $f(\xx)$. Hence, algorithm selection, as well as hyper-parameter tuning, is a tedious task for practitioners who are typically not experts in search algorithms.

Covariance Matrix Adaptation Evolution Strategy (CMA-ES), developed by \citet{Hansen2001ec,Hansen2003ec,Hansen2004ppsn,Jastrebski2006cec}, is recognized as a state-of-the-art derivative-free search algorithm for \emph{difficult} continuous optimization problems \citep{rios2013derivative}. CMA-ES is a stochastic and comparison-based search algorithm that maintains the multivariate normal distribution as a sampling distribution of candidate solutions. The distribution parameters such as the mean vector and the covariance matrix are updated at each iteration based on the candidate solutions and their objective value ranking, so that the sampling distribution will produce promising candidate solutions more likely in the next algorithmic iteration. The update of the distribution parameters is partially found as the natural gradient ascent on the manifold of the distribution parameter space equipped with the Fisher metric \citep{Akimoto2010ppsn,Glasmachers2010gecco,Ollivier2017jmlr},
thereby revealing the connection to natural evolution strategies \citet{Wierstra2008cec,Sun2009gecco,Glasmachers2010gecco,wierstra2014jmlr}, whose parameter update is derived explicitly from the natural gradient principle.

Invariance \citep{Hansen2000ppsn,Hansen2011impacts} is one of the governing principles of the design of CMA-ES and the essence of its success on \emph{difficult} continuous optimization problems consisting of ruggedness, ill-conditioning, and non-separability. CMA-ES exhibits several invariance properties such as invariance to order preserving transformation of the objective function, invariance to translation, rotation and coordinate-wise scaling of the search space \citep{Hansen:2013vt}. Invariance guarantees the algorithm to work identically on an original problem and its transformed version.
Thanks to its invariance properties, CMA-ES works, after an adaptation phase, equally well on separable and well-conditioned functions, which are easy for most search algorithms, and on non-separable and ill-conditioned functions produced by an affine coordinate transformation of the former, which are considered difficult for many other search algorithms. This also contributes to allow default hyper-parameter values to depend solely on the search space dimension and the population size, whereas many other search algorithms need tuning depending on problem difficulties to make the algorithms efficient \citep{Karafotias2015}.

On the other hand, exploiting problem structure, if possible, is beneficial for optimization speed. Different variants of CMA-ES have been proposed to exploit problem structure such as separability and limited variable dependency. They aim to achieve a better scaling with the dimension \citep{Knight:2007,Ros2008ppsn,Akimoto2014a,akimoto2016gecco,Loshchilov:2017}. However, they lose some of the invariance properties that CMA-ES has and compromise the performance on problems where their specific, more or less restrictive assumptions on the problem structure do not hold. For instance, the separable CMA-ES \citep{Ros2008ppsn} reduces the degrees of freedom of the covariance matrix from $(n^2 + n) / 2$ to $n$ by adapting only the diagonal elements of the covariance matrix. It scales better in terms of internal time and space complexity and in terms of number of function evaluations to adapt the coordinate-wise variance of the search distribution.
Good results are hence observed on functions with weak variable dependencies. However, unsurprisingly, the convergence speed of the separable CMA is significantly deteriorated on non-separable and ill-conditioned functions, where the shape of the level sets of the objective function can not be reasonably well approximated by the equi-probability ellipsoid defined by the normal distribution with diagonal (co)variance matrix.

In this paper, we aim to improve the performance of CMA-ES on a relatively wide class of problems
by exploiting problem structure, however crucially, without compromising the performance on more difficult problems without this structure.\footnote{%
Any covariance matrix, $\bm{\Sigma}$, can be uniquely decomposed into $\bm{\Sigma}=\D\C\D$, where $\D$ is a diagonal matrix and $\C$ is a correlation matrix.
The addressed problem class can be characterized in that for the best problem approximation $\bm{\Sigma}=\D\C\D$ \emph{both matrices, \C\ and $\D$, have non-negligible condition number, say no less than $100$.}
}%
% \del{On a convex quadratic function, once the covariance matrix becomes proportional to the inverse Hessian of the objective function, we observe linear convergence with the rate of convergence equal to the one for the sphere function, $f(x) = \norm{x}^2$, and the convergence speed of the CMA-ES is controlled mainly by the step-size adaptation mechanism. By accelerating CMA we try to shorten the adaptation time of CMA-ES to reach convergence phase.}{}

The first mechanism we are concerned with is the so-called active covariance matrix update, which was originally proposed for the $(\mu, \lambda)$-CMA-ES with intermediate recombination \citep{Jastrebski2006cec}, and later incorporated into the $(1+1)$-CMA-ES \citep{Arnold2010gecco}.
It utilizes unpromising solutions to actively decrease the eigenvalues of the covariance matrix.
The active update consistently improves the adaptation speed of the covariance matrix in particular on functions where a low dimensional subspace dominates the function value. The positive definiteness of the covariance matrix is however not guaranteed when the active update is utilized. Practically, a small enough learning rate of the covariance matrix is sufficient to keep the covariance matrix positive definite with overwhelming probability, however, we would like to increase the learning rate when the population size is large. We propose two novel schemes that guarantee the positive definiteness of the covariance matrix, so that we take advantage of the active update even when a large population size is desired, e.g., when the objective function evaluations are distributed on many CPUs or when the objective function is rugged.

% \new{%TPA part
%   The second mechanism is the two point step-size adaptation (TPA) \citep{Hansen2008tps,hansen2014assess}. TPA has advantages over the default step-size adaptation mechanism for the CMA-ES, namely the cumulative step-size adaptation (CSA), especially in high dimension. It has been already employed as the primary step-size adaptation in a variant of the CMA-ES for high dimensional problems, namely VkD-CMA \citep{akimoto2016gecco,akimoto2016ppsn}. Although a combination of TPA with CMA-ES has been benchmarked in \citet{Atamna:2015:BII:2739482.2768467}, the algorithm description there was incomplete. In this paper, we provide a complete formulation of TPA when combined with covariance matrix adaptation, and propose a novel default parameter setting. We review the situations where CSA and TPA perform differently. 
% }%
% \del{The second mechanism we put forward is two point step-size adaptation (TPA), which is an alternative mechanism to the default step-size adaptation in CMA-ES, the cumulative step-size adaptation (CSA).
% \cite{Krause2017foga} and our own experience over the past years suggest that TPA is the most liable and competitive alternative to CSA. We propose a simplified parameter setting and we show cases where TPA and CSA perform remarkably different.}{}%
%
The main contribution of this paper is
the diagonal acceleration of CMA by means of \emph{adaptive diagonal decoding}, referred to as \DDCMA.
We introduce a coordinate-wise variance matrix, $\D^2$, of the sampling distribution alongside the positive definite symmetric matrix $\C$, such that the resulting covariance matrix of the sampling distribution is represented by $\D \C \D$. We call $\D$ the diagonal decoding matrix. We update $\C$ with the original CMA, whereas $\D$ is updated similarly to separable CMA. An important point is that we want to update $\D$ faster than $\C$, by setting higher learning rates for the $\D$ update.
However, when $\C$ contains non-zero covariances, the update of $\D$ can result in a drastic change of the sampling distribution and disturb the adaptation of $\C$.
To resolve this issue, we introduce an adaptive damping mechanism for the $\D$ update, so that the difference (e.g., Kullback-Leibler divergence) between the distributions before and after the update remains sufficiently small. With this damping, $\D$ is updated as fast as by separable CMA on a separable function if the correlation matrix of the sampling distribution is close to the identity, and it suppresses the $\D$ update when the correlation matrix is away from the identity, i.e., variable dependencies have been learned.

The update of $\D$ breaks the rotation invariance of the original CMA, hence we lose a mathematical guarantee of the original CMA that it performs identical on functions in an equivalence group defined by this invariance property.
The ultimate aim of this work is to gain significant speed-up in some situations and to preserve the performance of the original CMA in the worst case.
Functions where we expect that
the diagonally accelerated CMA outperforms the original one have variables with different sensitivities, i.e., coordinate-wise ill-conditioning. Such functions may often appear in practice, since variables in a black-box continuous optimization problem can have distinct meanings. Diagonal acceleration however can even be superior to the optimal additive portfolio of the original CMA and the separable CMA. We demonstrate in numerical experiments that \DDCMA{} outperforms the original CMA not only on separable functions but also on non-separable ill-conditioned functions with coordinate-wise ill-conditioning that separable CMA can not efficiently solve.

The last contribution is a set of improved and simplified default parameter settings for the covariance matrix adaptation and for the cumulation factor for the so-called evolution path. These learning rates, whose default values have been previously expressed with independent formulas, are reformulated so that their dependencies are clearer.
The new default learning rates also improve the adaptation speed of the covariance matrix on high dimensional problems without compromising stability.

The rest of this paper is organized as follows. We introduce the original and separable CMA-ES in Section~\ref{sec:cmaes}. The active update of the covariance matrix with positive definiteness guarantee is proposed in Section~\ref{sec:active}. The adaptive diagonal decoding mechanism is introduced in Section~\ref{sec:add}. Section~\ref{sec:param} is devoted to explain the renewed default hyper-parameter values for the covariance matrix adaptation and provides an algorithm summary of \DDCMA-ES and a link to publicly available Python code. Numerical experiments are conducted in Section~\ref{sec:exp} to see how effective each component of CMA with diagonal acceleration works in different situations. We conclude the paper in Section~\ref{sec:conc}.

\section{Evolution Strategy with Covariance Matrix Adaptation}\label{sec:cmaes}

We summon up the ($\mueff$, $\lambda$)-CMA-ES consisting of weighted recombination, cumulative step-size adaptation, and rank-one and rank-$\mu$ covariance matrix adaptation.

%\del{While the cumulative step-size adaptation is considered in this paper, other step-size adaptation mechanisms such as two-point step-size adaptation \citep{Hansen2008tps,hansen2014assess} can be used. }{}

The CMA-ES maintains the multivariate normal distribution, $\mathcal{N}(\m, \sigma^2 \D\C\D)$, where $\m \in \R^n$ is the mean vector that represents the center of the search distribution, $\sigma \in \R_{+}$ is the so-called step-size that represents the scaling factor of the distribution spread, and $\C\in \R^{n\times n} $ is a positive definite symmetric matrix that represents the shape of the distribution ellipsoid. Though the covariance matrix of the sampling distribution is $\sigma^2 \D \C \D$, we often call $\C$ the covariance matrix in the context of CMA-ES. In this paper, we apply this terminology, and we will call $\sigma^2\D\C\D$ the covariance matrix of the sampling distribution to distinguish them when necessary. The positive definite diagonal matrix  $\D \in \R^{n\times n}$ is regarded as a diagonal decoding matrix, which represents the scaling factor of each design variable. It is fixed during the optimization, usually $\D = \eye$, and does not appear in the standard terminology. However, it plays an important role in this paper, and we define the CMA-ES with $\D$ for the later modification.

\subsection{Original CMA-ES}\label{sec:cma}

The CMA-ES repeats the following steps until it meets one or more termination criteria:
\begin{enumerate}%\setlength{\itemsep}{0.0em}
\item Sample $\lambda$ candidate solutions, $\xx_{i}$, independently from $\mathcal{N}(\m, \sigma^2 \D\C\D)$; 
\item Evaluate candidate solutions on the objective, $f(\xx_{i})$, and sort them in ascending order\footnote{Ties, $f(\xx_{i:\lambda}) = \cdots = f(\xx_{i+k-1:\lambda})$, are treated by redistributing the averaged recombination weights $\sum_{l=0}^{k-1} w_{i+l} / k$ to tied solutions $\xx_{i:\lambda}, \dots, \xx_{i+k-1:\lambda}$.}, $f(\xx_{1:\lambda}) \le \cdots \le f(\xx_{\lambda:\lambda})$, where $i:\lambda$ is the index of the $i$-th best candidate solution among $\xx_1, \dots, \xx_\lambda$; 
\item Update the distribution parameters, $\m$, $\sigma$, and $\C$.
\end{enumerate}

\paragraph{Sampling and Recombination} To generate candidate solutions, we compute the (unique, symmetric) square root of the covariance matrix $\sqrt{\C}$ obeying $\C^{(t)} = \big(\sqrt{\C}\big)^2$. The candidate solutions are the affine transformation of independent and standard normally distributed random vectors,
\begin{equation}
 \begin{split}
   \zz_{i} &\sim \mathcal{N}(\bm{0}, \eye) \enspace, \\
   \yy_{i} &= \sqrt{\C} \zz_{i} \sim \mathcal{N}(\bm{0}, \C) \enspace,\\
   \xx_i &= \m^{(t)} + \sigma^{(t)} \D \yy_{i}  \sim \mathcal{N}(\m^{(t)}, (\sigma^{(t)})^2 \D\C\D) \enspace.
 \end{split}
 \label{eq:sampling}
 \end{equation}
 The default population size is $\lambda = 4 + \lfloor 3 \ln n \rfloor$. To reduce the time complexity per $f$-call without compromising the performance, we compute the matrix decomposition $\C^{(t)} = \big(\sqrt{\C}\big)^2$ every $t_\text{eig} = \max\big(1, \big\lfloor (\beta_\mathrm{eig}  (\cone + \cmu) )^{-1} \big\rfloor\big)$ iteration, where $\cone$ and $\cmu$ are the learning rates for the covariance matrix adaptation that appear later, and $\beta_\mathrm{eig} = 10 n$. If the learning rates are small enough ($\cone + \cmu \leq (2 \beta_\mathrm{eig})^{-1}$), the covariance matrix is regarded as insignificantly changing in each iteration and we stall the decomposition. 

The mean vector $\m$ is updated by taking the weighted average of the promising candidate directions,
 \begin{equation}
   \m^{(t+1)} = \m^{(t)} + \cm \sum_{i=1}^{\mu} w_i (\xx_{i:\lambda} - \m) \enspace, \label{eq:m}
 \end{equation}
 where $\cm$ is the learning rate for the mean vector update, usually $\cm = 1$.
 The number of promising candidate solutions are denoted by $\mu$, and $\big(w_i\big)_{i=1}^{\mu}$ are recombination weights satisfying $w_i > 0$ for $i \leq \mu$. A standard choice is $w_i \propto \ln\frac{\lambda+1}{2} - \ln i$ for $i = 1, \dots, \mu$ and $\mu = \lfloor \lambda / 2 \rfloor$
 and $\sum_{i=1}^\mu w_i = 1$.

\paragraph{Step-Size Adaptation} The cumulative step-size adaptation (CSA)
updates the step-size $\sigma$ based on the length of the evolution path that accumulates the mean shift normalized by the current distribution covariance, i.e.,%
\footnote{When $\D$ is not the identity, \eqref{eq:ps} is not exactly equivalent to the original CSA \citep{Hansen2001ec}: $\zz_{i:\lambda} = (\D\sqrt{\C})^{-1}(\xx_{i:\lambda} - \m) / \sigma$ in this paper whereas originally $\zz_{i:\lambda} = \sqrt{\D\C\D}^{-1}(\xx_{i:\lambda} - \m) / \sigma$. This difference results in rotating the second term of the RHS of \eqref{eq:ps} at each iteration with a different orthogonal matrix, and ends up in a different $\norm{\ps}$. \citet{Krause2016nips} have theoretically studied the effect of this difference and argued that this systematic difference becomes small if the parameterization of the covariance matrix of the sampling distribution is unique and it converges up to scale. If $\D$ is fixed, the parameterization is unique. Later in this paper, we update both $\D$ and $\C$ but we force the parameterization to be unique by \eqref{eq:dc1} and \eqref{eq:dc2}. Hence the systematic difference is expected to be small. See \citet{Krause2016nips} for details.}
 \begin{align}
   \ps^{(t+1)} &= (1 - \cs) \ps^{(t)} + \sqrt{\cs (2 - \cs) \mueff} \sum_{i=1}^{\mu} w_i \zz_{i:\lambda} \enspace,\label{eq:ps}\\
   \gs^{(t+1)} &= (1 - \cs)^2 \gs^{(t)} + \cs (2 - \cs)  \enspace,\label{eq:gs}
   \end{align}
where $\mueff = 1/ \sum_{i=1}^{\mu} w_i^2$ is the so-called variance effective selection mass, and $\cs$ is the inverse time horizon for the $\ps$ update, aka cumulation factor. The scalar $\gs^{(t+1)}$ is a correction factor for small $t$ and converges to $1$, where $\gs^{(0)} = \|\ps^{(0)}\|^2/n = 0$.\footnote{%
    An elegant alternative to introducing $\gs$ is to use $\cs^{(t)} = \max(\cs, 1/t)$ in place of \cs\ in \eqref{eq:ps}, assuming the first $t=1$.
    This resembles a simple average while $t\le1/\cs$ and only afterwards discounts older information by the original decay of $1-\cs$.
}
The log step-size is updated as
 \begin{equation}
   \ln \sigma^{(t+1)} = \ln \sigma^{(t)} + \frac{\cs}{\ds}\left(\frac{\norm{\ps^{(t+1)}}}{\chi_{n}} - \sqrt{\gs^{(t+1)}}\right)
   \enspace,
   \label{eq:sigma}
 \end{equation}
 where $\ds$ is the damping factor for the $\sigma$ update and $\chi_{n} = \sqrt{n}\big(1 - \frac{1}{4n} + \frac{1}{21n^2}\big)$ is an approximation of the expected value of the norm of an $n$-dimensional standard normally distributed random vector, $\sqrt{2} \Gamma\big(\frac{n+1}{2}\big) / \Gamma\big(\frac{n}{2}\big)$. The default values for $\cs$ and $\ds$ are
 \begin{align}
   \cs &= \frac{\mueff + 2}{n + \mueff + 5} \label{eq:cs}\\
   \ds &= 1 + \cs + 2 \max\left( 0,\ \sqrt{\frac{\mueff - 1}{n + 1}} - 1\right) \enspace. \label{eq:ds} 
 \end{align}
The damping parameter $\ds$ is introduced to stabilize the step-size adaptation when the population size is large 
\citep{Hansen2004ppsn}. When the step-size becomes too small by accident or is initialized so, the norm of the evolution path will become $O\big(\sqrt{\mueff / n}\big)$, which results in a quick increase of $\sigma$ if $\ds=1$ \citep{Akimoto2008gecco}. For large $\mueff$, the chosen damping factor $\ds$ prevents an unreasonable increase of $\sigma$ at the price of a reduced convergence speed. In case of $\mueff \gg n$, the covariance matrix adaptation is the main component decreasing the overall variance of the sampling distribution, while the CSA is still effective to increase $\sigma$ when necessary.

%\niko{remark to self: in this case the exponential update may have a significantly smaller convergence speed due to its unbiasedness, but TPA should fix this in either case.}

\paragraph{Covariance Matrix Adaptation}

 The covariance matrix $\C$ is updated by the following formula that combines rank-one and rank-$\mu$ update
\begin{multline}
  \C^{(t+1)} = \left(1 - \cone\gc^{(t+1)} - \cmu \sum_{i=1}^{\mu} w_{i}\right) \C^{(t)} \\
    + \cone \left(\D^{-1}\pc^{(t+1)}\right) \left(\D^{-1}\pc^{(t+1)}\right)^\T + \cmu\sum_{i=1}^{\mu} w_{i} \yy_{i:\lambda} \yy_{i:\lambda}^\T \enspace,
  \label{eq:cma}
\end{multline}
where $\cone$ and $\cmu$ are the learning rates for the rank-one update (2nd term) and the rank-$\mu$ update (3rd term), respectively, $\pc$ is the evolution path that accumulates the successive mean movements and $\gc$ is the correction factor for the rank-one update, which are updated as
\begin{align}
  \pc^{(t+1)} &= (1 - \cc) \pc^{(t)} + \hsig^{(t+1)} \sqrt{\cc (2 - \cc) \mueff} \sum_{i=1}^{\mu} w_i \D \yy_{i:\lambda} \enspace,\label{eq:pc}\\
  \gc^{(t+1)} &= (1 - \cc)^2 \gc^{(t)} + \hsig^{(t+1)} \cc (2 - \cc)  \enspace,\label{eq:gc}
\end{align}
where $\cc$ is the inverse time horizon for the $\pc$ update. The Heaviside function $\hsig^{(t+1)}$ is introduced to stall the update of $\pc$ if $\norm{\ps}$ is large, i.e., when the step-size is rapidly increasing. It is defined as
\begin{equation}
  \hsig^{(t+1)} = \begin{cases} 1 & \text{if } \norm{\ps^{(t+1)}}^2 / \gs^{(t+1)} < \big(2 + \frac{4}{n+1}\big) n \\
    0 & \text{otherwise}
  \end{cases}.
  \label{eq:hs}
\end{equation}
The default parameters for $\cone$, $\cmu$, and $\cc$ are smaller than one and presented later. 

\subsection{Separable Covariance Matrix Adaptation}\label{sec:sep}

The separable covariance matrix adaptation (sep-CMA, \citet{Ros2008ppsn}) adapts only the coordinate-wise variance of the sampling distribution, i.e., the diagonal elements of the covariance matrix in the same way as \eqref{eq:cma}, but with larger learning rates.
In our notation scheme, we keep \C\ to be the identity and describe sep-CMA by updating \D.
The update of coordinate $k$ follows
\begin{multline}
  [\D^{(t+1)}]_{k,k}^2 = [\D^{(t)}]_{k,k}^2\Bigg( 1 - \cone\gc^{(t+1)} - \cmu \sum_{i=1}^{\mu} w_{i} \\ + \cone [\pc^{(t+1)}]_{k}^2/[\D^{(t)}]_{k,k}^2  + \cmu \sum_{i=1}^{\mu} w_{i} \left[\zz_{i:\lambda}\right]_{k}^2\Bigg) \enspace.
  \label{eq:sep}
\end{multline}
The learning rates $\cone$ and $\cmu$ are set differently from those used for the $\C$ update.

One advantage of the separable CMA is that all operations can be done in linear time per $f$-call. Therefore, it is promising if $f$-calls are cheap. The other advantage is that one can set the learning rate greater than those used for the $\C$ update, since the degrees of freedom of the covariance matrix of the sampling distribution is $n$, rather than $n(n+1)/2$. The adaptation speed of the covariance matrix is faster than for the original CMA. However, if the problem is non-separable and has strong variable dependencies, adapting the coordinate-wise scaling is not enough to make the search efficient. More concisely, if the inverse Hessian of the objective function, $\text{Hess}(f)^{-1}$, is not well approximated by a diagonal matrix, the convergence speed will be $O(1/\Cond(\D^2 \text{Hess}(f)))$, which is empirically observed on convex quadratic functions as well as theoretically deduced in \citet{AAH2017TCS}.
In practice, it is rarely known in advance whether the separable CMA is appropriate or not. \citet{Ros2008ppsn} propose to use the separable CMA for hundred times dimension function evaluations and then switch to the original CMA afterwards. Such an algorithm has been benchmarked in \citet{Ros:2009:BSB}, where the first $1 + 100 n /\sqrt{\lambda}$ iterations are used for the separable CMA.

\section{Active covariance matrix update with guarantee of positive definiteness}\label{sec:active}

\begin{figure}[t]
  \centering
  \begin{subfigure}{0.33\textwidth}%
    \centering%
    \includegraphics[trim={2cm 2cm 1.5cm 1.5cm},clip,width=\hsize]{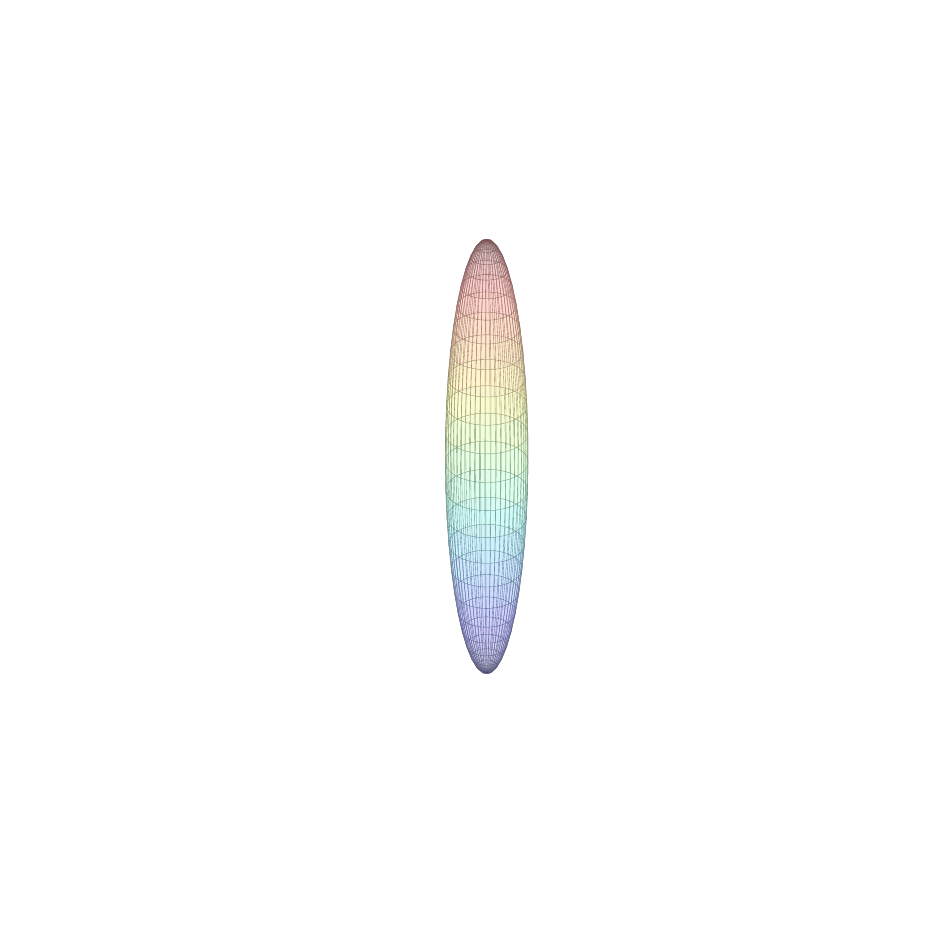}%
    \caption{Cigar type}%
    \label{fig:cigar}
  \end{subfigure}%
  \begin{subfigure}{0.33\textwidth}%
    \centering%
    \includegraphics[trim={2cm 2cm 1.5cm 1.5cm},clip,width=\hsize]{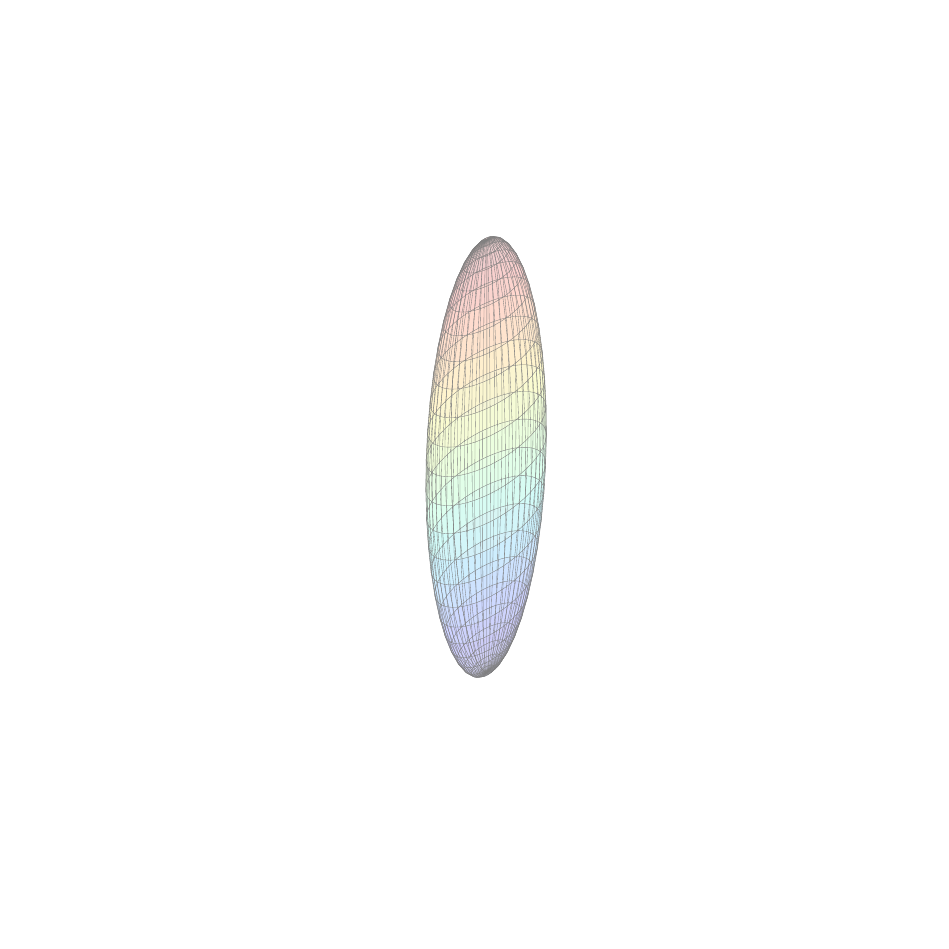}% 
    \caption{Ellipsoid type}%
        \label{fig:ellipsoid}
  \end{subfigure}%
  \begin{subfigure}{0.33\textwidth}%
    \centering%
    \includegraphics[trim={2cm 2cm 1.5cm 1.5cm},clip,width=\hsize]{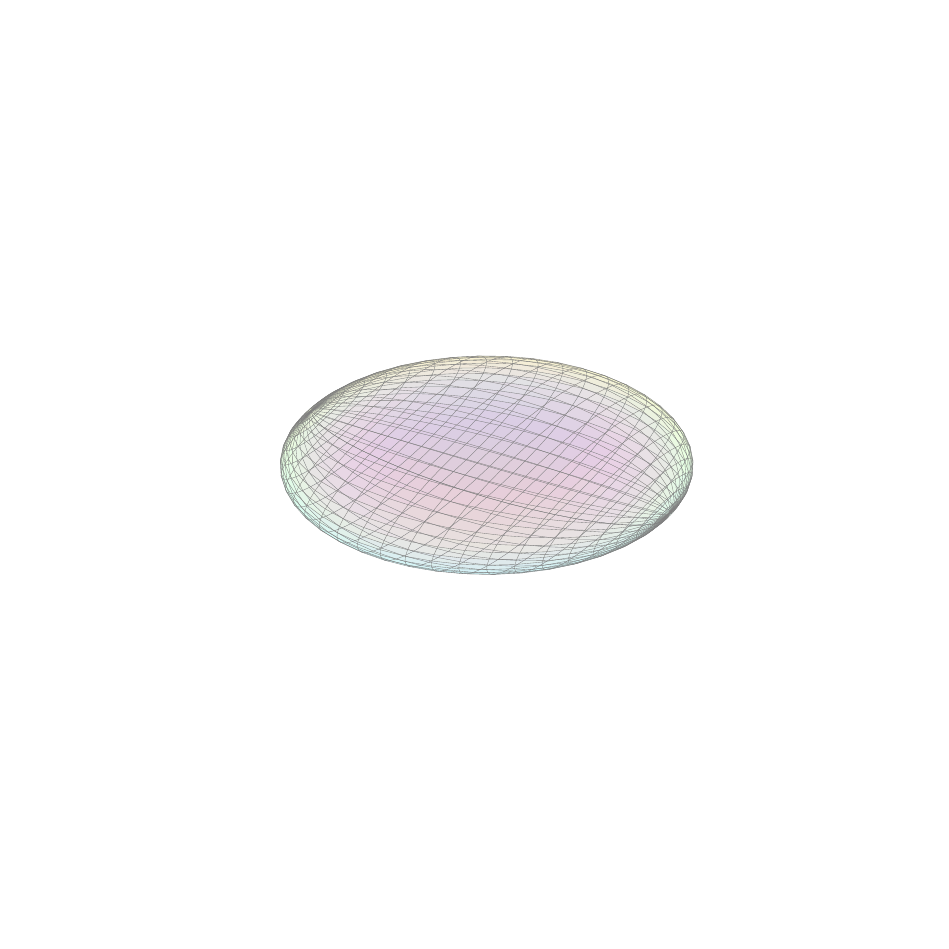}%
    \caption{Discus type}%
    \label{fig:discus}    
  \end{subfigure}%
  \caption{Levelset of three convex quadratic functions $\xx^\T \mathbf{H} \xx$ with Hessian $\mathbf{H}$ of condition number $30$. Cigar type: $\mathbf{H} = \diag(30, 30, 1)$, Ellipsoid type: $\mathbf{H} = \diag(1, \sqrt{30}, 30)$, Discus type: $\mathbf{H} = \diag(1, 1, 30)$.}
  \label{fig:acma}
\end{figure}

Active covariance matrix adaptation (referred to as Active-CMA, \citealp{Jastrebski2006cec}) utilizes the information of unpromising solutions, $\xx_{i:\lambda}$ for $i > \mu$, to update the covariance matrix. The modification is rather simple. We prepare $\lambda$ recombination weights $(\bw_i)_{i=1}^{\lambda}$ for the active covariance matrix update, where $\bw_i$ are not anymore restricted to be positive. For example, the recombination weights used in \citet{Jastrebski2006cec} are
\begin{equation}
  \bw_i = \begin{cases}
    1/\mu & (i \leq \mu) \\
    0 & (\mu < i \leq \lambda - \mu) \\
    - 1/\mu & (\lambda - \mu < i) \\
  \end{cases}
  \label{eq:activemu}
\end{equation}
The update formula \eqref{eq:cma} is then replaced with
\begin{multline}
  \C^{(t+1)} = \left(1 - \cone\gc^{(t+1)} - \cmu \sum_{i=1}^{\newalg{\lambda}} \bw_i\right) \C^{(t)} \\
    + \cone \left(\D^{-1}\pc^{(t+1)}\right) \left(\D^{-1}\pc^{(t+1)}\right)^\T + \cmu \sum_{i=1}^{\newalg{\lambda}} \bw_{i} \yy_{i:\lambda} \yy_{i:\lambda}^\T \enspace,
  \label{eq:acma}
\end{multline}
where shaded areas depict the difference to \eqref{eq:cma}.
The only difference is that we use all $\lambda$ candidate solutions to update the covariance matrix with positive and negative recombination weights.%
\footnote{%
As of 2018, many implementations of CMA-ES feature the active update of \C\ and it should be considered as default.
}

Each component of the covariance matrix adaptation, rank-one update, rank-$\mu$ update, and active update, produces complementary effects. The rank-one update of the covariance matrix (the second term on the RHS of \eqref{eq:acma}, \citealp{Hansen2001ec}) accumulates the successive steps of the distribution mean and increases the eigenvalues of the covariance matrix in the direction of the successive movements. It excels at learning one long axis of the covariance matrix, and is effective on functions with a subspace of a relatively small dimension where the function value is less sensitive than in its orthogonal subspace.
Figure~\ref{fig:cigar} visualizes an example of such function. See also Figure~7 in \citet{Hansen:2013vt} for numerical results. On the other hand, since the update is always of rank one, the learning rate $\cone$ needs to be sufficiently small to keep the covariance matrix regular and stable.

The rank-$\mu$ update (the third term on the RHS of \eqref{eq:acma} with positive $w_i$, \citealp{Hansen2003ec}) utilizes the information of $\mu$ successful candidate solutions in a different way than the rank-one update. It computes the empirical covariance matrix of successful mutation steps $\yy_i$. The update matrix is with probability one of rank $\min(n, \mu)$, allowing to set a relatively large learning rate $\cmu$. It reduces the number of iterations to adapt the covariance matrix when $\lambda \gg 1$.

Both, rank-one and rank-$\mu$ update, try to increase the variances in the subspace of successful mutation steps.
The eigenvalues corresponding to unsuccessful mutation steps only passively fade out. The active update actively decreases such eigenvalues. It consistently accelerates covariance matrix adaptation, and the improvement is particularly pronounced on functions with a small number of dominating eigenvalues of the Hessian of the objective function. Figure~\ref{fig:discus} is an example of such function.

A disadvantage of the active update with negative recombination weights such as \eqref{eq:activemu} is to have no guarantee of the positive definiteness of the covariance matrix anymore.
It is easy to see that the rank-one and rank-$\mu$ update of CMA in \eqref{eq:cma} guarantee that the minimal eigenvalue of $\C^{(t+1)}$ is no smaller than the minimal eigenvalue of $\C^{(t)}$ times $1 - \cone - \cmu\sum_{i=1}^{\mu}w_i$, since the second and third terms only increase the eigenvalues. However, the introduction of negative recombination weights can violate the positive definiteness because the third term may decrease the minimum eigenvalue arbitrarily.
\citet{Jastrebski2006cec} set a sufficiently small learning rate for the active update, i.e., the absolute values of the negative recombination weights sum up to a smaller value than one. It will prevent the covariance matrix from being non-positive with high probability, but it does not \emph{guarantee} positive definiteness. Moreover, it becomes ineffective when the population size is relatively large and a greater learning rate is desired. 

\citet{Krause2015} apply the active update with positive definiteness guarantee by introducing the exponential covariance matrix update, called xCMA. Instead of updating the covariance matrix in an additive way as in \eqref{eq:acma}, the covariance matrix is updated as
\begin{equation}
  \C^{(t+1)} = \sqrt{\C^{(t)}} \exp\left( \Delta \right) \sqrt{\C^{(t)}}\enspace.
  \label{eq:xcma}
\end{equation}
where $\Delta$ is a symmetric matrix. Since the eigenvalues of the matrix exponential are $e^{\delta_i}$ where $\delta_i$ are the eigenvalues of $\Delta$, the positive definiteness is naturally guaranteed.
\citet{Arnold2010gecco} achieve the positive definiteness guarantee in the $(1+1)$-CMA-ES by rescaling the negative recombination weights depending on the norm of unsuccessful mutation steps $\norm{\zz}$.
In this paper we introduce two strategies that are both considered as generalization of this idea to the $(\mueff, \lambda)$-CMA-ES.\footnote{There are variants of CMA-ES that update a factored matrix $\mathbf{A}$ satisfying $\C = \mathbf{A}\mathbf{A}^\T$ (e.g., eNES by \citealp{Sun2009gecco}). No matter how $\mathbf{A}$ is updated, the positive semi-definiteness of $\C$ is guaranteed since $\mathbf{A}\mathbf{A}^\T$ is always positive semi-definite. However this approach has a potential drawback that a negative update may end up increasing a variance. To see this, consider the case $\mathbf{A} \leftarrow \mathbf{A} (\eye - \eta \bm{e}\bm{e}^\T)$, where $\bm{e}$ is some vector and $\eta > 0$ represents the learning rate times the recombination weight. Then, the covariance matrix follows $\C \leftarrow \mathbf{A} (\eye - \eta \bm{e}\bm{e}^\T)^2  \mathbf{A}^\T$. This update shrinks a variance if $\eta$ is sufficiently small ($\eta < 1/\norm{\bm{e}}^2$), however, it increases the variance if $\eta$ is large ($\eta > 1/\norm{\bm{e}}^2$). Hence, a negative update with a long vector $\bm{e}$ and/or a large $\eta$ will cause an opposite effect. Therefore, the factored matrix update is not a conclusive solution to the positive definiteness issue.}

\subsection{Method 1: Scaling Down the Update Factor}\label{sec:method1}

To guarantee the positive definiteness of the covariance matrix, we rescale the update factor of the covariance matrix so that the changes of the minimum eigenvalue of the covariance matrix is bounded.
We start by introducing the rescaling of unpromising solutions,
\begin{equation}\label{eq:zrescaled}
  \tilde \yy_{i:\lambda} =
  \begin{cases}
    \yy_{i:\lambda} & w_i \geq 0 \\
    \frac{\sqrt n}{\norm{\zz_{i:\lambda}}} \yy_{i:\lambda} & w_i < 0,
  \end{cases}
  \quad \text{and analogously,} \quad
  \tilde \zz_{i:\lambda} =
  \begin{cases}
    \zz_{i:\lambda} & w_i \geq 0 \\
    \frac{\sqrt n}{\norm{\zz_{i:\lambda}}} \zz_{i:\lambda} & w_i < 0.
  \end{cases}
\end{equation}
This rescaling results in projecting unpromising solutions onto the surface of the hyper-ellipsoid defined by $\yy^\T\C^{-1}\yy = n$.
By this projection we achieve three desired effects. First, the update becomes bounded which makes it easier to control positive definiteness. Second, short steps are elongated, enhancing their effect in the update. Third, long steps are shortened, reducing their effect in the update.
The two latter effects counter the correlation between rank and length of steps in unfavorable directions. For any given unfavorable direction, longer steps in this direction are most likely ranked worse than shorter steps.

With these scaled solutions,  the covariance matrix is updated as
\begin{equation}
  \begin{aligned}
  \C^{(t+1)} = \C^{(t)}  
  + \newalg{\alpha^{(t)}}\Bigg(
  &\cone \left[\left(\D^{-1}\pc^{(t+1)}\right) \left(\D^{-1}\pc^{(t+1)}\right)^\T - \gc^{(t+1)} \C^{(t)} \right] \\
  &+  \cmu \sum_{i=1}^{\lambda} \bw_{i} \bigg[\newalg{\tilde\yy_{i:\lambda}\tilde\yy_{i:\lambda}^\T} - \C^{(t)}\bigg] \Bigg)\enspace,
  \end{aligned}
  \label{eq:aacma}
\end{equation}
where shaded areas highlight the difference to \eqref{eq:acma} and $\alpha^{(t)} \leq 1$ is the scaling factor to guarantee the positive definiteness of the covariance matrix.
Note that if $\alpha^{(t)} = 1$, it is equivalent to \eqref{eq:acma} except for the rescaling of the unpromising samples.
Importantly, the rescaling of unpromising solutions does not affect the stationarity of the parameter update, i.e., $\E[\C^{(t+1)} \mid \C^{(t)}] = \C^{(t)}$ under a random function such as $f(\xx) \sim \mathcal{U}(0, 1)$. This is shown as follows. First, given $\pc^{(t)} \sim \mathcal{N}(\bm{0}, \gc^{(t)} \D\C^{(t)}\D)$, it is easy to see that $\pc^{(t+1)} \sim \mathcal{N}(\bm{0}, \gc^{(t+1)} \D\C^{(t)}\D)$. Second, using $\E[\zz\zz^\T / \norm{\zz}^2] = \eye / n$ (Lemma~1 of \citealp{Heijmans1999}) we have $\E\big[\sum_{i=1}^{\lambda} \bw_{i} \tilde \yy_{i:\lambda} \tilde \yy_{i:\lambda}^\T \mid \C^{(t)}\big] = \big(\sum_{i=1}^{\lambda} \bw_i\big)\C^{(t)}$. Finally combining them, we obtain $\E[\C^{(t+1)} \mid \C^{(t)}] = \C^{(t)}$.

The main idea is to scale down, if necessary, the update of the covariance matrix by setting $\alpha^{(t)}<1$ in \eqref{eq:aacma}. To provide a better intuition, we start by considering the case $t_\mathrm{eig} = 1$, i.e., $\C^{(t)} = \sqrt{\C}^2$ for every iteration $t$.
Equation \eqref{eq:aacma} can be written as
\begin{equation}
  \C^{(t+1)} = \sqrt{\C}\Bigg[ \eye
  + \alpha^{(t)}\mathbf{Z}^{(t)} \Bigg]\sqrt{\C} %\enspace.
  \label{eq:aacma2}
\end{equation}
with
\begin{multline}
  \mathbf{Z}^{(t)} = \cone \left[\left(\sqrt{\C}^{-1}\D^{-1}\pc^{(t+1)}\right) \left(\sqrt{\C}^{-1}\D^{-1}\pc^{(t+1)}\right)^\T - \gc^{(t+1)} \eye \right] \\
  +  \cmu \sum_{i=1}^\lambda \bw_{i} \bigg[\tilde\zz_{i:\lambda}\tilde\zz_{i:\lambda}^\T - \eye\bigg] \enspace.
  \label{eq:zmat}
\end{multline}
Then, the scaling down parameter $\alpha^{(t)}$ in \eqref{eq:aacma2} is taken so that $\eye + \alpha^{(t)} \mathbf{Z}^{(t)}$ is positive definite.
Then, with
maximum and minimum eigenvalue of a
matrix $\mathbf{A}$ denoted by $d_1(\mathbf{A})$ and $d_n(\mathbf{A})$, respectively, we have that
\begin{equation}
   d_1(\eye + \alpha^{(t)} \mathbf{Z}^{(t)}) \C  
  \succcurlyeq \sqrt{\C}\big[\eye + \alpha^{(t)} \mathbf{Z}^{(t)} \big]\sqrt{\C}
  \succcurlyeq d_n(\eye + \alpha^{(t)} \mathbf{Z}^{(t)}) \C \enspace,
\end{equation}
where $\mathbf{A} \succcurlyeq \mathbf{B}$ mean that $\mathbf{A} - \mathbf{B}$ is positive semi-definite (i.e., all eigenvalues are non-negative)\footnote{Some references (e.g., \citet{Harville2008book}) use the term ``positive semi-definite matrix'' for a matrix with positive and zero eigenvalues and ``non-negative definite matrix'' is used for matrices with non-negative eigenvalues, whereas in other references these terms are used for matrices with non-negative eigenvalues. In this paper, we apply the latter terminology.} for any two square matrices $\mathbf{A}$ and $\mathbf{B}$. 
Moreover, $d_n(\eye + \alpha^{(t)} \mathbf{Z}^{(t)}) = 1 + \alpha^{(t)} d_n(\mathbf{Z}^{(t)})$. Therefore, if $\alpha^{(t)} < 1 / \abs{d_n(\mathbf{Z}^{(t)})}$, the resulting covariance matrix is guaranteed to be positive definite.%
% \niko{Check again why we do not scale down the negative part only.}\yohe{It is tricky to do. To scale down only the negative parts, one needs to decompose the matrix $Z$ to its positive part and negative part. I don't think it is worth doing.}\niko{I see that it means that we have to keep two separate $Z$-matrices. I am not sure how difficult it is to compute $\alpha$, but it seems that we want $d_n(Z_+ + \alpha Z_- - \cmu\sum_{w_i<0}w_i(1 - \alpha)I ) \ge -3/4$ and we may need to compute $d_n$ for two values of $\alpha$ (instead of once) to solve the equation?}%
% \yohe{We were talking about slightly different approaches. My (above) approach is to keep $Z$ without scaling factor for negative, decompose $Z$ and get the positive part $Z_+$ and negative part $Z_-$, and scale down $Z_-$ only as in your equation. Now I understand your approach. I think computing a sufficient $\alpha$ in your approach is easy if $Z_+$ if full rank (taking the minimum eigenvalue of $Z_+$ and the maximum eigenvalue of $Z_-$ and solving the equation similar to the above should give as $\alpha$). This may be however as small as the one computed in method 2 as it doesn't mix the positive part and the negative part and it results in a loose condition on $\alpha$. I think it is difficult to derive a tighter condition for $\alpha$.}

For the general case ($t_\mathrm{eig} \geq 1$), let $t$ be the iteration at which the last eigen decomposition was performed, i.e., $\C^{(t)} = \sqrt{\C}^2$. For iterations $k \in [t, t+t_\mathrm{eig})$, we store the intermediate update matrix $\mathbf{Z}^{(k)}$. The covariance matrix is updated only when the eigen decomposition is required, i.e., at iteration $t + t_\mathrm{eig}$, as
\begin{equation}
  \C^{(t+t_\mathrm{eig})} = \sqrt{\C}\Bigg[ \eye + \alpha^{(t+t_\mathrm{eig})} \sum_{k=t}^{t+t_\mathrm{eig}-1}\mathbf{Z}^{(k)}\Bigg]\sqrt{\C} \enspace.
  \label{eq:aaacma}
\end{equation}
Analogously to the above argument, the resulting covariance matrix is positive definite if the inside of the brackets is positive definite. We set the scaling down parameter as
\begin{equation}
  \alpha^{(t + t_\mathrm{eig})} = \min\left( 
    \frac{0.75}{\abs*{d_n\big(\sum_{k=t}^{t+t_\mathrm{eig}-1}\mathbf{Z}^{(k)}\big)}},\
    1
  \right) \enspace.
  \label{eq:alpha1}
\end{equation}
The first argument guarantees that the minimum eigenvalue of $\C^{(t + t_\mathrm{eig})}$ is greater than or equal to $1/4$th of the minimum eigenvalue of $\C^{(t)}$. More concisely, we have $\C^{(t + t_\mathrm{eig})} \succcurlyeq \frac14 \C^{(t)}$. 
The last argument, which is the most frequent case in practice, implies that the covariance matrix update does not need to be scaled down. 

\subsection{Method 2: Scaling Down Negative Weights}\label{sec:method2}

An alternative way to guarantee the positive definiteness of the covariance matrix is to scale down the sum of the absolute values of the negative weights, combined with the projection of unpromising solutions introduced above. We use the same update formula \eqref{eq:aaacma}, but $\alpha^{(t+t_\mathrm{eig})} = 1$. 

The positive definiteness of the covariance matrix is guaranteed under the condition $1 / t_\mathrm{eig} > \cone + \cmu \sum_{i=1}^{\lambda} \bw_i + n \cmu \sum_{i:\ w_i < 0} \abs{\bw_i}$. More precisely, we have the following claim.
\begin{theorem}\label{thm:pos}
  If $1 / t_\mathrm{eig} > \cone + \cmu \sum_{i=1}^{\lambda} \bw_i + n \cmu \sum_{i:\ w_i < 0} \abs{\bw_i}$, then $\C^{(t+t_\mathrm{eig})} \succcurlyeq\big(1 - t_\mathrm{eig} \big(\cone + \cmu \sum_{i=1}^{\lambda} \bw_i + n \cmu \sum_{i:\ w_i < 0} \abs{\bw_i}\big)\big) \C^{(t)}$.
\end{theorem}
\begin{proof}[Proof of Theorem~\ref{thm:pos}]
  First, we consider the case $t_\mathrm{eig} = 1$. Using the fact that the self outer product $\mathbf{v}\mathbf{v}^\T$ of a vector $\mathbf{v}$ is a matrix of rank $1$ with the only nonzero eigenvalue of $\norm{\mathbf{v}}^2$, we have
\begin{align*}
  \C^{(t+1)}
  &\succcurlyeq \left(1 - \cone - \cmu \sum_{i=1}^{\lambda} \bw_i\right) \C^{(t)} 
    - \cmu \sum_{i:\ w_i<0} \abs{\bw_{i}} \frac{n \yy_{i:\lambda}\yy_{i:\lambda}^\T}{\norm{\zz_{i:\lambda}}^2}\\
  &= \sqrt{\C}\left[ \left(1 - \cone - \cmu \sum_{i=1}^{\lambda} \bw_i\right) \eye
    - \cmu \sum_{i:\ w_i<0} n \abs{\bw_{i}} \frac{\zz_{i:\lambda}\zz_{i:\lambda}^\T}{\norm{\zz_{i:\lambda}}^2} \right]\sqrt{\C}\\
  &\succcurlyeq \sqrt{\C}\left[ \left(1 - \cone - \cmu \sum_{i=1}^{\lambda} \bw_i 
    - n \cmu \sum_{i:\ w_i<0} \abs{\bw_{i}}\right)  \eye\right]\sqrt{\C}\\
  &= \left(1 - \cone - \cmu \sum_{i=1}^{\lambda} \bw_i - n \cmu \sum_{i:\ w_i<0}\abs{\bw_i}\right) \C^{(t)} \enspace.
\end{align*}
The analogous argument holds for $t_\mathrm{eig} > 1$ as well. This completes the proof.
\end{proof}

We show how to construct the recombination weights so that they satisfy the sufficient condition of Theorem~\ref{thm:pos} as follows. Let $(w_i')_{i=1}^\lambda$ be the pre-defined weights that are non-increasing. Without loss of generality, we assume that the first $\mu$ weights are positive and sum to one. The recombination weights are $w_i = w_i'$ for $w_i' \geq 0$ and $w_i = \alpha w_i'$ for $w_i' < 0$, where $\alpha \in (0, 1]$. Then, the sufficient condition in Theorem~\ref{thm:pos} reads $1 / t_\mathrm{eig} > \cone + \cmu + (n-1) \cmu \alpha \sum_{i:\ w_i < 0} \abs{\bw_i'}$. It holds if we set for example $t_\mathrm{eig} < (\cone + \cmu)^{-1}$, as satisfied by the default choice $t_\text{eig} = \max\big(1, \big\lfloor (\beta_\mathrm{eig}  (\cone + \cmu) )^{-1} \big\rfloor\big)$, and
\begin{equation}
  \alpha = \frac{1 / t_\mathrm{eig}  - (\cone + \cmu)}{n \cmu \sum_{i:\ w_i < 0} \abs{w_i'} } \enspace.
  \label{eq:alpha2}
\end{equation}
Then, it is guaranteed that $\C^{(t+t_\mathrm{eig})} \succcurlyeq \frac1n \C^{(t)}$.

\begin{figure}[t]
  \centering
  \includegraphics[width=0.8\hsize]{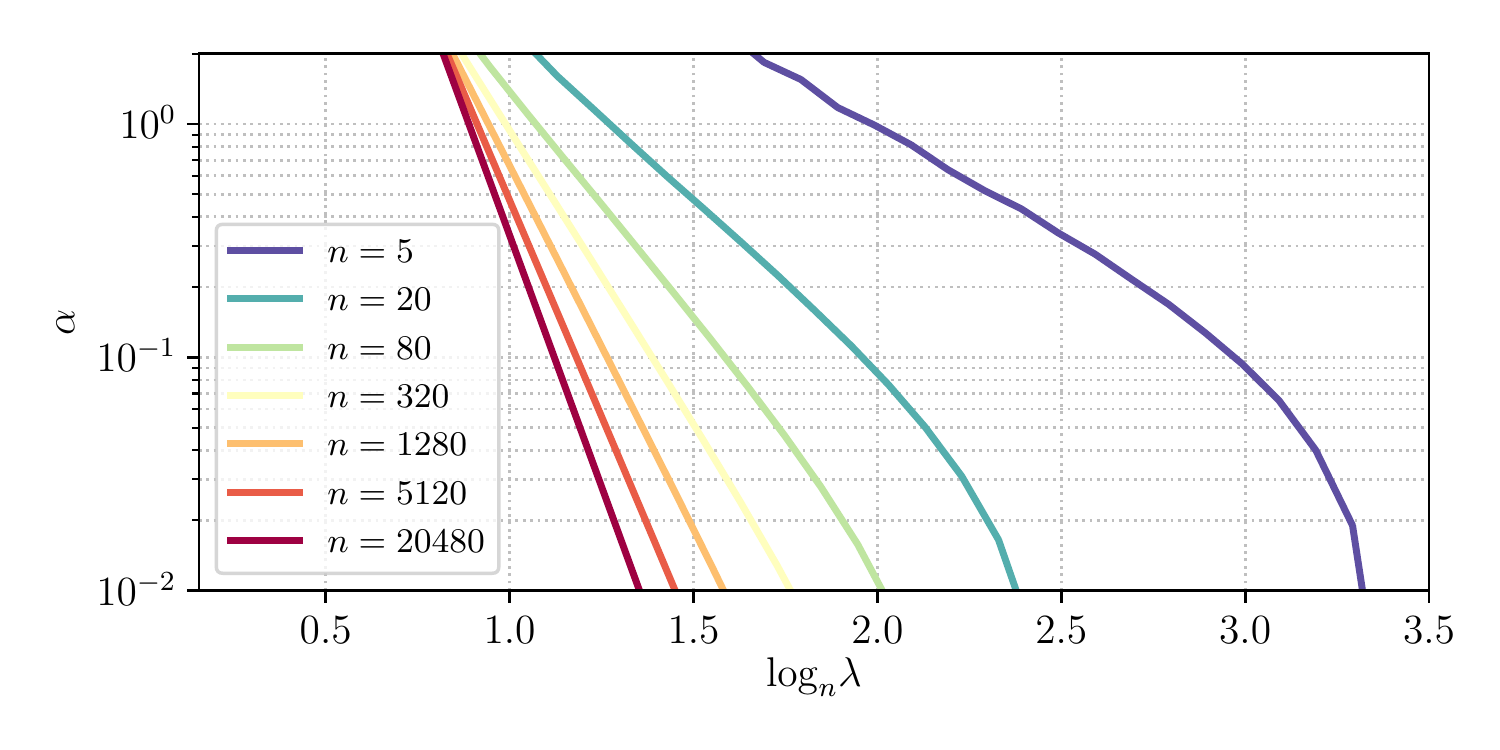}%
  \caption{The correction factor $\alpha$ in \eqref{eq:alpha2}.}
  \label{fig:alpha}
\end{figure}

This method is simpler than the method described in Section~\ref{sec:method1} since it does not require an additional eigenvalue computation. However, to guarantee the positive definiteness in this way, we bound the unrealistic worst case where all unpromising candidate solutions are sampled on a line. Therefore, the scaling down factor is set much smaller than the value that is actually needed in practice. Figure~\ref{fig:alpha} visualizes the correction factor $\alpha$ under our choice of the default pre-weights and learning rates described in Section~\ref{sec:param}. The sum of the negative recombination weights needs to decrease as the population size increases. For $n \geq 320$, the factor is less than $0.1$ for $\lambda \geq n^{1.5}$. Unless the internal computational time becomes a critical bottleneck, we prefer the method described in Section~\ref{sec:method1} over the method in Section~\ref{sec:method2}. 

\subsection{Choice of Recombination Weights}

We first review the rationale for the choice of the positive recombination weights. The default positive recombination weights, $\ln\frac{\lambda+1}{2} - \ln i$ for $i \leq \lambda / 2$, are based on the quality gain analysis of the weighted recombination ES with isotropic Gaussian distribution (i.e., $\C \propto \eye$). \citet{Arnold2006} studied the quality gain on a spherical function, and derived the optimal recombination weights for the infinite dimensional case. They are
\begin{equation}
  w_i \propto - \E[\NN_{i:\lambda}] \quad \text{for all $i = 1,\dots, \lambda$, }
\end{equation}
where $\NN_{1:\lambda} \leq \cdots \leq \NN_{\lambda:\lambda}$ is the ordered statistics from the independent and standard normally distributed random variables $\NN_1, \dots, \NN_\lambda$. Recently, it has been shown that the same recombination weights are optimal for a general convex quadratic function with Hessian matrix $\mathbf{A}$ satisfying $\Tr(\mathbf{A}^2)  / \Tr(\mathbf{A})^2 \ll 1$ for $n$ large \citep{AAH2017TCS}. The value $\ln\frac{\lambda + 1}{2} - \ln i$ approximates the first half of the optimal recombination weights.\footnote{%
  Interesting results are derived from the quality gain (and progress rate) analysis, see e.g.\ Section~4.2 of \citet{Hansen:2013wf}.
Comparing the ($1+1$)-ES and ($\mu$, $\lambda$)-ES with intermediate recombination, we realize that they reach the same quality gain in the limit for $n$ to infinity when $\mu$ is the optimal value, $\mu \approx 0.27  \lambda$ \citep{beyer1995toward}.
That is, a non-elitist strategy with optimal $\mu / \lambda$ for intermediate recombination can reach the same speed as the elitist strategy. With the optimal recombination weights above, the weighted recombination ES is $2.475$ times faster than those algorithms.
If we employ only the positive half of the optimal recombination weights, the speed will be halved, yet it is faster by the factor of $1.237$.}

We propose to use the recombination weights constructed as follows. Let
\begin{equation}
  w_i' = \ln\frac{\lambda+1}{2} - \ln i \qquad \text{for $i = 1,\dots, \lambda$.}
\end{equation}
The variance effective selection mass for the positive and negative weights is computed as
\begin{align}
  \mueff = \frac{(\sum_{i:w_i' > 0} \abs{w_i'})^2}{\sum_{i:w_i' > 0} \abs{w_i'}^2} \qquad\text{and}\qquad
  \mueff^{-} = \frac{(\sum_{i:w_i' < 0} \abs{w_i'})^2 }{\sum_{i:w_i' < 0} \abs{w_i'}^2}  \enspace.
\end{align}
The learning rates $\cone$ and $\cmu$ may be computed depending on the above quantities. The default values for the learning rates are discussed in Section~\ref{sec:param}. The recombination weights $w_i$ are set as follows
\begin{equation}
  w_i =
  \begin{cases}
    \frac{w_i'}{ \sum_{j:w_j' > 0} \abs{w_j'}} & \text{for $i$ with $w_i' \geq 0$,}\\
    \frac{w_i'}{ \sum_{j:w_j' < 0} \abs{w_j'}} \times \min\left( 1 + \frac{\cone}{\cmu}, \ 1 + \frac{2\mueff^{-}}{\mueff + 2}\right) & \text{for $i$ with $w_i' < 0$.}
  \end{cases}
  \label{eq:ourw}
\end{equation}

The positive recombination weights are unchanged from the default settings without active CMA. They are approximately proportional to the positive half of the optimal recombination weights. However, this is not the case for the negative recombination weights. The optimal recombination weights are symmetric about zero, i.e.\ $w_i = - w_{\lambda - i + 1}$, whereas our negative weights tend to level out for the following reasons.
The above mentioned quality gain results consider only the mean update.
The obtained optimal values are not necessarily optimal for the covariance matrix adaptation. Since we use only positive recombination weights for the mean vector update, the negative weights do not need to correspond to these optimal recombination weights. Furthermore, the optimal negative weights---greater absolute values for worse steps---counteract our motivation for rescaling of unpromising steps discussed in Section~\ref{sec:method1}.
Our choice of the negative recombination weights is a natural extension of the default positive recombination weights. The shape of our recombination weights is somewhat similar to the one in xCMA-ES \citep{Krause2015}, where the first half is $w_i - 1/\lambda$ and the last half is $-1/\lambda$. A minor difference is that xCMA-ES assigns negative values even for some of the better half of the candidate solutions, whereas our setting assigns positive values only for the better half and negative values only for the worse half.

Positive and negative recombination weights are scaled differently. Positive weights are scaled to sum to one, whereas negative weights are scaled so that the sum of their absolute values is the minimum of $1 + \frac{\cone}{\cmu}$ and $1 + \frac{2\mueff^{-}}{\mueff + 2}$. The latter corrects for unbalanced positive and negative weights, but is usually greater than the former with our default values. The former makes the decay factor $1 - \cone - \cmu \sum_{i=1}^{\lambda} w_i$ of the covariance matrix update (see \eqref{eq:acma}) to $1$. That is, the covariance matrix does not passively shrink, and only the negative update can shrink the covariance matrix. \citet{Krause2015} mention to have no passive shrinkage by setting the sum of the weights to zero, but the effect of the rank-one update was not taken into account and thus in xCMA the passive decay factor of $1 - \cone$ remains.

\section{Adaptive Diagonal Decoding (\DDCMA)}\label{sec:add}

\newcommand{\coned}{\ensuremath{c_{1,D}}}
\newcommand{\cmud}{\ensuremath{c_{\mu,D}}}
\newcommand{\ccd}{\ensuremath{c_{c,D}}}
\newcommand{\pcd}{\ensuremath{\ve{p}_{c,D}}}

The separable CMA-ES \citep{Ros2008ppsn} enables to adapt a diagonal covariance matrix faster than the standard CMA-ES because the learning rates are $O(n)$ times greater than in standard CMA-ES.
This works well on separable objective functions, where separable CMA-ES adapts the (co)variance matrix by a factor of $O(n)$ faster than CMA-ES. To accelerate standard CMA, we combine separable CMA and standard CMA. We adapt both $\D$ and $\C$ at the same time\footnote{%
  Considering the eigen decomposition of the covariance matrix, $\mathbf{E}\Lam\mathbf{E}^\T$, instead of the decomposition $\D\C\D$, our attempts to additionally adapt \Lam\ were unsuccessful. We attribute the failure to the strong interdependency between \Lam\ and $\mathbf{E}$. Compared to the eigen decomposition, the diagonal decoding model $\D\C\D$ also has the advantage to be interpretable as a rescaling of variables. We never considered the decomposition $\sqrt{\C}\D^2\sqrt{\C}$ as proposed by one of the reviewers of this paper, however, at first sight, we do not see any particular advantage in favor of this decomposition.
  % \niko{$\sqrt{\C}$ does not rotate a basis of eigenvectors, yet it is still difficult to see whether/how this decomposition can be interpreted in the given coordinate system.}
},
where $\D$ is updated with \emph{adaptive} learning rates which can be much greater than those used to update $\C$.
 
\subsection{Algorithm}
\newcommand{\betathresh}{\ensuremath{\beta_\mathrm{thresh}}}
The update of $\D$ is similar to separable CMA, but is done in local exponential coordinates. The update of $\D$ is as follows
\begin{gather}
  [\Delta_D]_{k,k}
  = \coned \big(\big[\sqrt{\C}^{-1}(\D^{(t)})^{-1}\pcd\big]_{k}^2 - \gamma_{c,D}\big)
  + \cmud \sum_{i= 0}^{\lambda} \bw_{i,D} \big(\big[\tilde\zz_{i:\lambda}\big]_{k}^2 - 1\big)
   \label{eq:deltad}\\
  \D^{(t+1)} \leftarrow \D^{(t)} \cdot \exp\left(\frac{\Delta_D}{2\beta^{(t)}} \right) \enspace,\label{eq:dup}
\end{gather}
% \del{\begin{gather}
%   [\Delta_D]_{k,k}
%   = \coned \big[\sqrt{\C}^{-1}(\D^{(t)})^{-1}\pcd\big]_{k}^2
%   + \cmud \left[\sum_{i:\ w_i > 0} \bw_i \big[\zz_{i:\lambda}\big]_{k}^2
%   +  \sum_{i:\ w_i < 0} \bw_i \frac{n \big[\zz_{i:\lambda}\big]_{k}^2}{\norm{\zz_{i:\lambda}}^2}\right]  \label{eq:deltad}\\
%   \D^{(t+1)} \leftarrow \D^{(t)} \cdot \exp\left(\frac{1}{2\beta^{(t)}} \left(\Delta_D - \frac{\Tr(\Delta_D)}{n}\eye\right) \right) \enspace,\label{eq:dup}
% \end{gather}}%
%
where $[\Delta_D]_{k,k}$ is the $k$-th diagonal element of the diagonal matrix $\Delta_D$, the evolution path $\pcd$ feeds the rank-one $\D$ update, $\coned$ and $\cmud$ are the learning rates for the rank-one and rank-$\mu$ $\D$ updates, respectively, the recombination weights $w_{i,D}$ are computed by \eqref{eq:ourw} using $\coned/\cmud$ instead of $\cone/\cmu$, $\tilde\zz_{i:\lambda}$ is defined in \eqref{eq:zrescaled}, and $\beta^{(t)}$ is the damping factor for the diagonal matrix update given in \eqref{eq:ddamp} below.
The evolution path $\pcd$ and its normalization factor $\gamma_{c,D}$ are updated in the same way as \eqref{eq:pc} and \eqref{eq:gc} with the cumulation factor $\ccd$ rather than $\cc$.
The matrix exponential of a diagonal matrix is $\exp(\diag(d_1, \dots, d_n)) = \diag(\exp(d_1), \dots, \exp(d_n))$.

Using the correlation matrix of the covariance matrix $\C$ of the last time the matrix was decomposed,
\begin{equation}
\Ccorr := \sqrt{\diag(\C)}^{-1}\C \sqrt{\diag(\C)}^{-1}
        %\sqrt{\diag(\C)}^{-1}\C \sqrt{\diag(\C)}^{-1}
        \enspace,
\end{equation}
where $\diag(\C)$ is the diagonal matrix whose diagonal elements are the diagonal elements of $\C$, the damping factor is based on the square root of the condition number of this correlation matrix as
\begin{equation}
  \beta^{(t)} = \max\left(1,\ 
  %\sqrt{\frac{ d_1 \big( \Ccorr\big) }{d_n\big( \Ccorr\big) }}
  \sqrt{\Cond\left(\Ccorr\right)}
  - \betathresh + 1\right)
  \enspace,
  \label{eq:ddamp}
\end{equation}
where %$d_i$ denotes the $i$-th largest eigenvalue and
$\betathresh:=2$ is the threshold parameter at which $\beta^{(t)}$ becomes larger than one. We remark that $\beta^{(t)}$ changes only every $t_\text{eig}$ iterations.

The $\D$-update is multiplicative as in
\citet{Ostermeier1994ec} or \citet{Krause2015} to be unbiased on the log-scale and to simply guarantee the positive definiteness of $\D$.
Note that the first order approximation of the update formula $\D^{(t)} \exp((2\beta)^{-1} \Delta_D)$ gives
\begin{equation*}
  \left(\D^{(t+1)}\right)^2 \approx \left(\D^{(t)}\right)^2  + \frac{1}{\beta^{(t)}} \D^{(t)} \Delta_D \D^{(t)} \enspace,
\end{equation*}
which is the update in separable CMA. 

Importantly, we set the learning rates for the $\D$ update, $\coned$ and $\cmud$, to be about $n$ times greater than the learning rates for the $\C$ update, $\cone$ and $\cmu$. Moreover, we maintain an additional evolution path, $\pcd$, for the rank-one update of $\D$ since an adequate cumulation factor, $\ccd$, may be different from the one for $\pc$.

\begin{figure}[t]
  \centering
  \begin{subfigure}{0.33\textwidth}%
    \centering%
    \includegraphics[trim={1.0cm 1.0cm .8cm .8cm},clip,width=\hsize]{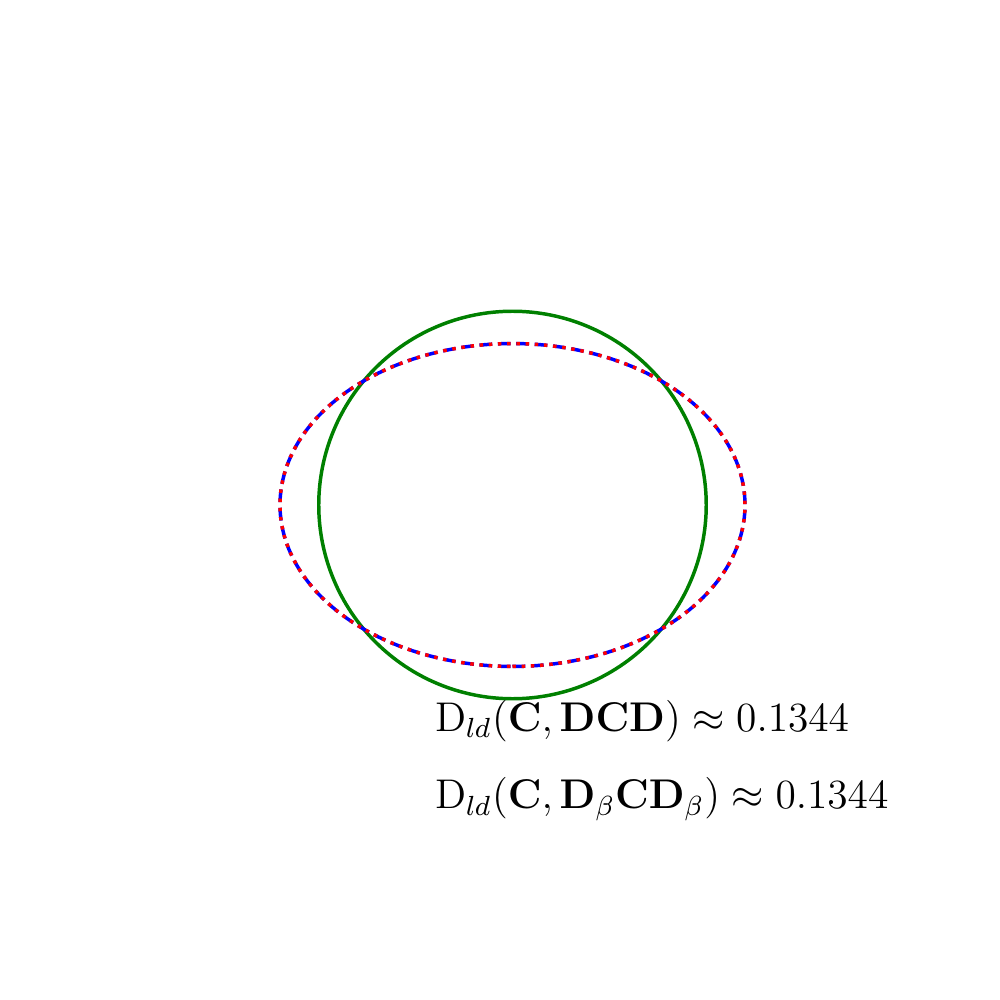}%
    \caption{axis ratio: $1$}%
  \end{subfigure}%
  \begin{subfigure}{0.33\textwidth}%
    \centering%
    \includegraphics[trim={1.0cm 1.0cm .8cm .8cm},clip,width=\hsize]{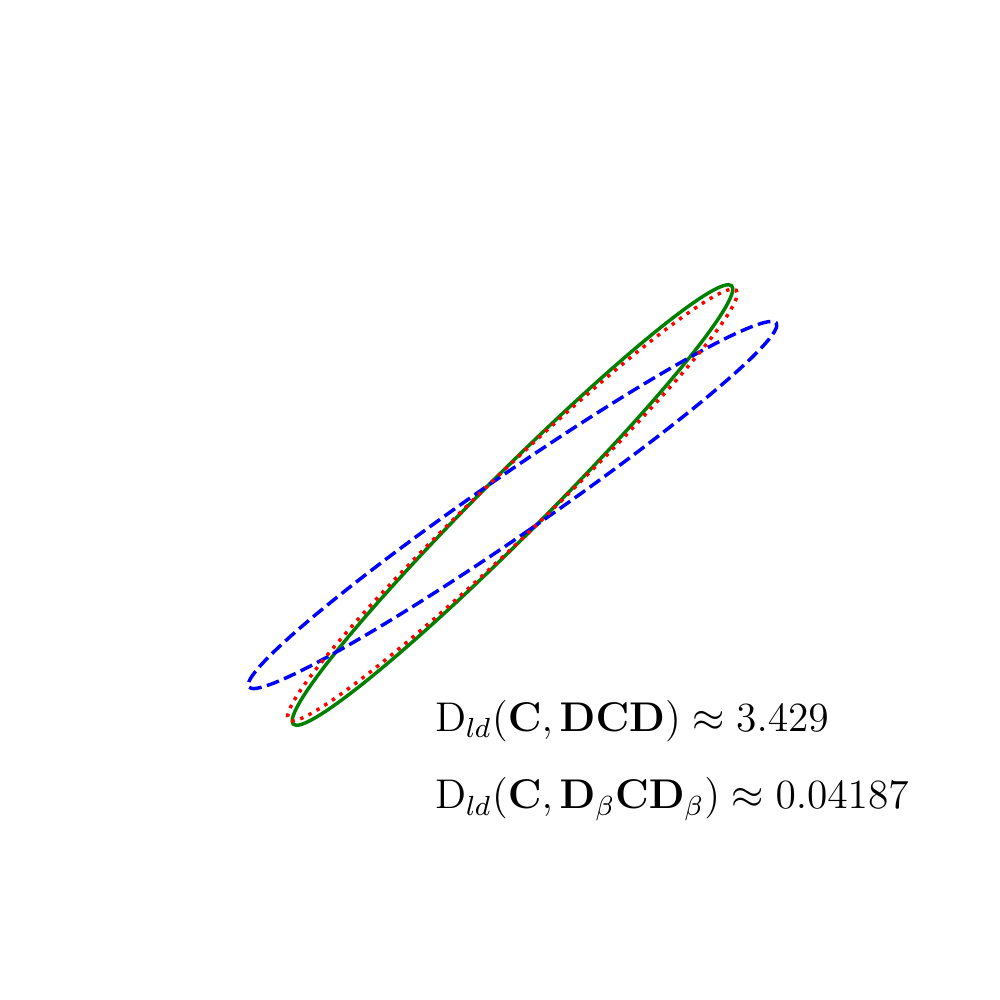}%
    \caption{axis ratio: $10$}%    
  \end{subfigure}%
  \begin{subfigure}{0.33\textwidth}%
    \centering%
    \includegraphics[trim={1.0cm 1.0cm .8cm .8cm},clip,width=\hsize]{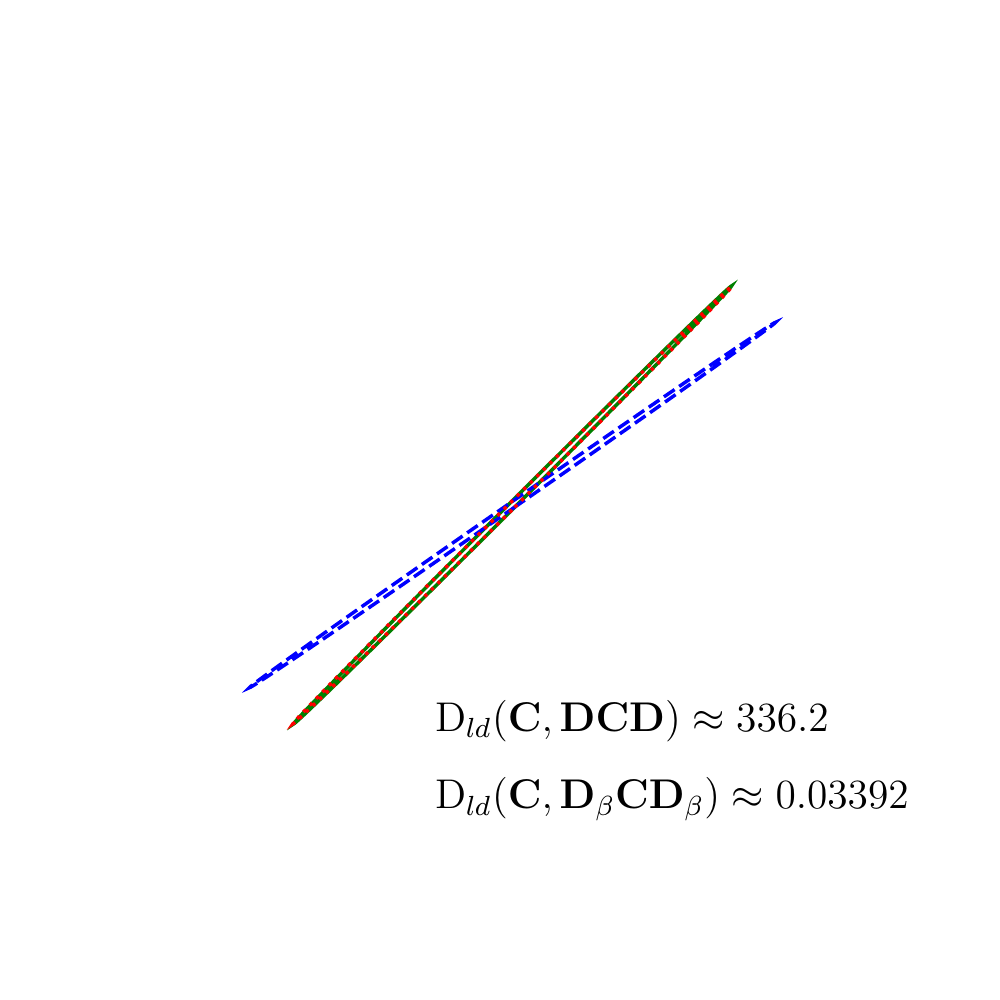}%
    \caption{axis ratio: $100$}%        
  \end{subfigure}%
  \caption{Equiprobability ellipse defined by $\C$, $\D\C\D$, and $\D_{\beta}\C\D_{\beta}$. Green ellipses represent $\mathcal{N}(\bm{0}, \C)$ with axis ratio of $1$, $10$, $100$ and inclination angle $\pi / 4$. Blue dashed ellipses represent $\mathcal{N}(\bm{0}, \D\C\D)$ with $\D = [1.2, 1/1.2]$. Red ellipses represent $\mathcal{N}(\bm{0}, \D_{\beta}\C\D_{\beta})$ with $\D_{\beta}$ damped by using $\beta$ in \eqref{eq:ddamp}.}
  \label{fig:add}
\end{figure}

\subsection{Damping factor}

The \emph{dynamic} damping factor $\beta^{(t)}$ is
the crucial novelty that prevents diagonal decoding to disrupt the adaptation of the covariance matrix $\C$ in CMA-ES. The damping factor is introduced to prevent the sampling distributions from drastically changing due to the diagonal matrix update. Figure~\ref{fig:add} visualizes three example cases with different $\C$. When the diagonal decoding matrix changes from the identity matrix to $\D$, the change of the distribution from $\mathcal{N}(\bm{0}, \C)$ to $\mathcal{N}(\bm{0}, \D\C\D)$ is minimal when $\C$ is diagonal, and is comparatively large if $\C$ is non-diagonal. It will be greater as the condition number of the correlation matrix $\Ccorr$ increases. In the worst case in Figure~\ref{fig:add}, we see that two distributions are overlapping only at the center of distribution. 

The damping factor $\beta^{(t)}$ computed in \eqref{eq:ddamp} is motivated from the KL divergence between the sampling distributions before and after the $\D$ update.
The KL divergence between two multivariate normal distributions that have the same mean vectors is equivalent to the half of the Log-Determinant divergence between the covariance matrices, defined as%
\begin{equation}
  D_{ld}(\mathbf{A}, \mathbf{B}) = \Tr(\mathbf{A}\mathbf{B}^{-1}) - \log\det(\mathbf{A}\mathbf{B}^{-1}) - n \enspace,
\end{equation}
where $\mathbf{A}$ and $\mathbf{B}$ are two positive definite symmetric matrices. Let $\mathbf{A} = \D' \C \D'$ and $\mathbf{B} = \D\C\D$, i.e., $\mathbf{A}$ and $\mathbf{B}$ are the covariance matrices after and before the $\D$ update, respectively. The divergence $D_{ld}$ is asymmetric, and the value changes if we exchange $\mathbf{A}$ and $\mathbf{B}$, but we will obtain a symmetric approximation in the end. To analyze the divergence between $\mathbf{A}$ and $\mathbf{B}$ due to the difference in shape and put aside the effect of a global scaling, we scale both to be of determinant one. We denote the resulting divergence by $\bar D_{ld}(\mathbf{A}, \mathbf{B}) = D_{ld}(\mathbf{A} / \det(\mathbf{A}), \mathbf{B} / \det(\mathbf{B}))$. Then, we obtain

\begin{align*}
  \bar D_{ld}(\mathbf{A}, \mathbf{B})
  &= \Tr\left(\exp\left(\frac{1}{2\beta}\bar\Delta_D \right)\C\exp\left(\frac{1}{2\beta} \bar\Delta_D \right)\C^{-1}\right) - n      
   \enspace,
\end{align*}
where $\bar\Delta_D = \Delta_D - (\Tr(\Delta_D ) / n) \eye$.
With the approximation $\exp\left(\frac{1}{2\beta} \bar\Delta_D \right) \approx \eye + \left(\frac{1}{2\beta}\right) \bar\Delta_D + \frac12 \left(\frac{1}{2\beta}\right)^2 \bar\Delta_D^2$ and neglecting $(\Delta_D / \beta)^3$ and higher terms, we have\footnote{Let $\mathbf{A}$ and $\mathbf{B}$ be an arbitrary positive definite symmetric matrix and an arbitrary symmetric matrix, respectively. Let $\otimes$ and $\mathrm{vec}$ denote the Kronecker product of two matrices and the matrix-vector rearrangement operator that successively stacks the columns of a matrix. From Theorem~16.2.2~of \citet{Harville2008book}, we have $\Tr(\mathbf{A}\mathbf{B}\mathbf{A}^{-1}\mathbf{B}) = \mathrm{vec}(\mathbf{B})^\T(\mathbf{A} \otimes \mathbf{A}^{-1}) \mathrm{vec}(\mathbf{B})$. The eigenvalues of the Kronecker product $\mathbf{A} \otimes \mathbf{A}^{-1}$ are the products of the eigenvalues of $\mathbf{A}$ and $\mathbf{A}^{-1}$ (Theorem~21.11.1~of \citet{Harville2008book}). They are all positive and upper bounded by the condition number of $\mathbf{A}$. Therefore, we have $\mathrm{vec}(\mathbf{B})^\T(\mathbf{A} \otimes \mathbf{A}^{-1}) \mathrm{vec}(\mathbf{B}) \leq \Cond(\mathbf{A}) \norm{\mathrm{vec}(\mathbf{B})}^2 = \Cond(\mathbf{A}) \Tr(\mathbf{B}^2)$. Letting $\mathbf{A} = \Ccorr$ and $\mathbf{B} = \bar\Delta_D$, we have $\Tr\left( \bar\Delta_D\C\bar\Delta_D\C^{-1}\right) = \Tr(\mathbf{A}\mathbf{B}\mathbf{A}^{-1}\mathbf{B}) \leq (\beta+\betathresh-1)^{2}\Tr(\bar\Delta_D^2)$. }
\begin{align*}
  \bar D_{ld}(\mathbf{A}, \mathbf{B})
  &\approx \left(\frac{1}{2\beta}\right)^2 \Tr\left( \bar\Delta_D\C\bar\Delta_D\C^{-1}\right) + \left(\frac{1}{2\beta}\right)^2 \Tr\left(\bar\Delta_D^2\right) \\
  &\leq \frac{(\beta+\betathresh-1)^{2}+1}{4\beta^2} \Tr\left(\bar\Delta_D^2\right) 
  %&= \left(\frac{\beta}{2}\right)^2 \left(\sum_{i=1}^{d}\sum_{j=1}^{d} [\C]_{i,j}[\C^{-1}]_{i,j}[\bar\Delta_D]_i[\bar\Delta_D]_j + \sum_{i=1}^{d}[\bar\Delta_D]_i^2\right) 
  \enspace.
%\label{eq:logdet}
\end{align*}
Note that $\bar D_{ld}((\D^{(t+1)})^2, (\D^{(t)})^2) \approx 2^{-1} \Tr\left(\bar\Delta_D^2\right)$ if $\C$ is diagonal and $\beta = 1$. Therefore, with $\beta$ computed in \eqref{eq:ddamp}, we bound the realized divergence approximately by the divergence of \D\ like
\begin{equation}
  \bar D_{ld}(\D^{(t+1)} \C \D^{(t+1)}, \D^{(t)}\C \D^{(t)}) \leq \bar D_{ld}((\D^{(t+1)})^2, (\D^{(t)})^2)
  \label{eq:divineq}
\end{equation}
for large $\beta$.

In a nutshell, we have quantified the distribution change from changing \D\ by measuring its KL-divergence. We derive that $\beta$ in \eqref{eq:ddamp} upper bounds the KL-divergence due to changes of \D\ approximately by the KL-divergence from the same change of \D\ when $\C=\eye$. The latter is directly determined by the learning rates \coned\ and \cmud.

\subsection{Implementation Remark}

To avoid unnecessary numerical errors and additional computational effort for eigen decompositions in the adaptive diagonal decoding mechanism \eqref{eq:deltad}, \eqref{eq:dup}, \eqref{eq:ddamp}, we force the diagonal elements of $\C$ to be all one by assigning
\begin{gather}
  \D^{(t)} \leftarrow \D^{(t)} \diag(\C^{(t)})^{\frac12} \label{eq:dc1}\\
  \C^{(t)} \leftarrow \mathrm{corr}(\C^{(t)}) = \diag(\C^{(t)})^{-\frac12} \C^{(t)} \diag(\C^{(t)})^{-\frac12} \enspace.\label{eq:dc2}
\end{gather}
These lines are performed just before the eigen decomposition. It means that $\D$ and $\C$ are the variance matrix and the correlation matrix of the sampling distribution, respectively.
This reparametrization does not change the algorithm itself, it only improves the numerical stability of the implementation.
Note that if $\C$ is constrained to be a correlation matrix, then $\D\C\D$ is a \emph{unique} parametrization for the covariance matrix.

The eigen decomposition of $\C^{(t)}$ is performed every $t_\text{eig}$ iterations. We keep $\sqrt{\C}$ and $\sqrt{\C}^{-1}$ until the next matrix decomposition is performed. Suppose that the eigen decomposition $\C^{(t)} = \mathbf{E} \Lambda \mathbf{E}^\T$ is performed at iteration $t$.
Then, the above matrices are
$\sqrt{\mathbf{C}} = \mathbf{E} \Lambda^{\frac12} \mathbf{E}^\T$ and $\sqrt{\mathbf{C}}^{-1} = \mathbf{E} \Lambda^{-\frac12} \mathbf{E}^\T$. Then, we can compute $\beta$ in \eqref{eq:ddamp} as $\beta = \max(1, (\max_i([\Lambda]_{i,i}) / \min_i([\Lambda]_{i,i}))^\frac12 - \beta_\mathrm{thresh} + 1)$.
This $\beta$ is kept until the next eigen decomposition is performed. Then,
the additional computational cost for adaptive diagonal decoding is $O(n^2 / t_\mathrm{eig} + \lambda n)$,
which is smaller than the computational cost for the other parts of the CMA-ES, $O(n^3 / t_\mathrm{eig}  + \lambda n^2)$ per iteration.

One might think that maintaining $\D$ is not necessary and $\C$ could be updated by pre- and post-multiplying by a diagonal matrix absorbing the effect of the $\D$ update. However, because diagonal decoding may lead to a fast change of the distribution shape,  the decomposition of $\C$ then needs to be done every iteration. The chosen parametrization circumvents frequent matrix decompositions despite changing the sampling distribution rapidly.

\section{Algorithm Summary and Strategy Parameters}\label{sec:param}

The \DDCMA-ES combines weighted recombination, active covariance matrix update with positive definiteness guarantee (described in Section~\ref{sec:method1}), adaptive diagonal decoding (described in Section~\ref{sec:add}) and CSA, and is summarized in Algorithm~\ref{alg:ddcma}.\footnote{Its python code is available at\\  \url{https://gist.github.com/youheiakimoto/1180b67b5a0b1265c204cba991fa8518}}
\begin{algorithm}[t]
\begin{algorithmic}[1]
  % \Require{$\lambda$, $w_i$, $\cm$}\Comment{for sampling and recombination}
  % \Require{$\cc$, $\cone$, $\cmu$}\Comment{for covariance matrix adaptation}
  % \Require{$\ccd$, $\coned$, $\cmud$, $\beta_\mathrm{thresh}$}\Comment{for adaptive diagonal decoding}
  % \Require{$\cs$, $\ds$}\Comment{for step-size adaptation}
  \Require{$\m$, $\sigma$}\Comment{initial distribution parameters}
  \State $\D = \sqrt{\C} = \sqrt{\C}^{-1} = \eye$, $\mathbf{K} = \mathbf{O}$ ($\mathbf{O}$: zero matrix)
  , $\pc = \ps = \pcd = \bm{0}$, $\gc = \gs = 0$, $\beta = 1$, $t = 0$
  \Repeat
  \State sample $\lambda$ points $(\zz_i,\ \yy_i,\ \xx_i)_{i=1}^{\lambda}$ by \eqref{eq:sampling}
  \State evaluate $f(\xx_i)$  for all $i=1,\dots,\lambda$
  \State sort in ascending order of $f(\xx_i)$ and denote $i$th best point as $\zz_{i:\lambda}$, $\yy_{i:\lambda}$, and $\xx_{i:\lambda}$
  \State update $\m$ by \eqref{eq:m}
  \State update $\ps$, $\gs$ and $\sigma$ by \eqref{eq:ps}, \eqref{eq:gs} and \eqref{eq:sigma}, then compute $\hsig$ by \eqref{eq:hs}
  \State update $\pc$ and $\pcd$ by \eqref{eq:pc} with factors $\cc$ and $\ccd$, resp., and update $\gc$ by \eqref{eq:gc}
  \State compute $\mathbf{Z}^{(t)}$ by \eqref{eq:zmat} and add it to $\mathbf{K}$ \Comment{$\mathbf{K} = \sum_{k=t - (t \mod\teig)}^{t} \mathbf{Z}^{(k)}$}
  \State update $\D$ by \eqref{eq:deltad} and \eqref{eq:dup} with $\beta^{(t)} = \beta$
  \State $t \leftarrow t + 1$
  \If{$(t \mod t_\mathrm{eig}) = 0$}
  \State compute $\C$ by \eqref{eq:aaacma} and  $\alpha$ by \eqref{eq:alpha1} using $\mathbf{K}$
  \State $\D\leftarrow \D \sqrt{\diag(\C)}$ and $\C \leftarrow \sqrt{\diag(\C)}^{-1} \C\sqrt{\diag(\C)}^{-1}$\Comment{\eqref{eq:dc1} and \eqref{eq:dc2}}
  \State perform eigen decomposition $\C = \mathbf{E}\Lambda\mathbf{E}^\T$
  \State $\beta \leftarrow \max(1, \sqrt{\max_{i}[\Lambda]_{i,i} / \min_{j}[\Lambda]_{j,j}} - \beta_\mathrm{thresh}+1)$ \Comment{\eqref{eq:ddamp}}
  \State $\sqrt{\C} \leftarrow \mathbf{E}\Lambda^\frac{1}{2}\mathbf{E}^\T$, $\sqrt{\C}^{-1} \leftarrow \mathbf{E}\Lambda^{-\frac{1}{2}}\mathbf{E}^\T$, $\mathbf{K} \leftarrow \mathbf{O}$
  \EndIf
  \Until{termination condition are met}
\end{algorithmic}
\caption{\DDCMA-ES: CMA-ES with adaptive diagonal decoding}
\label{alg:ddcma}
\end{algorithm}

\begin{algorithm}[t]
\begin{algorithmic}[1]
  % \Require{$\lambda$, $w_i$, $\cm$}\Comment{for sampling and recombination}
  % \Require{$\cc$, $\cone$, $\cmu$}\Comment{for covariance matrix adaptation}
  % \Require{$\ccd$, $\coned$, $\cmud$, $\beta_\mathrm{thresh}$}\Comment{for adaptive diagonal decoding}
  % \Require{$\cs$, $\ds$}\Comment{for step-size adaptation}
  \Require{$n$ and optionally $\lambda$}
  \State $\lambda = 4 + \lfloor 3 \ln n \rfloor$ if not given
  %\State $\mu = \lfloor \lambda / 2 \rfloor$
  \State $w_i' = \ln\frac{\lambda + 1}{2} - \ln i$ for $i = 1, \dots, \lambda$
  \State $\mueff = \frac{(\sum_{i:w_i' > 0} \abs{w_i'})^2}{\sum_{i:w_i' > 0} \abs{w_i'}^2}$ and $\mueff^{-} = \frac{(\sum_{i:w_i' < 0} \abs{w_i'})^2}{ \sum_{i:w_i' < 0} \abs{w_i'}^2}$
  \State $\cm = 1$
  \State $\cs = \frac{\mueff + 2}{n + \mueff + 5}$ and $\ds = 1 + \cs + 2 \max\left( 0,\ \sqrt{\frac{\mueff - 1}{n + 1}} - 1\right)$
  \State $\cone,\ \coned = \frac{1}{ 2 (m / n + 1)(n + 1)^{3/4} + \mueff/2}$ with $m = n(n+1)/2$ and $n$, resp.
  \State $\cmu,\ \cmud = \min(\mu'\cone, 1 - \cone),\ \min(\mu'\coned, 1 - \coned)$ with $\mu'= \mueff + \frac{1}{\mueff} - 2 + \frac12\frac{\lambda}{\lambda + 5}$
  % \min\left(1 - \cone, \frac{\mueff + 1/\mueff - 2 + \lambda / (2\lambda + 10)}{ 2 (m / n + 1)(n + 1)^{3/4} + \mueff/2}\right)$ with $m = n(n+1)/2$ and $n$, resp.
  \State $\cc,\ \ccd = \frac{\sqrt{\mueff\cone}}{2}, \frac{\sqrt{\mueff\coned}}{2}$
  \State $w_i, w_{i,D} =
  \begin{cases}
    \frac{w_i'}{ \sum_{j:w_j' > 0} \abs{w_j'}} & \text{for $w_i' \geq 0$}\\
    \frac{w_i' \times \min\left( 1 + r, \ 1 + 2\mueff^{-} / (\mueff + 2)\right)}{ \sum_{j:w_j' < 0} \abs{w_j'}}  & \text{for $w_i' < 0$}
  \end{cases} \text{ with $r = \frac{\cone}{\cmu}$ and $\frac{\coned}{\cmud}$, resp.}$
  \State $t_\text{eig} = \max\big(1, \big\lfloor (\beta_\mathrm{eig}  (\cone + \cmu) )^{-1} \big\rfloor\big)$ with $\beta_\mathrm{eig} = 10n$
  \State $\beta_\mathrm{thresh} = 2$
\end{algorithmic}
\caption{Default parameter computation for \DDCMA-ES}
\label{alg:param}
\end{algorithm}
Providing good default values for the strategy parameters (aka hyper parameters) is, needless to say, essential for the success of a novel algorithm. Especially in a black-box optimization scenario, parameter tuning for an application relies on trial-and-error and is usually prohibitively time consuming. The computation of the default parameter values is summarized in Algorithm~\ref{alg:param}. Since we have $\cone / \cmu = \coned / \cmud$ as long as $\cone + \cmu < 1$ (see below), the recombination weights for the update of $\C$ and of $\D$ are the same, $w_i = w_{i,D}$.

\begin{figure}\centering
  \includegraphics[width=0.5\hsize]{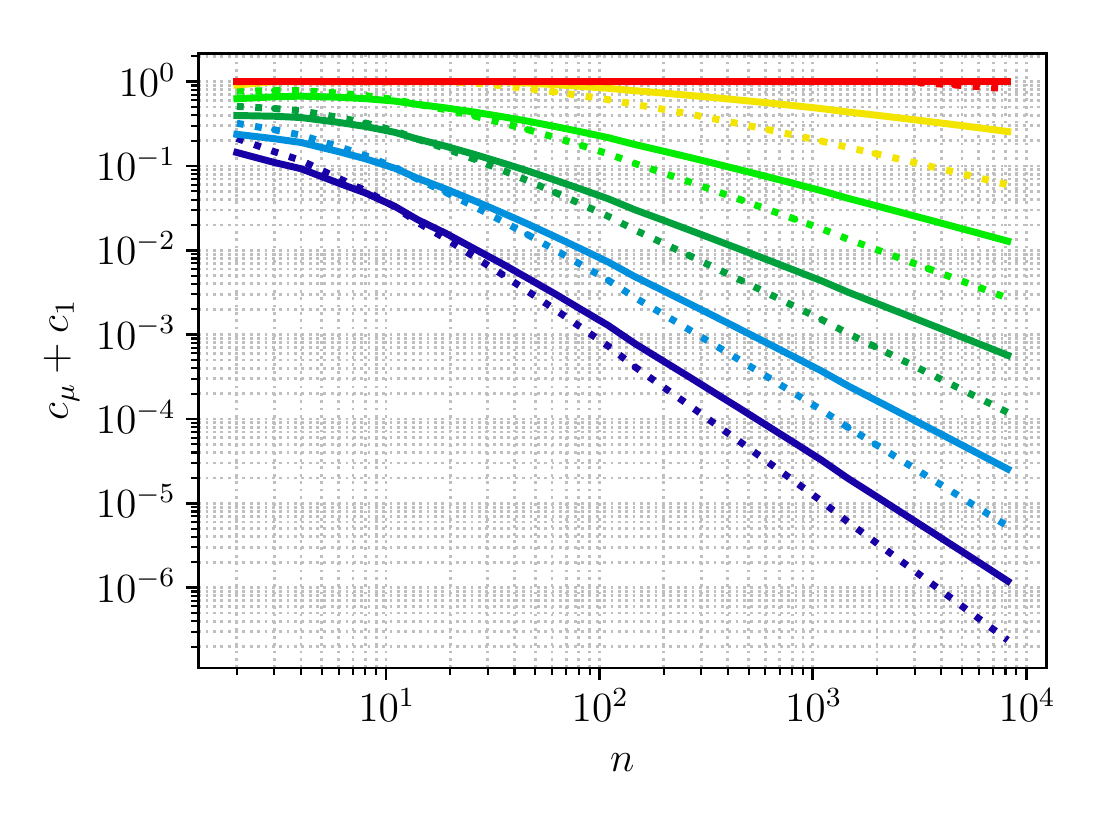}%
  \includegraphics[width=0.5\hsize]{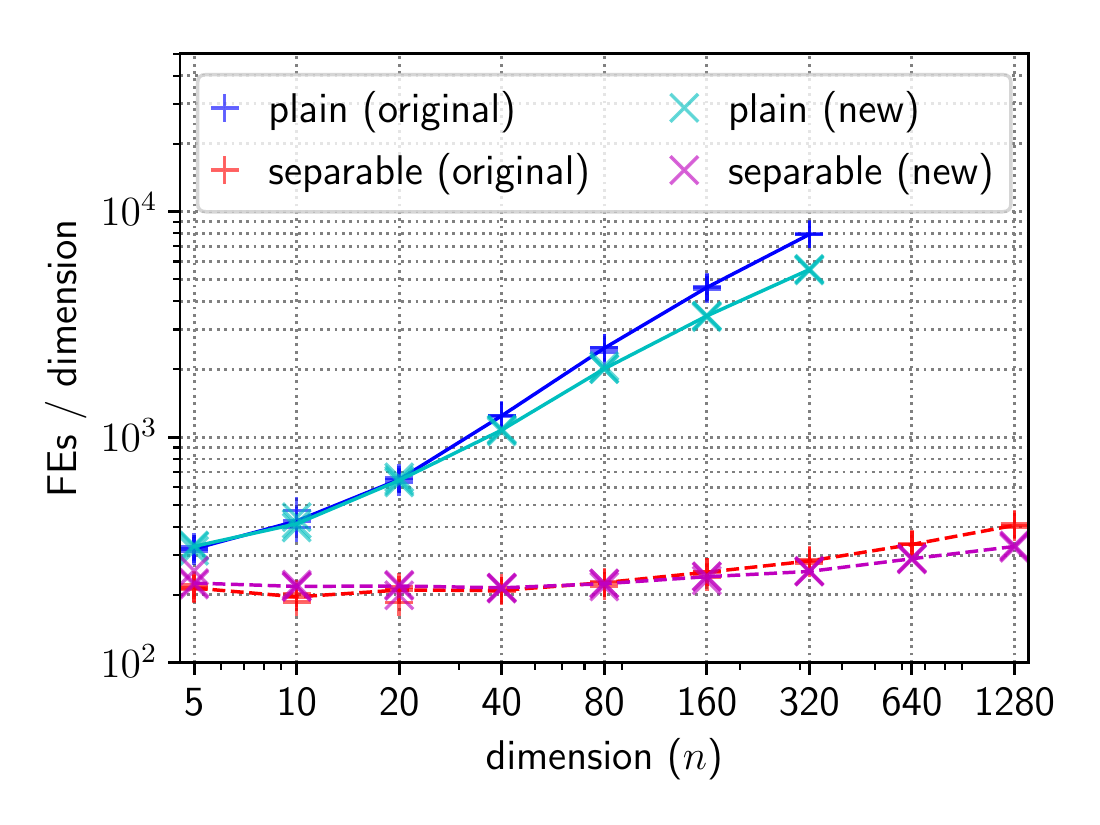}%
  \caption{Left: Scaling of $\cone + \cmu$ with the dimension $n$ for different population sizes. Solid lines are for the new setting from Table~\ref{alg:param}, $\Theta(n^{1/4}\lambda/m)$ and dashed lines are for the original setting, $\Theta(\lambda/m)$. 
For $n<10$, the new learning rates are slightly more conservative, for $n>20$ they are more ambitious.
Shown are six population sizes equally log-spaced between $\lambda = 4 + \lfloor 3 \log(n) \rfloor$ (bottom) and $\lambda = 64n^2$ (top). Right: Function evaluations on the ellipsoid function for plain and separable CMA with original and new parameter settings. Three independent trials have been conducted for each setting, and the median of each setting is indicated by a line (solid: plain CMA, dashed: separable CMA). See Section~\ref{sec:exp} for details.}
\label{fig:lrcomp}
\end{figure}

The learning rates for $\D$ and $\C$ are modified to improve the adaptation speed especially in high dimension.
Let $m$ be the degrees of freedom of the matrix to be adapted, i.e., $m = (n+1)n/2$ for the $\C$-update and $m = n$ for the $\D$-update. The learning rate parameters and the cumulation factors are set to the following values
\begin{align}
    \cone,\ \coned &= \frac{1}{ 2 (m / n + 1)(n + 1)^{3/4} + \mueff/2} \label{eq:cone}\\
    %\new{
    \cmu,\ \cmud &= \min(\mu'\cone, 1 - \cone),\
                    \min(\mu'\coned, 1 - \coned) \label{eq:cmu}\\
     &\text{with~} \mu' = 
    \mueff + \frac{1}{\mueff} - 2 + \frac{\lambda}{2(\lambda + 5)}\nonumber\\
    %}
%    \cmu,\ \cmud &= \min\left(1 - \cone, \frac{\mueff + 1/\mueff - 2 + \lambda / (2\lambda + 10)}{ 2 (m / n + 1)(n + 1)^{3/4} + \mueff/2}\right) \label{eq:cmu}\\
  \cc,\ \ccd &=
               \frac{\sqrt{\mueff \cone}}{2},\ \frac{\sqrt{\mueff \coned}}{2}\enspace.\label{eq:cc}
\end{align}
% \niko{So we could also use $\cc=\sqrt{\cmu}/2$.
% It feels like there may be an overall better way to decompose the learning rate computations, using something like $\alpha(\mueff, \lambda)=\mu', \beta(m, n), \gamma(\mueff)=\mueff/2$. }

The first important change is the scaling of the learning rate with the dimension~$n$. In previous works \citep{Ros2008ppsn,Hansen:2013vt,akimoto2016gecco}, the default learning rates were set inversely proportional to $m$, i.e., $\Theta(n^{-2})$ for the original CMA and $\Theta(n^{-1})$ for the separable CMA. Our choice is based on empirical observations that $\Theta(n^{-2})$ is exceedingly conservative for higher dimensional problems, say $n > 100$:
in experiments to identify $\cmu$ for the original CMA-ES, we vary $\cmu$ and investigate the number of evaluations to reach a given target.
We find that the location of the minimum and the shape of the dependency remains for a wide range of different dimensions roughly the same when the evaluations are taken against $\cmu m / n^{0.35}$ (see also Figure~\ref{fig:cc} below).
The observation suggests in particular that $\cmu= \Theta(\sqrt{n}/m)$ is likely to fail with increasing dimension, whereas
$\cmu=\Theta(n^{0.35}/m)$ is a stable setting  such that the new setting of $\Theta(n^{0.25}/m)$ (replacing $\Theta(n^{0}/m)$) is still sufficiently conservative.
Figure~\ref{fig:lrcomp} depicts the difference between the default learning rate values in \citet{Hansen:2013vt} and the above formulas.\footnote{The learning rate for the rank-$\mu$ update is usually $\mu'$ times greater than the learning rate for the rank-one update, where $\mu' \in (\mueff - 2, \mueff - 1/2)$ is monotonous in \mueff\ and approaches $\mueff - 3/2$ for $\lambda \to \infty$. Using $\mu'$ instead of $\mueff$ is a correction for small $\mueff$. When $\mueff = 1$, we have $\mu' < 1/2$ (the first three terms cancel out). In this case, because $\mu=1$, the rank-one update contains more information than the rank-$\mu$ update as the evolution path accumulates information over time. Using $\mueff$ instead of $\mu'$ would result in the same learning rates for both updates, whereas the learning rate for the rank-$\mu$ update should be smaller in this case.}

 Another important change in the default parameter values from previous studies~\citep{Ros2008ppsn,Hansen:2013vt} is in the cumulation factor $\cc$. Previously, the typical choices were $\cc = \frac{4}{n+4}$ or $\cc = \frac{4 + \mueff/n}{n+4 + 2 \mueff/n}$.
\begin{figure}[t]\centering
  \begin{subfigure}{0.485\textwidth}%
    \includegraphics[trim={0.4cm 0.5cm 0.4cm 0.4cm},clip,width=\hsize]{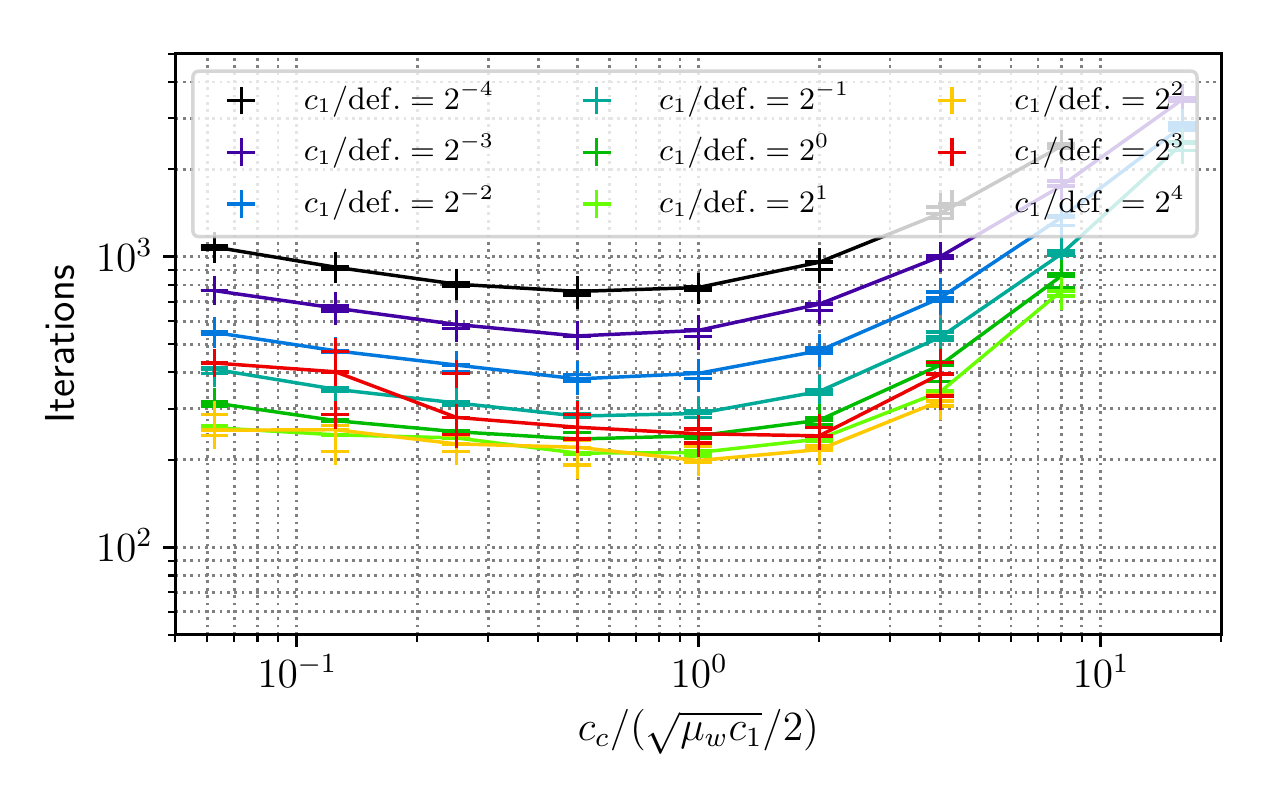}%
    \caption{CMA ($\lambda = 12$)}\label{fig:cc-cma12}%
  \end{subfigure}%
  \begin{subfigure}{0.485\textwidth}%
    \includegraphics[trim={0.4cm 0.5cm 0.4cm 0.4cm},clip,width=\hsize]{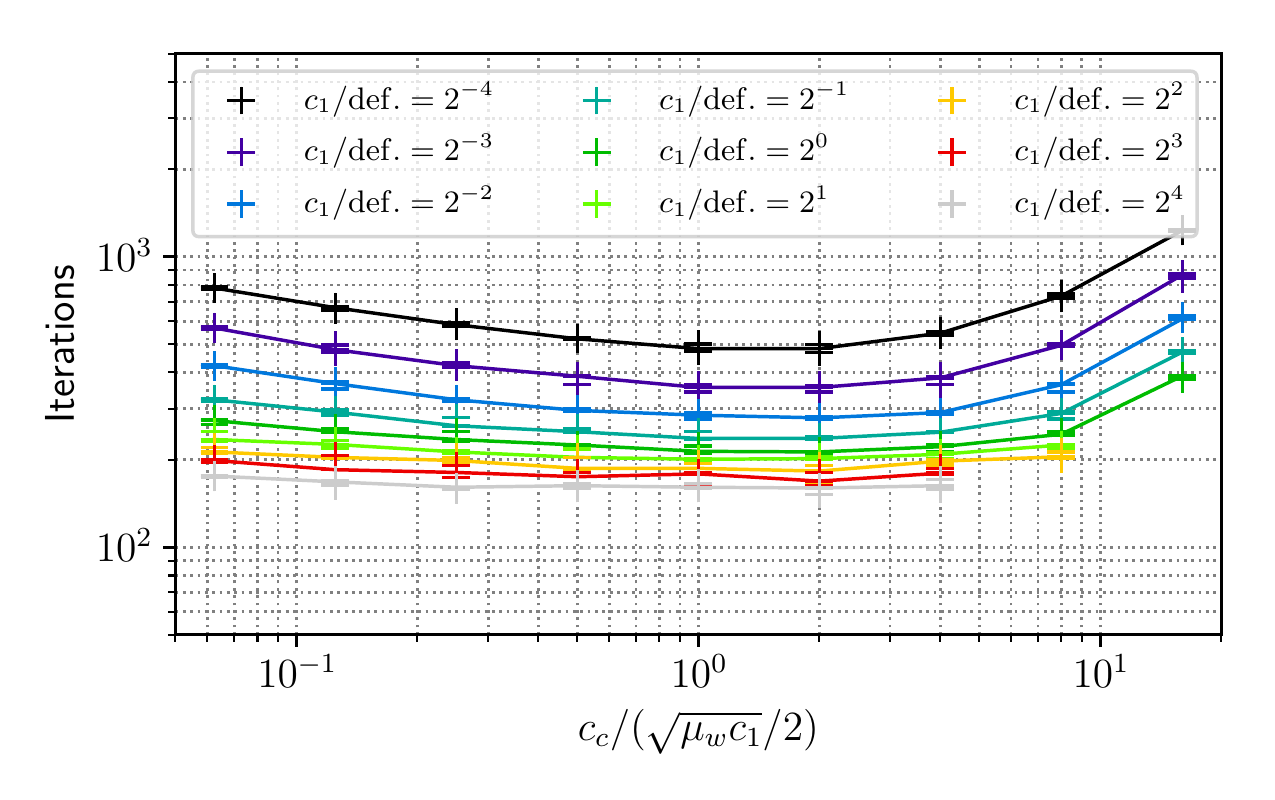}%
    \caption{CMA ($\lambda = 1200$)}\label{fig:cc-cma80}%
  \end{subfigure}%
  \\
  \begin{subfigure}{0.485\textwidth}%
    \includegraphics[trim={0.4cm 0.5cm 0.4cm 0.4cm},clip,width=\hsize]{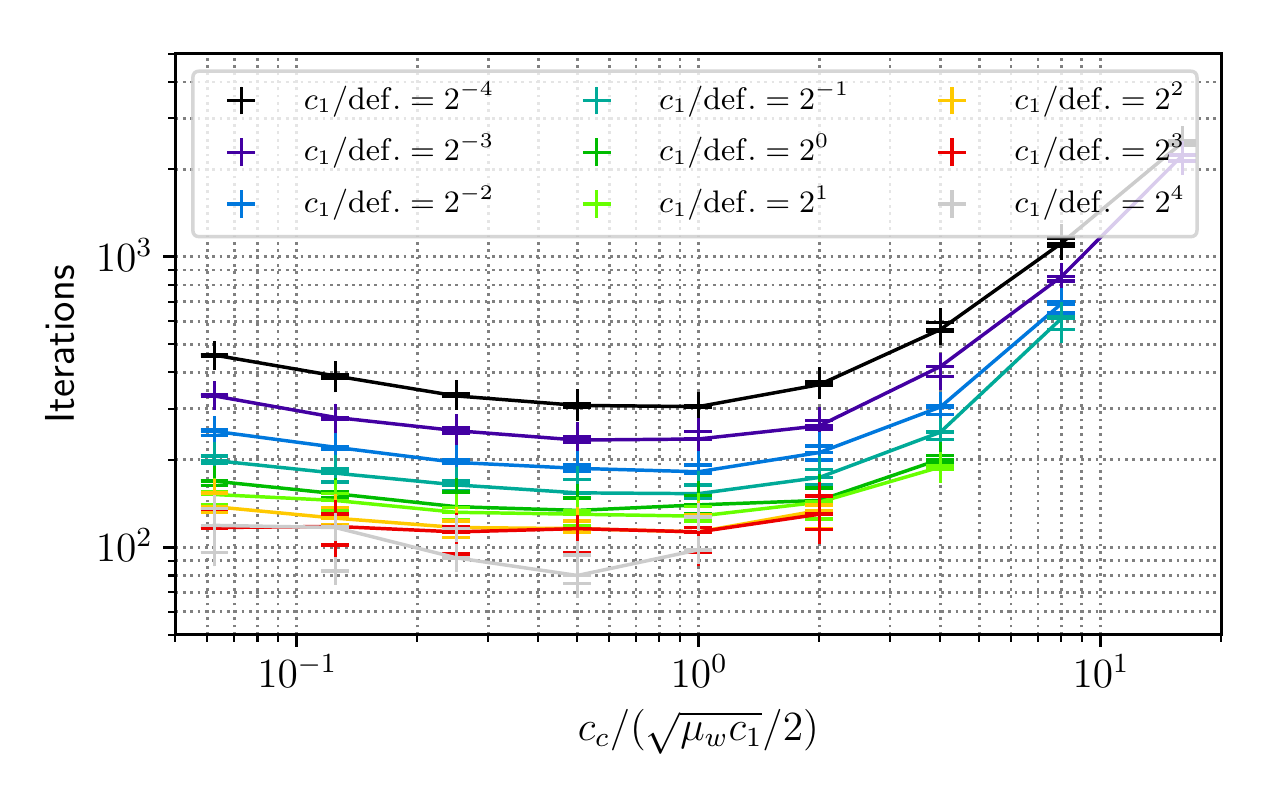}%
    \caption{SEP ($\lambda = 12$)}\label{fig:cc-sep12}%
  \end{subfigure}%
  \begin{subfigure}{0.485\textwidth}%
    \includegraphics[trim={0.4cm 0.5cm 0.4cm 0.4cm},clip,width=\hsize]{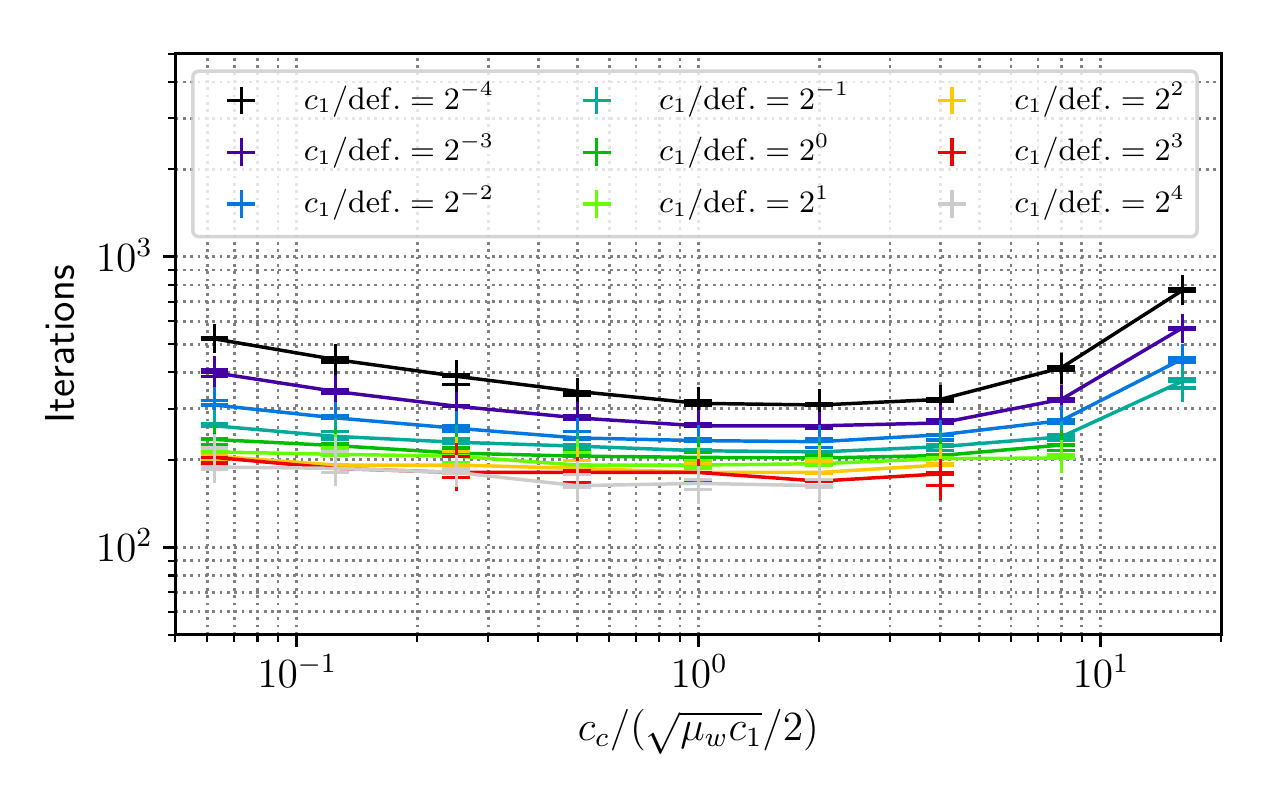}%
    \caption{SEP ($\lambda = 1200$)}\label{fig:cc-sep80}%
  \end{subfigure}%
  \caption{Function evaluations on the 20-D cigar function until the target $f$-value of $10^{-2}$ is achieved for different $\cone$ (or $\coned$) and different $\cc$ (or $\ccd$). Three independent trials have been conducted for each setting, and the median of each setting is indicated by a line. The label \emph{def} is the default value computed in \eqref{eq:cone}. Missing points indicate more than $5000$ iterations or invalid $c$-values, e.g., $\cone > 1$ or $\cc > 1$.}\label{fig:cc}
\end{figure}
The new and simpler default value for the cumulation factor\footnote{The appearance of $\sqrt{\mueff}$ in the cumulation factor is motivated as follows.
\citet{Hansen2001ec} analyzed the situation where the selection vector $\sqrt{\mueff}\sum_{i=1}^{\mu}w_i \zz_{i:\lambda}$ alternates between two vectors $\bm{v}$ and $\bm{w}$ over iterations and showed the squared Euclidean norm of the evolution path, i.e.\ the eigenvalue of $\pc\pc^\T$, remains in $O(\norm{\bm{v} + \bm{w}}^2 / \cc)$ if $\norm{\bm{v} + \bm{w}}>0$.
On the other hand, \citet{Akimoto2008gecco} showed that the above selection vector can potentially be proportional to $\sqrt{\mueff}$ when $\sigma$ is (too) small.
Altogether, we have $\cone \norm{\pc}^2 \in O(\cone \mueff / \cc)$, where $\cone\mueff \to 2$ and $\cc\to\sqrt{2}/2$ as $\mueff \to \infty$. Therefore, the eigenvalue added by the rank-one update is bounded by a constant with overwhelming probability.} is
motivated by an experiment on the Cigar function, investigating
the dependency of $\cc$ on $\cone$.
The effect of the rank-one update of the covariance matrix $\C$ is most pronounced when the covariance matrix needs to increase a single eigenvalue.
To skip over the usual initial phase where sigma decreases without changing the covariance matrix and emphasize on the effect of $\cc$, the mean vector is initialized at $\m^{(0)} = (1, 0, \dots, 0)$ and $\sigma^{(0)} = 10^{-6}$.
The rank-$\mu$ update is off, i.e., $\cmu = \cmud = 0$. We compared the number of $f$-calls until the target function value of $10^{-2}$ is reached. Note that $\sigma^{(0)}$ is chosen to avoid $\sigma$ adaptation at the beginning and the target value is chosen so that we stop the optimization once the covariance matrix learned the right scaling.
Figure~\ref{fig:cc} compares the number of $f$-calls spent by the original CMA and the separable CMA on the $20$ dimensional Cigar function.
Because the $x$-axis is normalized with the new default value for \cc,
the shape of the graph and the location of the minimum
remains roughly the same for different values of $\cone$ and $\lambda$.
The adaptation speed of the covariance matrix tends to be the best around $\cc = \sqrt{\mueff\cone}/2$ or $\sqrt{\mueff\coned}/2$.
The minimum is more pronounced for smaller $\cone$ (or $\coned$).
Nevertheless, we observe little improvement with the new setting in practice since more $f$-calls are usually spent to adapt the overall covariance matrix and even to adapt the step-size to converge. We have done the same experiments in dimension $80$ and observe the same trend.

None of the strategy parameters are meant to be tuned by users, except the population size $\lambda$. A larger population size typically results in finding a better local minimum on rugged functions such as multimodal or noisy functions \citep{Hansen2004ppsn}. It is also advantageous when the objective function values are evaluated in parallel \citep{Hansen2003ec}.
Restart strategies that perform independent restarts with increasing population size such as IPOP strategy \citep{harik1999parameter, Auger2005cecgs} or BIPOP strategy \citep{Hansen2009geccobbobbi} or NIPOP strategy \citep{loshchilov2012alternative} are handy policies that automate the parameter tuning.
They can be applied to the CMA-ES with diagonal acceleration in a straight-forward way. We leave research in this line as future work.

\section{Experiments}\label{sec:exp}

We conduct numerical simulations to see the effect of the algorithmic components of CMA-ES with diagonal acceleration, \DDCMA-ES---namely the active update with positive definiteness guarantee and the adaptive diagonal decoding. Since the effect of the active covariance matrix update with default population size has been investigated by \citet{Jastrebski2006cec} and our contribution is a positive definiteness guarantee for the covariance matrix especially for large population size, we investigate how the effect of the active update scales as the population size increases.
Our main focus is however on the effect of adaptive diagonal decoding.
Particularly, we investigate (i) how much better \DDCMA{} scales on separable functions than the CMA without diagonal decoding (plain CMA), (ii) how closely \DDCMA\ matches the performance of separable CMA on separable functions, and (iii) how the scaling of \DDCMA{} compares to the plain CMA on various non-separable functions. Moreover, we investigated how effective \DDCMA{} is when the population size is increased.

\subsection{Common Setting}
\newcommand{\commonmsig}{& $3 \cdot \bm{1}$ & $1$}
\begin{table}[t]
  \centering
  \caption{Test function definitions and initial conditions. The unit vector $\bm{u}$ is either $\bm{e}_1 = (1, 0, \dots, 0)$ for separable scenarios, or random vectors drawn uniformly on the unit sphere for non-separable scenarios or for Ell-Cig and Ell-Dis. A vector $\zz$ represents an orthogonal transformation $\mathbf{R} \xx$ of the input vector $\xx$, where the orthogonal matrix $\mathbf{R}$ is the identity matrix for separable scenarios, or is constructed by generating normal random elements and applying the Gram-Schmidt procedure. The diagonal matrix $\D_\mathrm{ell} = \diag(1, \cdots, 10^{\frac{i-1}{n-1}}, \cdots, 10)$ represents a coordinate-wise scaling and the vector $\mathbf{y}$ is a coordinate-wisely transformed input vector $\yy = \D_\mathrm{ell}^2 \xx$. }
  \label{tbl:func}
  \small
  \begin{tabular}{lccc}
  % \toprule
  & $f(\xx)$ & $\m^{(0)}$ & $\sigma^{(0)}$ \\
    \toprule
    Sphere & $\norm{\xx}^2$\commonmsig\\
    \midrule
    Cigar & $\inner{\bm{u}}{\xx}^2 + 10^6 (\norm{\xx}^2 - \inner{\bm{u}}{\xx}^2)$\commonmsig\\
    Discus& $10^6\inner{\bm{u}}{\xx}^2 + (\norm{\xx}^2 - \inner{\bm{u}}{\xx}^2)$\commonmsig\\     
    Ellipsoid & $\norm{\D_\mathrm{ell}^3 \zz}^2$\commonmsig\\
    TwoAxes & $10^{6} \sum_{i=1}^{n/2}[\zz]_{i}^2 + \sum_{i=n/2 + 1}^{n}[\zz]_{i}^2$ \commonmsig\\        
    \midrule
    Ell-Cig & $10^{-4} \inner{\bm{u}}{\yy}^2 + (\norm{\yy}^2 - \inner{\bm{u}}{\yy}^2)$ \commonmsig\\
    Ell-Dis & $10^{4} \inner{\bm{u}}{\yy}^2+ (\norm{\yy}^2 - \inner{\bm{u}}{\yy}^2)$ \commonmsig\\
    %Ell-Cig-Dis & $f(\xx) = 10^{-4} \inner{\bm{u}}{\yy}^2 + 10^{4} \inner{\bm{v}}{\yy}^2+ (\norm{\yy}^2 - \inner{\bm{u}}{\yy}^2 - \inner{\bm{v}}{\yy}^2)$\\
    \midrule
    Rosenbrock & $\sum_{i=1}^{n-1} 100 ( [\zz]_{i}^2 - [\zz]_{i+1})^2 + ([\zz]_{i} - 1)^2 $
    & $\bm{0}$ & $0.1$\\
    %Ill-Schaffer & $f(\xx) = \sum_{i=1}^{n-1} ([\zz]_{i}^2 + [\zz]_{i+1}^2)^{1/4}(1 + \sin^{2}(50 ([\zz]_{i}^2 + [\zz]_{i+1}^2)^{1/10}))$\\
    Bohachevsky & $\sum_{i=1}^{n-1} ([\zz]_{i}^2 + 2[\zz]_{i+1}^2 - 0.3\cos(3\pi[\zz]_{i}) \dots$
    & $\mathcal{N}(\bm{0}, 8^2\eye)$ & $7$ \\
    &\hfill ${} - 0.4\cos(4\pi[\zz]_{i+1}) + 0.7)$ \\    
    Rastrigin & $\sum_{i=1}^{n} [\zz]_{i}^2 + 10(1 - \cos(2\pi[\zz]_{i}))$ 
    & $\mathcal{N}(\bm{0}, 3^2\eye)$ & $2$\\ 
    \bottomrule
  \end{tabular}
\end{table}

Table~\ref{tbl:func} summarizes the test functions used in the following experiments together with the initial $\m^{(0)}$ and $\sigma^{(0)}$.
The initial covariance matrix $\C^{(0)}$ and the diagonal decoding matrix $\D^{(0)}$ are always set to the identity matrix. The random vectors and the random matrix appearing in Table~\ref{tbl:func} are initialized randomly for each problem instance, but the same values are used for different algorithms for fair comparison. A trial is considered as success if the target function value of $10^{-8}$ is reached before the algorithm spends $5\times 10^4 n$ function evaluations, otherwise regarded as failure. For $n \leq 40$ we conducted $20$ independent trials for each setting, $10$ for $80 \leq n \leq 320$, and $3$ for $n \geq 640$. When the computational effort of evaluating the test function scales worse than linear with the dimension (i.e., on rotated functions) we may omit dimensions larger than $320$.

We compare the following CMA variants:
\begin{description}
\item[Plain CMA]  updates $\C$  while $\D$ is kept the identity matrix, with or without active update described in Section~\ref{sec:active} (Algorithm~\ref{alg:ddcma} without $\D$-update);
\item[Separable CMA] updates $\D$ as described in Section~\ref{sec:add} while $\C$ is kept the identity matrix, with active update or without active update, where negative recombination weights are set to zero (Algorithm~\ref{alg:ddcma} without $\C$-update);
\item[\DDCMA{}] as summarized in Algorithm~\ref{alg:ddcma},  updates $\C$ as described in Section~\ref{sec:cma} with active update described in Section~\ref{sec:add}, and updates $\D$ as described in Section~\ref{sec:add}.
\end{description}
All strategy parameters such as the learning rates are set to their default value presented in Section~\ref{sec:param}. Note that the separable CMA is different from the original publication \citep{Ros2008ppsn} in that the matrix is updated in multiplicative form which however barely affects the performance.

\subsection{Active CMA}

First, the effect of the active update is investigated. The plain and separable CMA with and without active update are compared. 

\begin{figure}[t]\centering
  \begin{subfigure}{0.33\textwidth}%
    \includegraphics[trim={0.4cm 0.5cm 0.4cm 0.4cm},clip,width=\hsize]{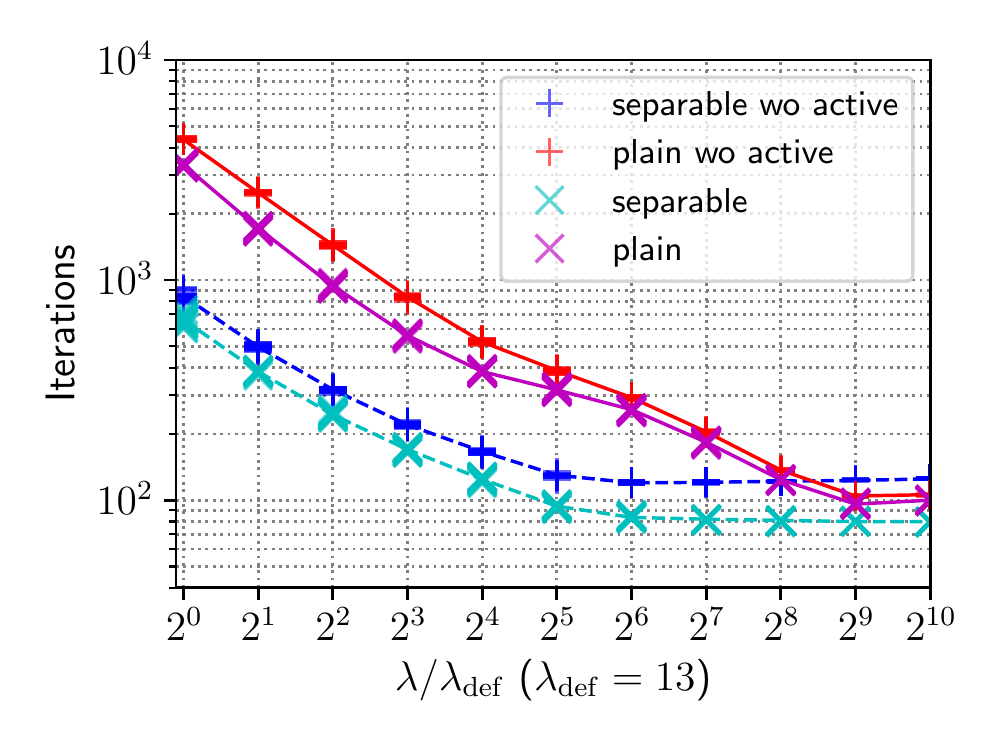}%
    \caption{Ellipsoid}\label{fig:active-ell}%
  \end{subfigure}%
  \begin{subfigure}{0.33\textwidth}%
    \includegraphics[trim={0.4cm 0.5cm 0.4cm 0.4cm},clip,width=\hsize]{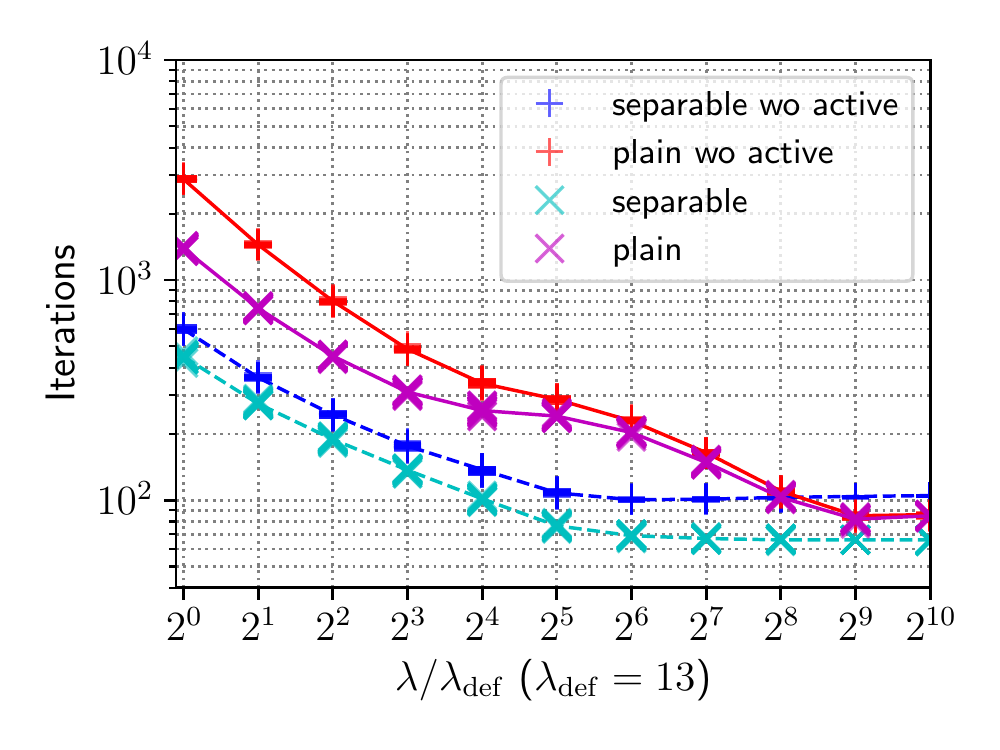}%
    \caption{Discus}\label{fig:active-dis}%
  \end{subfigure}%
  \begin{subfigure}{0.33\textwidth}%
    \includegraphics[trim={0.4cm 0.5cm 0.4cm 0.4cm},clip,width=\hsize]{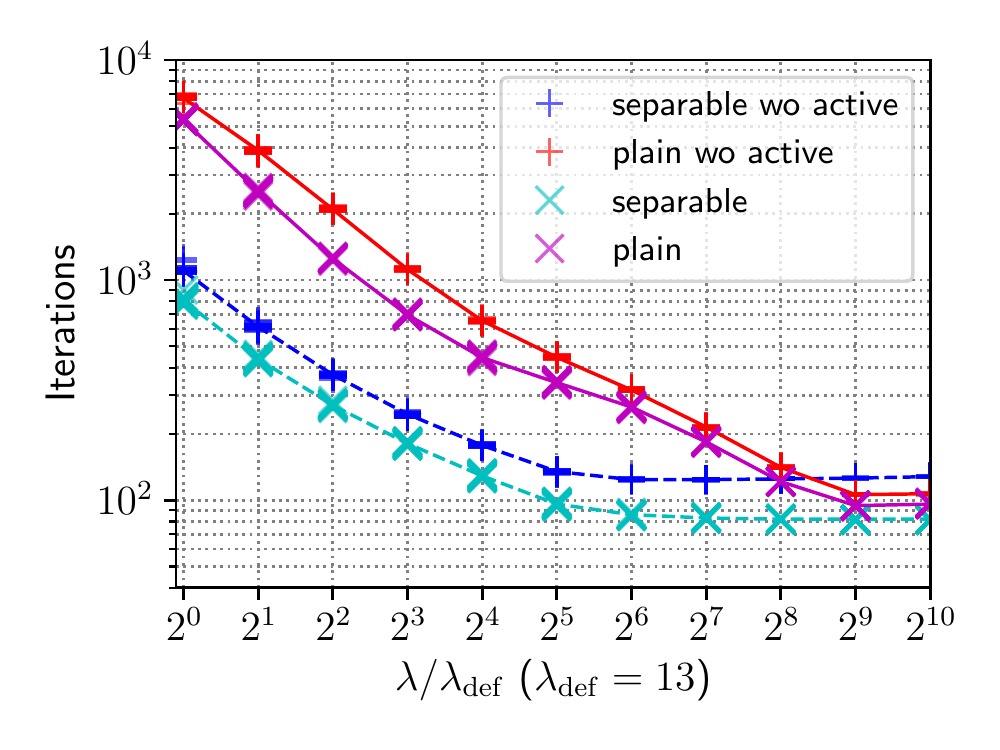}%
    \caption{TwoAxes}\label{fig:active-two}%
  \end{subfigure}%  
  \caption{Number of iterations spent by plain and separable CMA with and without active update to reach the target $f$-value of $10^{-8}$ on the 40 dimensional separable Ellipsoid, Discus, and TwoAxes functions. Each line indicates the median of each setting (solid: plain CMA, dashed: separable CMA).
  }\label{fig:active}
\end{figure}

Figure~\ref{fig:active} shows the number of iterations spent by each algorithm plotted against population size $\lambda$.
The number of iterations to reach the target function value decreases as $\lambda$ increases and tends to level out. 
The plain CMA is consistently faster with active update than without.
As expected from the results of \citet{Jastrebski2006cec}, the effect of active covariance matrix update is most pronounced on functions with a small number of sensitive variables such as the Discus function.
%\del{For $\lambda \in \Omega(n^{1.7})$ the sum of the negative weights drops to zero as indicated in Figure~\ref{fig:alpha}, and the effect of the active update disappears.}{}%
The advantage of the active update diminishes as $\lambda$ increases in plain CMA, whereas the speed-up in separable CMA becomes even slightly more pronounced.

\subsection{Adaptive Diagonal Decoding}

Next, we compare \DDCMA{} with the plain and the separable CMA with active update.

\begin{figure}[t]\centering
  \begin{subfigure}{0.25\textwidth}\centering%
    \includegraphics[trim={0.4cm 0.5cm 0.4cm 0.4cm},clip,width=\hsize]{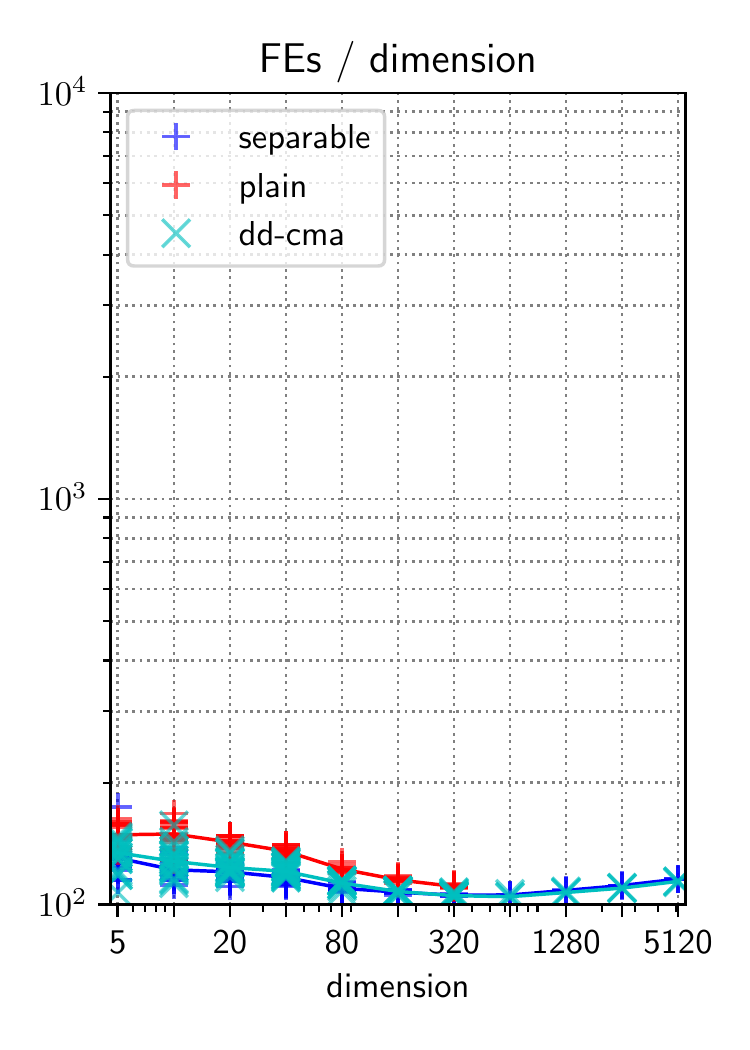}%
    \caption{Sphere}\label{fig:fes-sphere}%
  \end{subfigure}%
  \begin{subfigure}{0.25\textwidth}\centering%
    \includegraphics[trim={0.4cm 0.5cm 0.4cm 0.4cm},clip,width=\hsize]{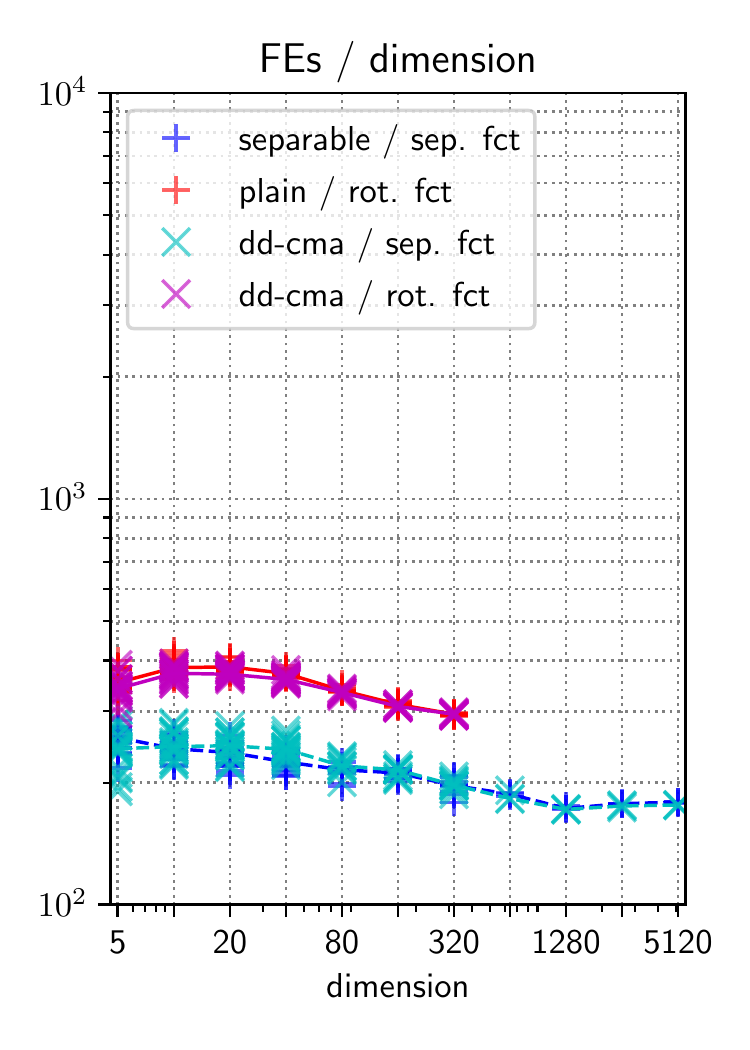}%
    \caption{Cigar}\label{fig:fes-cigar}%
  \end{subfigure}%
  %\\
  \begin{subfigure}{0.25\textwidth}\centering%
    \includegraphics[trim={0.4cm 0.5cm 0.4cm 0.4cm},clip,width=\hsize]{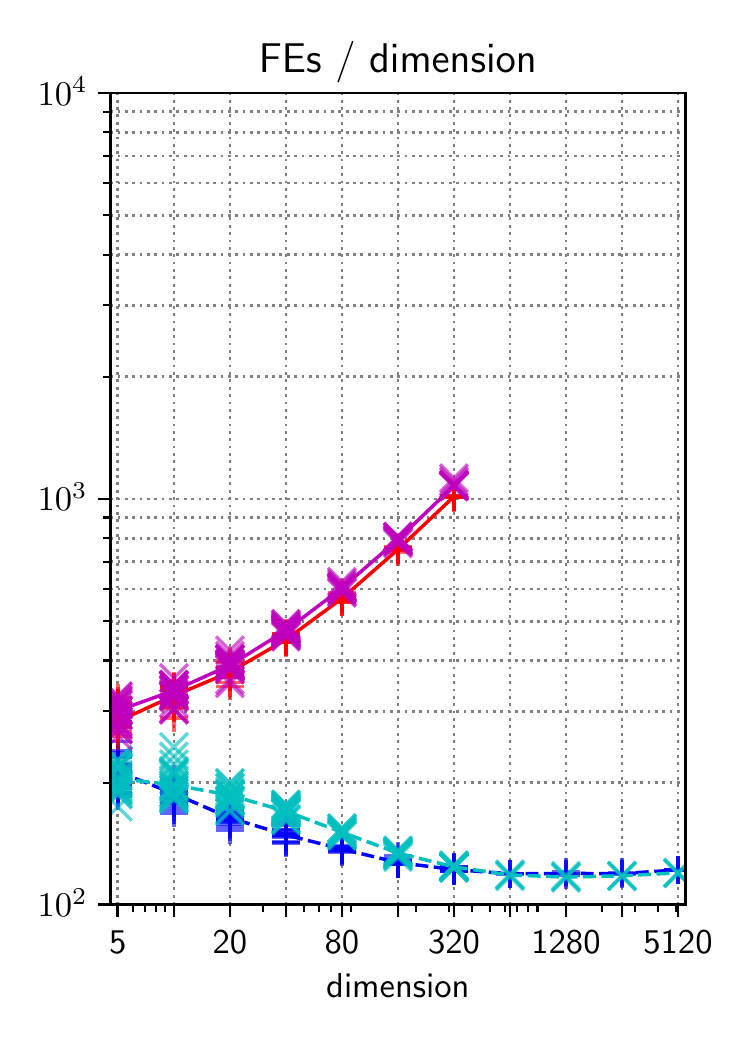}%
    \caption{Discus}\label{fig:fes-discus}%
  \end{subfigure}%
  \begin{subfigure}{0.25\textwidth}\centering%
    \includegraphics[trim={0.4cm 0.5cm 0.4cm 0.4cm},clip,width=\hsize]{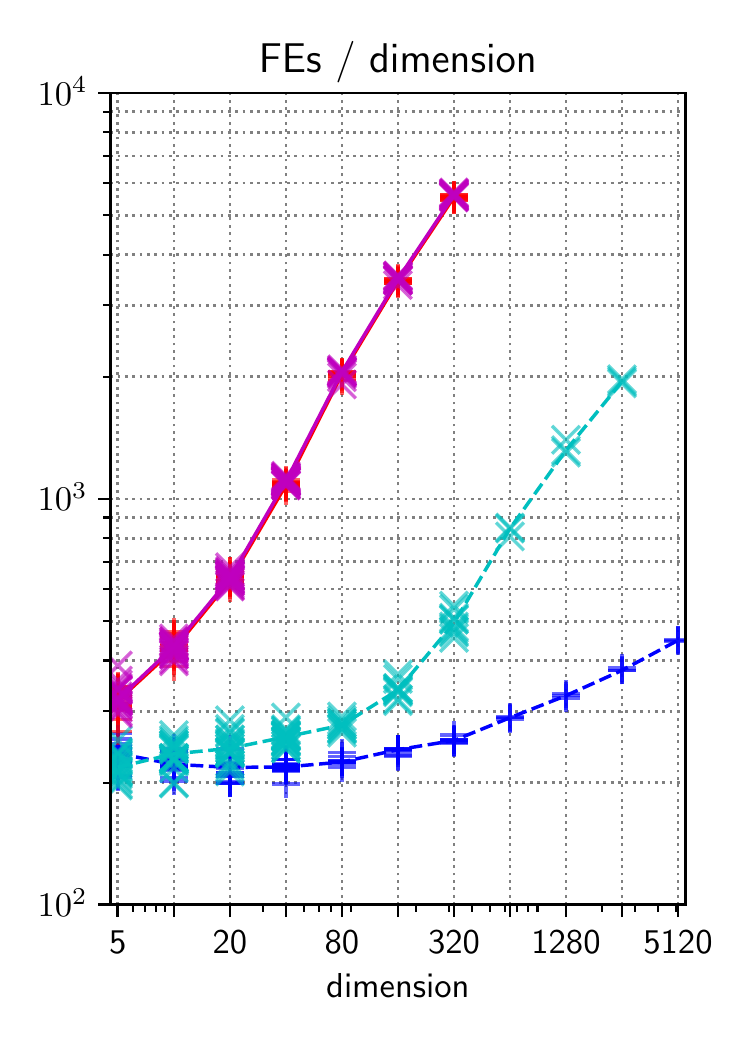}%
    \caption{Ellipsoid}\label{fig:fes-ellipsoid}%
  \end{subfigure}%
  \caption{Function evaluations (divided by $n$) spent by each algorithm until it reaches the target $f$-value of $10^{-8}$ and their median values (solid line: rotated function; dashed line: non-rotated, separable function). Data for \DDCMA{} on the separable Ellipsoid in dimension 5120 are missing since they did not finish in a limited CPU time.
    %\yohe{data point will be added later.}\niko{Data are also missing on Sphere, Cigar, and Discus (i.e.\ they do not comply with the given experimental setup, I refined the setup explanation above)}
 }\label{fig:fes}
\end{figure}

Figure~\ref{fig:fes} compares \DDCMA{} with the plain CMA on the rotated Cigar, Ellipsoid, and Discus functions, and with the separable CMA on the separable Cigar, Ellipsoid, and Discus functions. Note that the plain CMA scales equally well on separable functions and rotated functions, while the separable CMA will not find the target function value on rotated functions before the maximum budget is exhausted \citep{Ros2008ppsn}. Displayed are the number of function evaluations to reach the target function value on each problem instance. The results of \DD, plain and separable CMA on the Sphere function are displayed for reference.

On the Sphere function, no covariance matrix adaptation is required, and all algorithms perform quite similarly. On the Cigar function, the rank-one covariance matrix update is known to be rather effective, and even the plain CMA scales linearly on the rotated Cigar function. On both, separable and rotated Cigar functions, \DDCMA{} is competitive with separable and plain CMA thereby combining the better performance of the two.

% \begin{figure}[t]\centering
%   \begin{subfigure}{\textwidth}%
%     \includegraphics[width=\hsize]{dissep320.pdf}%
%     \caption{separable Discus}\label{fig:dissep}%
%   \end{subfigure}%
%   \\
%   \begin{subfigure}{\textwidth}%
%     \includegraphics[width=\hsize]{ellsep320.pdf}%
%     \caption{separable Ellipsoid}\label{fig:ellsep}%
%   \end{subfigure}%
%   \\
%   \begin{subfigure}{\textwidth}%
%     \includegraphics[width=\hsize]{ellrot320.pdf}%
%     \caption{rotated Ellipsoid}\label{fig:ellrot}%
%   \end{subfigure}%
%   \caption{Typical runs of the acelerated CMA-ES on three 320 dimensional test problems (x-axis: function evaluations). To understand the algorithmic behavior the normalization of $\C$ in \eqref{eq:dc1} and \eqref{eq:dc2} is not performed.}\label{fig:conv}
% \end{figure}

The discrepancy between plain and separable CMA is much more pronounced on Ellipsoid and Discus functions, where the plain CMA scales super-linearly, whereas the separable CMA exhibits linear or slightly worse than linear scaling on separable instances and fails to find the optimum within the given budget on rotated instances.
On the rotated functions, \DDCMA{} is competitive with the plain CMA, whereas it significantly reduces the number of function evaluations on the separable functions compared to plain CMA, e.g., ten times less $f$-calls are required on the $160$ dimensional separable Ellipsoid function.

The \DDCMA{} is competitive with separable CMA on the separable Discus function up to $5120$ dimension, which is the ideal situation since we can not expect faster adaptation of the distribution shape on separable functions.
On the separable Ellipsoid function, the performance curve of \DDCMA{} starts deviating from that of the separable CMA around $n = 320$, and scales more or less the same as the plain CMA afterwards, yet it requires ten times less $f$-calls.
To improve the scaling of \DDCMA{} on the separable Ellipsoid function, we might need to set $\betathresh$ depending on $n$, or use another monotone transformation of the condition number of the correlation matrix in \eqref{eq:ddamp}. Performing the eigen decomposition of $\C$ less frequently (i.e., setting $\beta_\mathrm{eig}$ smaller) is another possible way to improve the performance of \DDCMA{} on the separable Ellipsoid function, while compromising the performance on the rotated Ellipsoid function\footnote{If we set $\beta_\text{eig} = 10$ instead of $\beta_\text{eig} = 10n$, \DDCMA{} spends about $1.5$ times more FEs on the $5120$-dimensional separable Ellipsoid function than separable CMA as displayed in Figure~\ref{fig:fes-ellipsoid}, whereas it spends about $15\%$ more FEs on the $320$-dimensional rotated Ellipsoid function than plain CMA as displayed in Figure~\ref{fig:fes-ellipsoid}. This might be a more practical choice, but further investigation is required.}.
%\niko{It is probably enough to do this only in the beginning. One could also reduce the \C-learning rate in the beginning. It could even depend on the change observed in \D.}

% \del{In separable Ellipsoid, the acelerated CMA tended to loose the $f$-calls. Figure~\ref{fig:conv} shows typical run results of the acelerated CMA. The covariance matrix became away from the diagonal matrix in the middle of adaptation of the covariance matrix even though the inverse Hessian of the objective function is diagonal and the covariance matrix ended up very close to the diagonal matrix. Then, the damping factor $beta$ decreased and the diagonal decoding matrix adaptation slowed down. A smaller $d_\beta$ resulted in a better scaling on the separable Ellipsoid, at the price of slower adaptation of the covariance matrix on the rotated Discus function, where the condition number of $\D$ became about $10$ at the beginning and stopped changing due to a large $\beta$, and $\C$ needed to learn its inverse transformation and ended up $\C \approx \D^{-1} \mathrm{Hess}^{-1}\D^{-1}$. }{}%

\begin{figure}[t]\centering
  \begin{subfigure}{0.33\textwidth}%
    \includegraphics[trim={0.4cm 0.5cm 0.4cm 0.4cm},clip,width=\hsize]{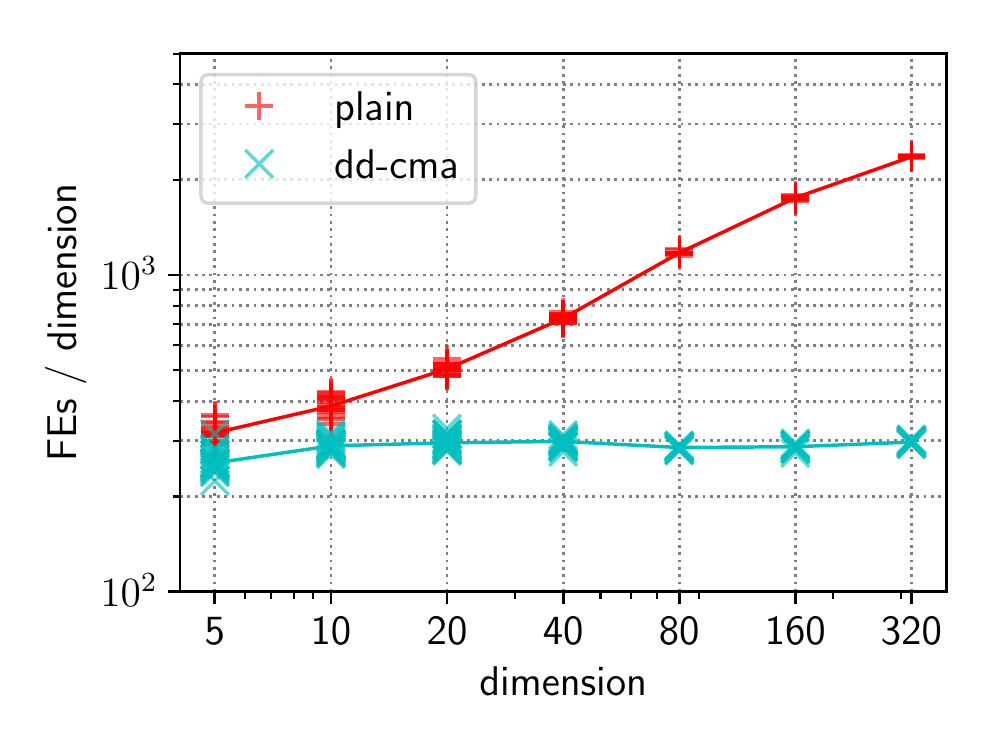}%
    \caption{Ell-Cig}\label{fig:fes-ellcig}%
  \end{subfigure}%
  \begin{subfigure}{0.33\textwidth}%
    \includegraphics[trim={0.4cm 0.5cm 0.4cm 0.4cm},clip,width=\hsize]{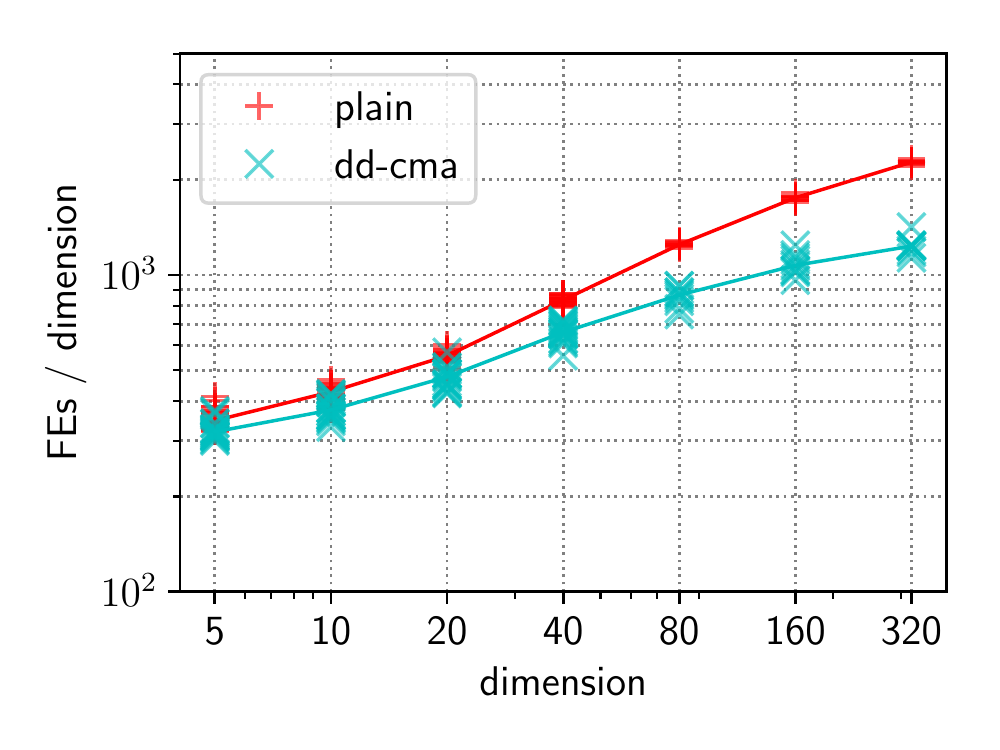}%
    \caption{Ell-Dis}\label{fig:fes-elldis}%
  \end{subfigure}%
  \begin{subfigure}{0.33\textwidth}%
    \includegraphics[trim={0.4cm 0.5cm 0.4cm 0.4cm},clip,width=\hsize]{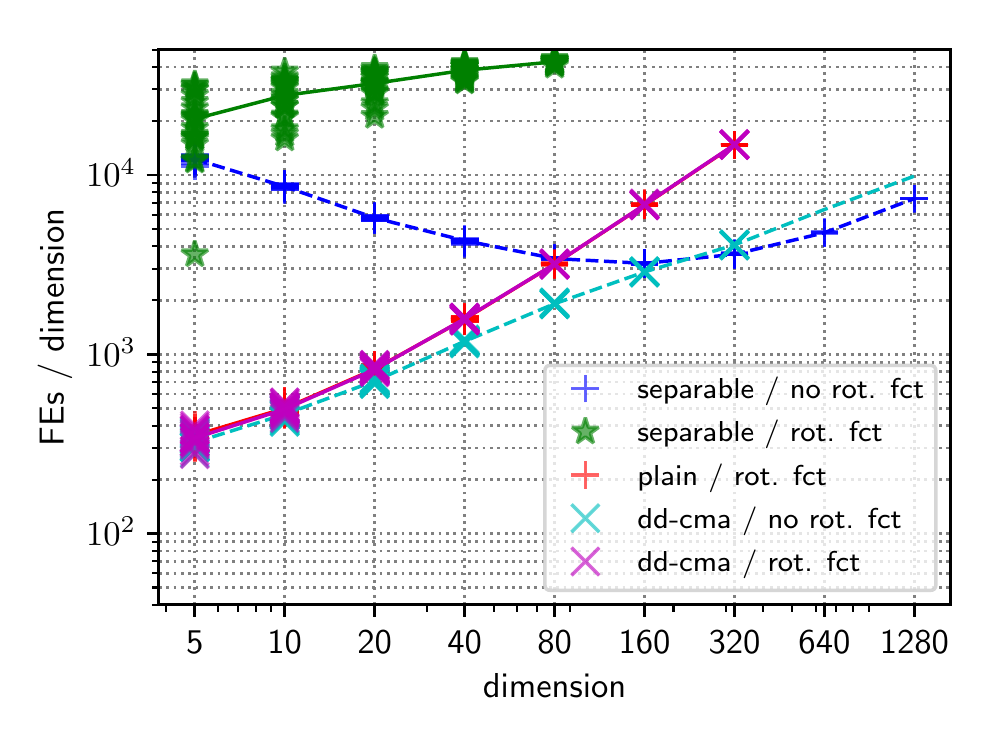}%
    \caption{Rosenbrock}\label{fig:fes-ros}%
  \end{subfigure}%
  \caption{%
  Examples where \DDCMA{} not only performs on par with the better of
       plain and separable CMA, but outperforms both at least in some dimension up to $n=320$ (up to $n = 1280$ only for separable CMA on the non-rotated Rosenbrock function).
       Function evaluations to reach the target $f$-value of $10^{-8}$ divided by $n$. Each line indicates the median of each setting (solid: rotated function, dashed: non-rotated function).
       The line for \DDCMA{} is extrapolated up to $n=1280$ with the least square log-log-linear regression on the median values.
}\label{fig:fes2}
\end{figure}

Figures~\ref{fig:fes-ellcig} and \ref{fig:fes-elldis} show the results on Ell-Cig and Ell-Dis functions, which are non-separable ill-conditioned functions with additional coordinate-wise scaling, Figure~\ref{fig:fes-ros} shows the results on non-rotated and rotated Rosenbrock functions.
The separable CMA locates the target $f$-value within the given budget only on the non-rotated Rosenbrock function in smaller dimension.
The experiment on the Ell-Cig function reveals that diagonal decoding can improve the scaling of the plain CMA \emph{even on non-separable functions}.
The improvement on the Ell-Dis function is much less pronounced.
The reason why, compared to the plain CMA, \DDCMA{} is more suitable for Ell-Cig than for Ell-Dis is discussed in relation with Figure~\ref{fig:conv} below.
On the non-rotated Rosenbrock function, \DDCMA{} becomes advantageous as the dimension increases, just as the separable CMA is faster than the plain CMA when $n>80$.

\begin{figure}[t]\centering
  \begin{subfigure}{\textwidth}%
    \includegraphics[width=\hsize]{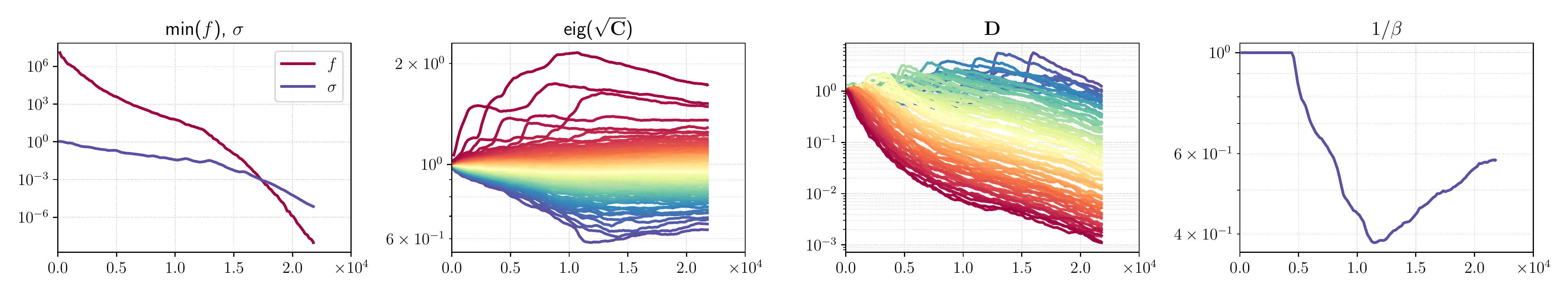}%
    \caption{separable Ellipsoid}\label{fig:sepell}%
  \end{subfigure}%
  \\
  \begin{subfigure}{\textwidth}%
    \includegraphics[width=\hsize]{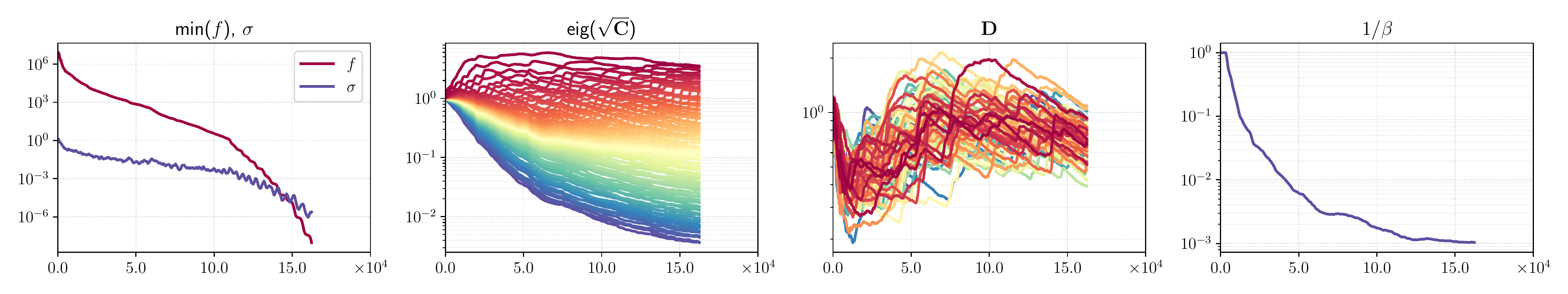}%
    \\
    \includegraphics[width=\hsize]{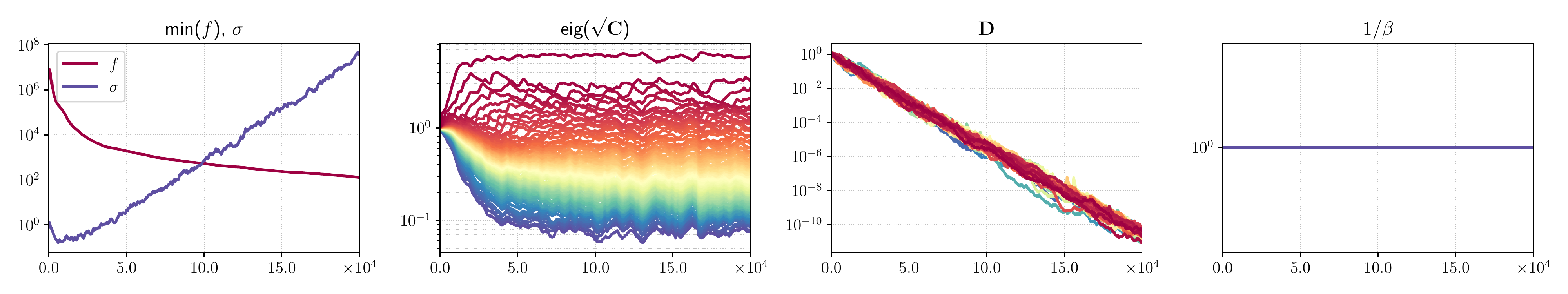}%    
    \caption{rotated Ellipsoid, below with fixed $\beta=1$
    instead of \eqref{eq:ddamp}}\label{fig:rotell}%
  \end{subfigure}%
  \\
  \begin{subfigure}{\textwidth}%
    \includegraphics[width=\hsize]{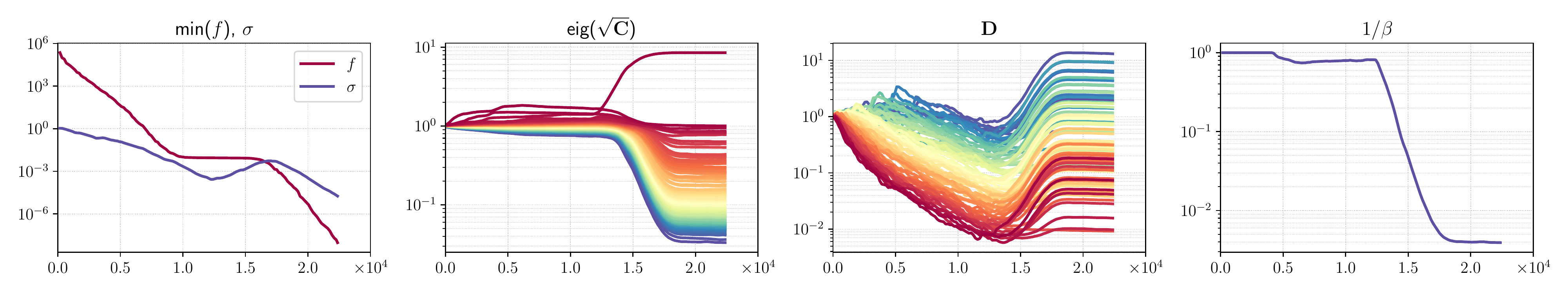}%
    \caption{Ellipsoid-Cigar}\label{fig:ellcig}%
  \end{subfigure}%
  \\
  \begin{subfigure}{\textwidth}%
    \includegraphics[width=\hsize]{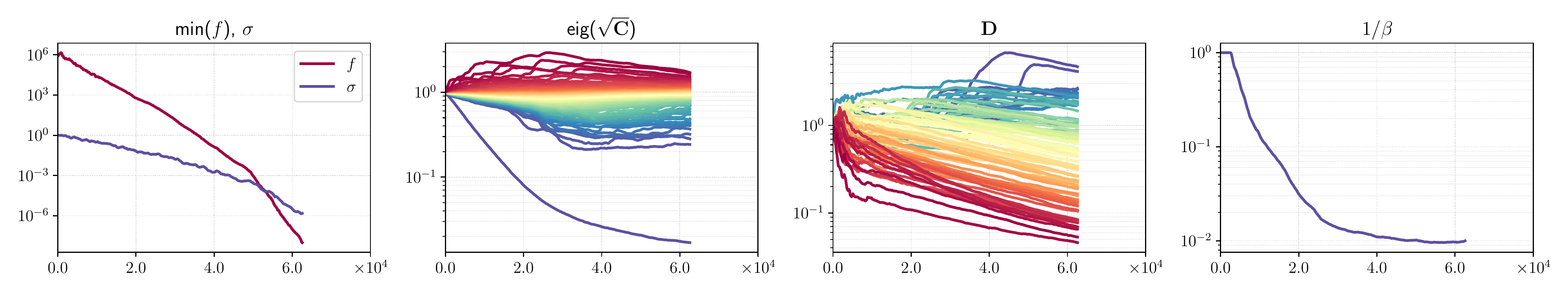}%
    \caption{Ellipsoid-Discus}\label{fig:elldis}%
  \end{subfigure}%
  % \\
  % \begin{subfigure}{\textwidth}%
  %   \includegraphics[width=\hsize]{sepell80bb.pdf}%
  %   \caption{separable Ellipsoid}\label{fig:rotellbb}%
  % \end{subfigure}%
  %\\
  % \begin{subfigure}{\textwidth}%
    
  %   \caption{rotated Ellipsoid, without dynamic damping factor \eqref{eq:ddamp}}\label{fig:rotellbb}%
  % \end{subfigure}%
  \caption{Typical runs of \DDCMA{}-ES on 80 dimensional test problems (x-axis: function evaluations). Note that the change of $\D$ partially comes from the update of $\C$.}\label{fig:conv}
\end{figure}

Figure~\ref{fig:conv} visualizes insights in the typical behaviour of \DDCMA{} on different functions. On the separable Ellipsoid, the correlation matrix deviates from the identity matrix due to stochastics and the initial parameter setting, and the inverse damping decreases marginally.
This effect is more pronounced for $n > 320$ and is the reason for the impaired scaling of \DDCMA{} on the separable Ellipsoid in Figure~\ref{fig:fes-ellipsoid}.
The diagonal decoding matrix, i.e., variance matrix, adapts the coordinate-wise scaling efficiently as long as the condition number of the correlation matrix $\C$ is not too high.
On the rotated Ellipsoid function, where the diagonal elements of the inverse Hessian of the objective function, $\frac12 \mathbf{R}^\T\D_\mathrm{ell}^{-6}\mathbf{R}$, likely have similar values,
the diagonal decoding matrix remains close to the identity and the correlation matrix learns the ill-conditioning.
The inverse damping parameter decreases so that the adaptation of the diagonal decoding matrix does not disturb the adaptation of the distribution shape. Figure~\ref{fig:rotell}, below, shows that adaptive diagonal decoding without the dynamic damping factor \eqref{eq:ddamp} severely disturbs the adaptation of $\C$ on the rotated Ellipsoid.

Ideally, the diagonal decoding matrix $\D$ adapts the coordinate-wise standard deviation and $\C$ adapts the correlation matrix of the sampling distribution.
If the correlation matrix $\C$ is easier to adapt than the full covariance matrix $\D\C\D$, we expect a speed-up of the adaptation of the distribution shape. The functions Ell-Cig and Ell-Dis in Figures~\ref{fig:ellcig} and \ref{fig:elldis} are such examples. Once $\D$ becomes inversely proportional to $\D_\text{ell}^2$, the problems become rotated Cigar and rotated Discus functions, respectively.
The run on the Ell-Cig function in \ref{fig:ellcig} depicts the ideal case. The coordinate-wise scaling is first adapted by $\D$ and then the correlation matrix learns a long Cigar axis. The inverse damping factor is kept at a high value similar to the one observed on the separable Ellipsoid and it starts decreasing only after $\D$ is adapted. 
On the Ell-Dis function (\ref{fig:elldis}) however,
the short non-coordinate axis is learned (too) quickly by the correlation matrix.
Therefore, the inverse damping factor decreases before $\D$ has adapted the full coordinate-wise scaling. Since the function value is more sensitive in the subspace corresponding to great eigenvalues of the Hessian of the objective and the ranking of candidate solutions is mostly determined by this subspace, the CMA-ES first shrinks the distribution in this subspace. This is the main reason why \DDCMA{} is more efficient on Ell-Cig than on Ell-Dis.

\begin{figure}[t]\centering
  \begin{subfigure}{0.33\textwidth}%
    \includegraphics[trim={0.4cm 0.5cm 0.4cm 0.4cm},clip,width=\hsize]{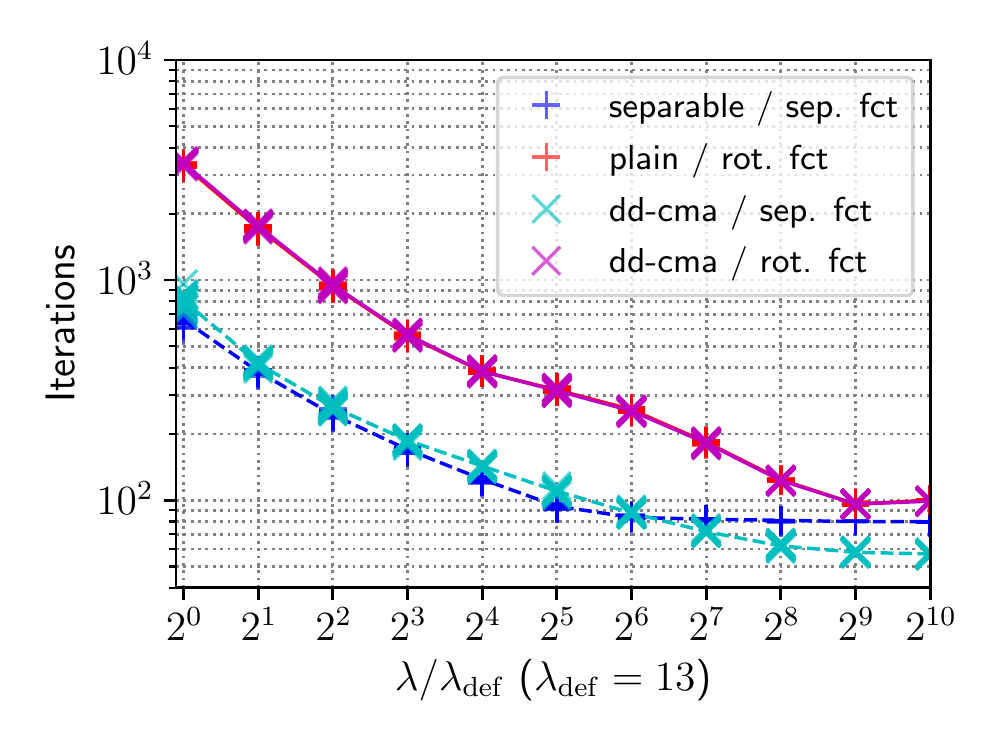}%
    \caption{Ellipsoid}\label{fig:lam-ell}%
  \end{subfigure}%
  \begin{subfigure}{0.33\textwidth}%
    \includegraphics[trim={0.4cm 0.5cm 0.4cm 0.4cm},clip,width=\hsize]{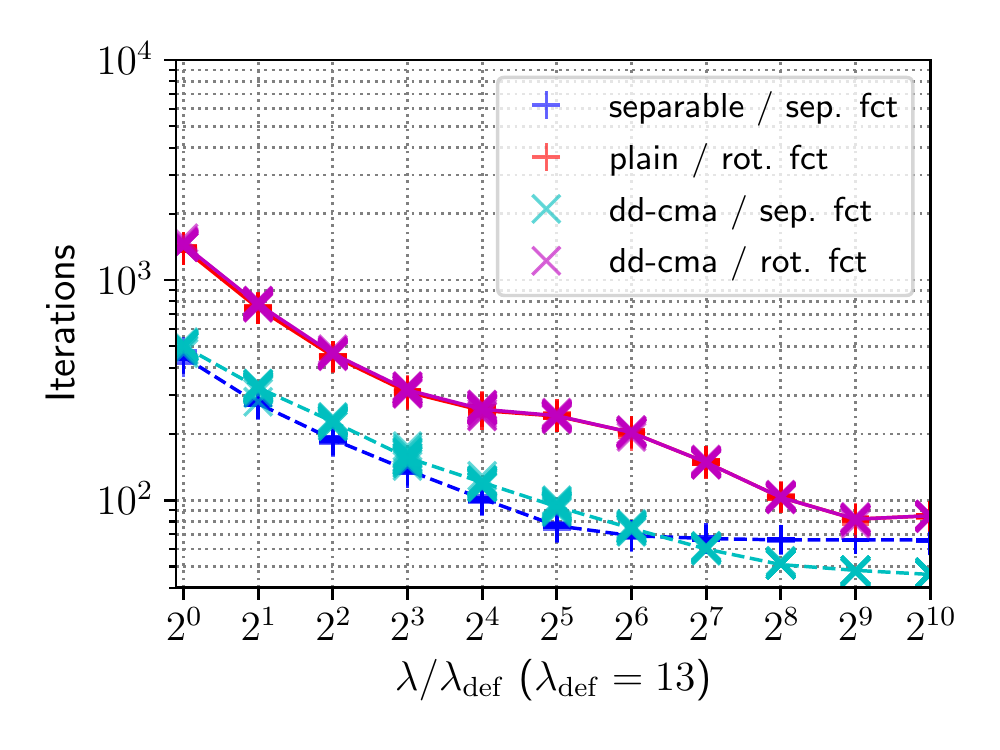}%
    \caption{Discus}\label{fig:lam-dis}%
  \end{subfigure}%
  %\\
  \begin{subfigure}{0.33\textwidth}%
    \includegraphics[trim={0.4cm 0.5cm 0.4cm 0.4cm},clip,width=\hsize]{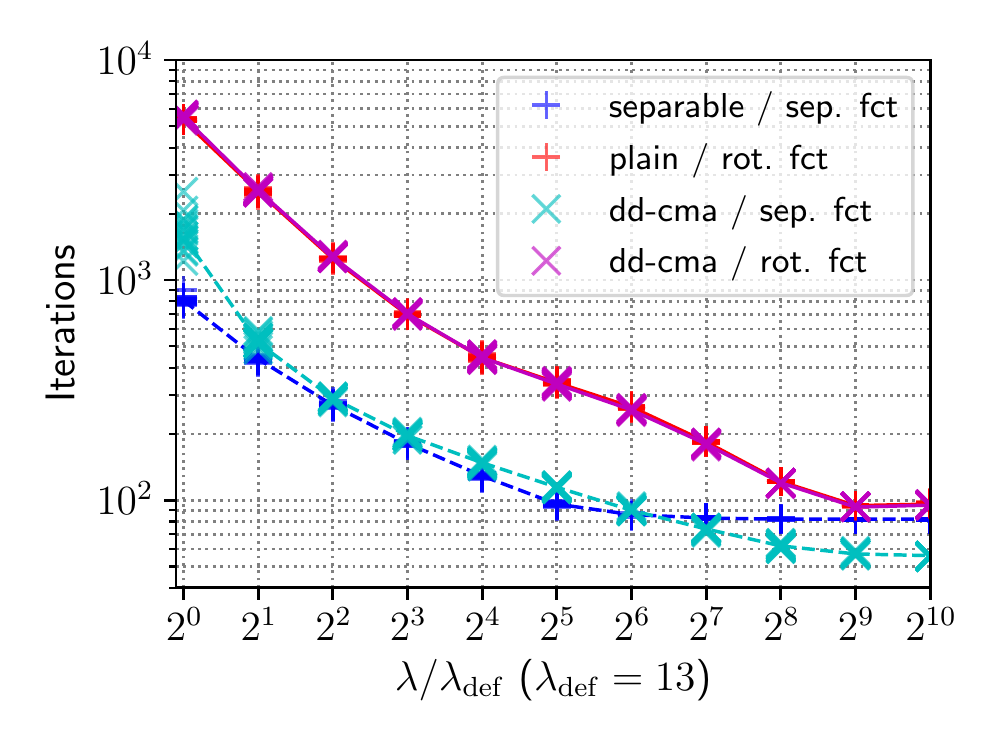}%
    \caption{TwoAxes}\label{fig:lam-two}%
  \end{subfigure}%
  \caption{Number of iterations to reach the target $f$-value of $10^{-8}$ and its median (solid line: rotated function, dashed line: non-rotated function) versus population size on the 40 dimensional Ellipsoid, Discus, and TwoAxes functions.}
\label{fig:lam}
\end{figure}

Figure~\ref{fig:lam} shows, similar to Figure~\ref{fig:active}, the number of iterations versus the population size $\lambda$ on three functions in dimension 40.
For all larger populations sizes up to $13312$, \DDCMA{} performs on par with the better of plain and separable CMA, as ideally to be expected.
%\del{The number of iterations levels out with increasing population size on all functions, on the separable functions with acelerated or separable CMA already for smaller population sizes.}{}%
%\del{Only for separable CMA the number of iterations increases for larger population sizes---a defect introduced with the normalization factor for \eye\ in \eqref{eq:dup} which is not present in the plain separable CMA.}{}

We remark that \DDCMA{} with default $\lambda=13$ has a relatively large variance in performance on the separable TwoAxes function.
Increasing $\lambda$ from its default value reduces this variance and even reduces the number of $f$-calls to reach the target.
This defect of \DDCMA{} is even more pronounced for higher dimensions. It may suggest to increase the default $\lambda$ for \DDCMA{}. We leave this to be addressed in future work.

\providecommand{\figlamprobras}{0.45}
\begin{figure}[t]\centering
  \begin{subfigure}{\figlamprobras\textwidth}%
  \centering\sf\small Bohachevsky\\
    \includegraphics[trim={0.4cm 0.5cm 0.4cm 0.4cm},clip,width=\hsize]{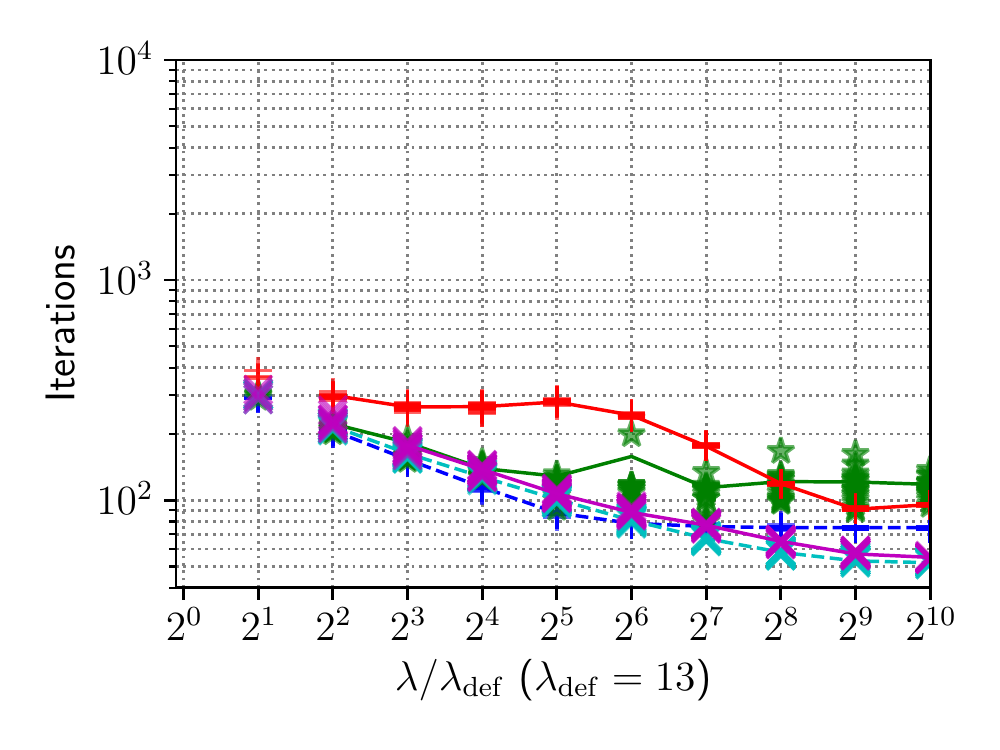}%
    %\caption{Iterations}%
    \label{fig:lam-boh}%
  \end{subfigure}%
  \qquad%
  \begin{subfigure}{\figlamprobras\textwidth}%
  \centering\sf\small Rastrigin\\
    \includegraphics[trim={0.4cm 0.5cm 0.4cm 0.4cm},clip,width=\hsize]{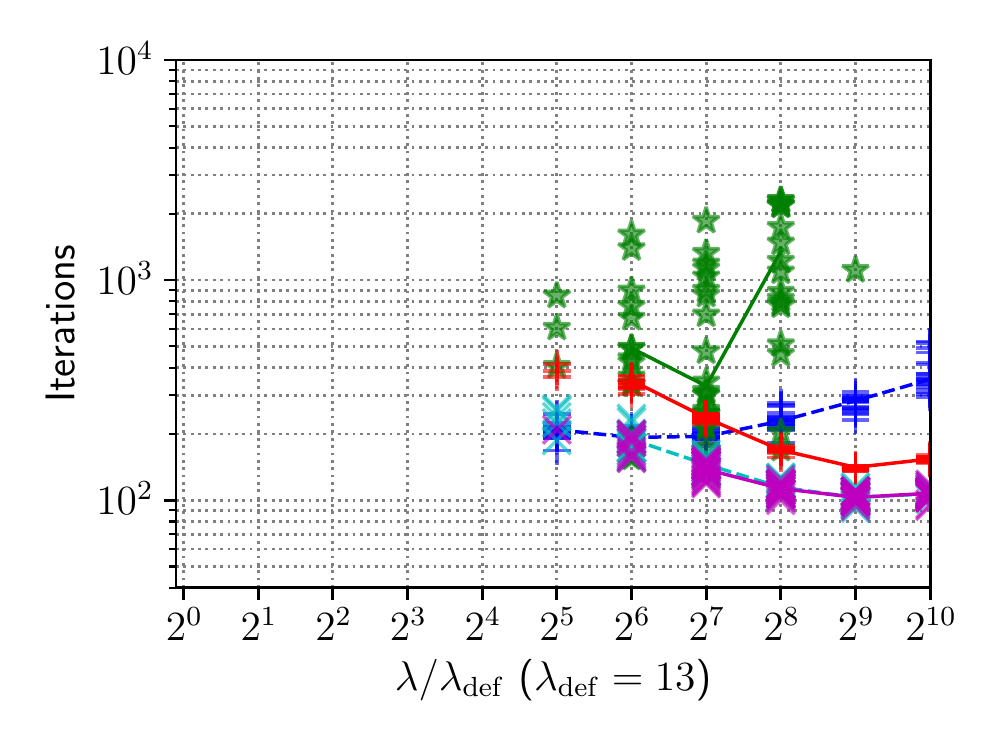}%
    %\caption{Iterations}
    \label{fig:lam-ras}%
  \end{subfigure}%
  \\%
  \begin{subfigure}{\figlamprobras\textwidth}%
    \includegraphics[trim={0.4cm 0.5cm 0.4cm 0.4cm},clip,width=\hsize]{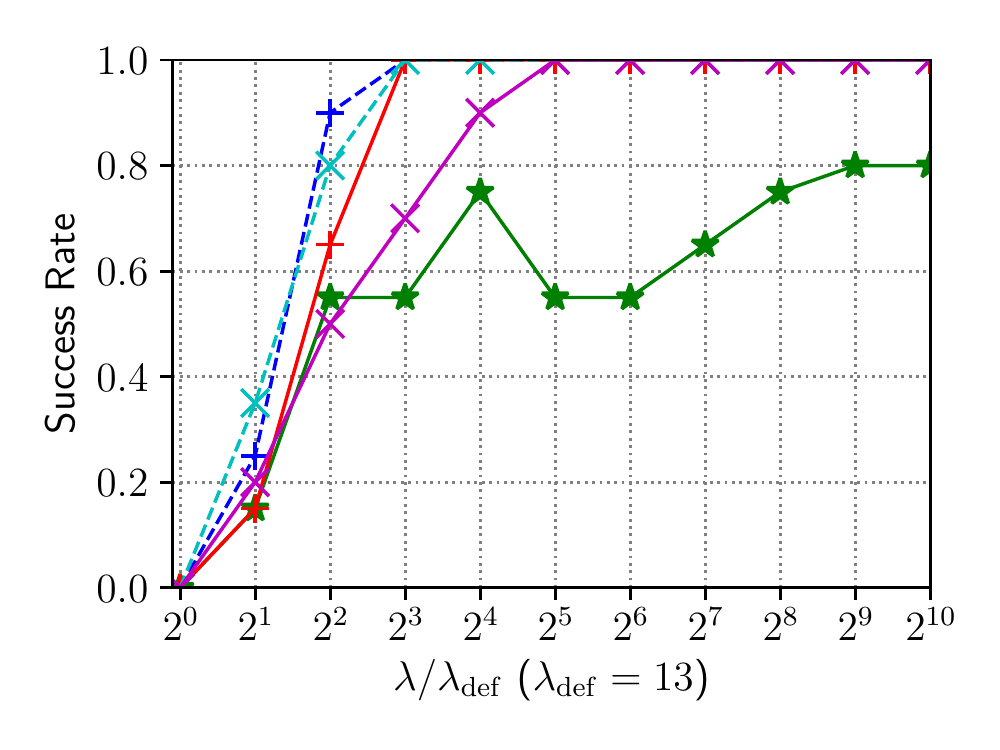}%
    %\caption{Success Rates}
    \label{fig:lamprob-boh}%
  \end{subfigure}%
  \qquad%  
  \begin{subfigure}{\figlamprobras\textwidth}%
    \includegraphics[trim={0.4cm 0.5cm 0.4cm 0.4cm},clip,width=\hsize]{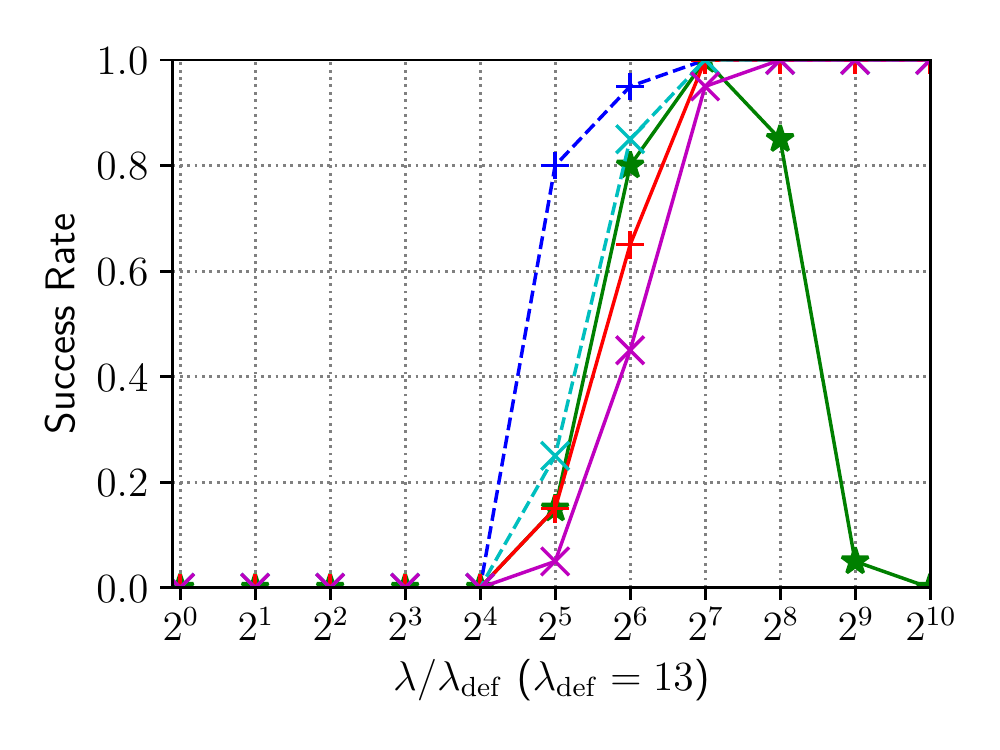}%
    %\caption{Success Rates}
    \label{fig:lamprob-ras}%
  \end{subfigure}%
  \\%
  \begin{subfigure}{\figlamprobras\textwidth}%
    \includegraphics[trim={0.4cm 0.5cm 0.4cm 0.4cm},clip,width=\hsize]{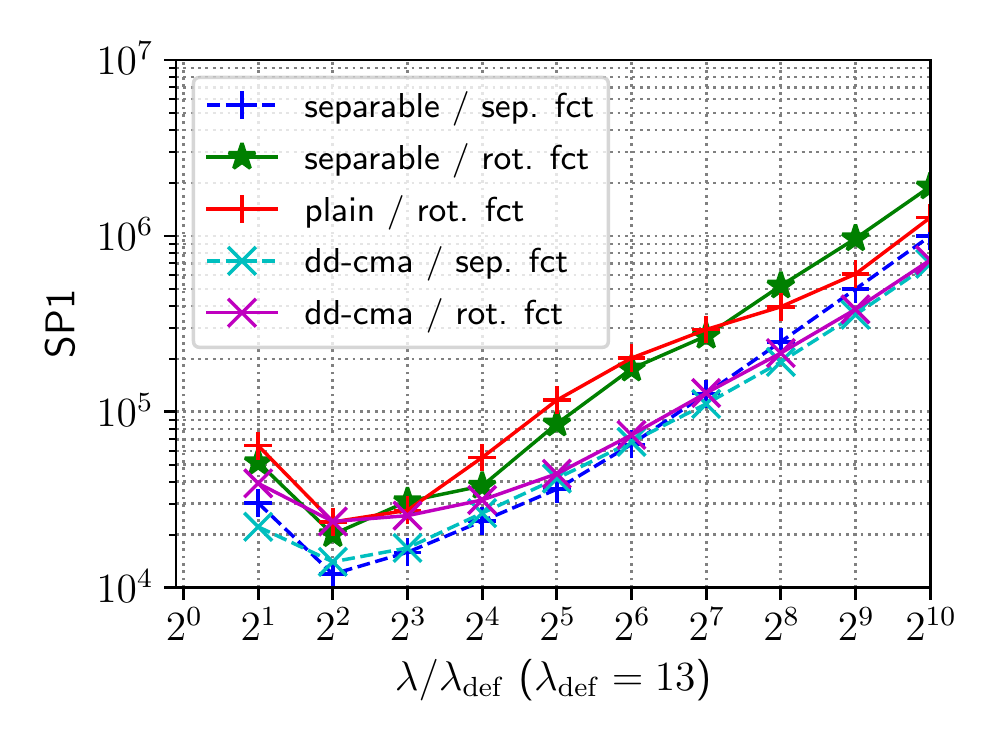}%
    %\caption{Success Rates}
    \label{fig:lamsp1-boh}%
  \end{subfigure}%
  \qquad%  
  \begin{subfigure}{\figlamprobras\textwidth}%
    \includegraphics[trim={0.4cm 0.5cm 0.4cm 0.4cm},clip,width=\hsize]{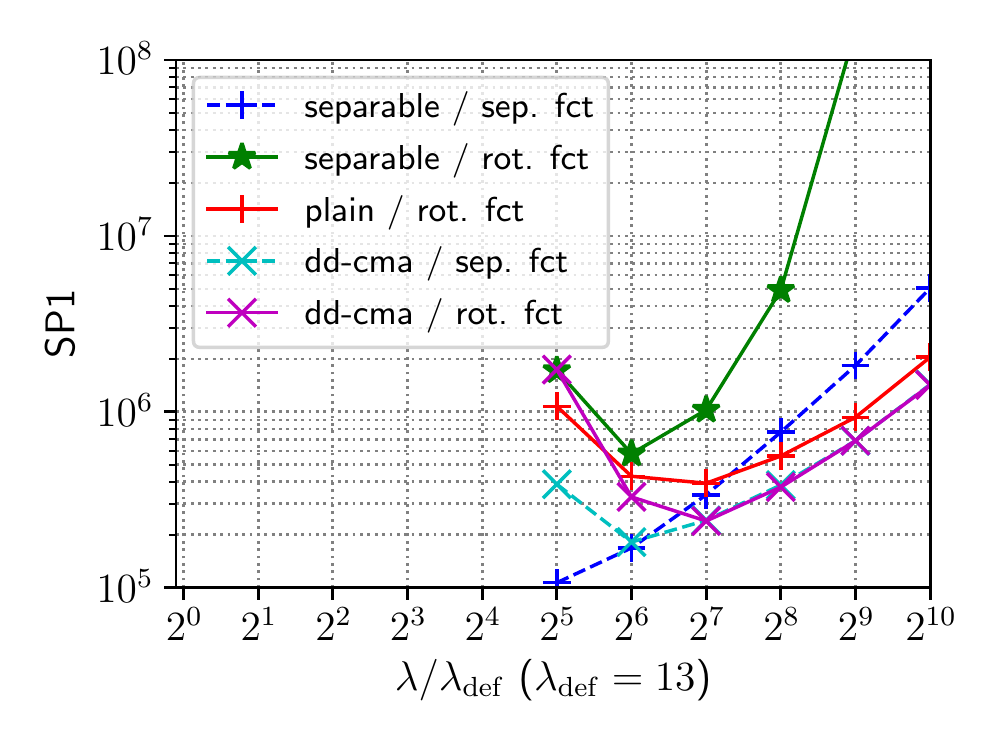}%
    %\caption{Success Rates}
    \label{fig:lamsp1-ras}%
  \end{subfigure}%
  \caption{Number of iterations (top) spent by each algorithm until it reaches the target $f$-value of $10^{-8}$ on the 40 dimensional Bohachevsky and Rastrigin functions and its median (line), success rates (middle), and the average number of evaluations in successful runs divided by the success rate (SP1, bottom).
}\label{fig:rastrigin}
\end{figure}

Figure~\ref{fig:rastrigin} shows the number of iterations in successful runs, the success rate and the average number of evaluations in successful runs divided by the success rate \citep{price1997differential, auger2005performance} on two multimodal $40$-dimensional functions.
The maximum numbers of $f$-calls are $5n \times 10^4$ on Bohachevsky and $2n \times 10^5$ on Rastrigin.
Comparing separable to \DDCMA{} on separable functions and plain to \DDCMA{} on rotated functions, \DDCMA{} tends to spend less iterations but to require greater $\lambda$ to reach the same success rate. The latter is attributed to doubly shrinking the overall variance by updating $\C$ and $\D$. For $\lambda \gg n$, we have $\cmu \approx 1$ and $\cmud \approx 1$ and the effect of shrinking the overall variance due to $\C$ and $\D$ updates is not negligible. Then, the overall variance is smaller than in plain and separable CMA, which also leads to faster convergence.
This effect disappears if we force the $\D$ update to keep its determinant unchanged, which can be realized by replacing $\Delta_D$ in \eqref{eq:dup} with $\Delta_D - (\Tr(\Delta_D)/n) \eye$.

\section{Summary and Conclusion}\label{sec:conc}

In this paper we have put forward several techniques to improve multi-recombinative non-elitistic covariance matrix adaptation evolution strategies without changing their internal computation effort.

% \del{The first component is two-point step-size adaption (TPA) as an alternative to cumulative step-size adaptation (CSA).
% A novel parameter setting for TPA is proposed.
% We have empirically highlighted differences between TPA and CSA.
% TPA can realize faster step-size changes in particular in larger dimension and with the new parameter setting.
% On functions with insensitive axes, which we expect to be more frequent when the number of variables is large,
% CSA reveals a known defect.
% On the highly asymmetrical but otherwise simple sector sphere test function, 
% we found settings where TPA (in combination with CMA) fails to exhibit linear convergence.
% Previous results on multimodal functions reveal that in some cases CSA is superior and in some cases TPA, most likely due to a smaller step-size in TPA.
% Overall, in our experience, TPA and CSA are the most reliable algorithms for step-size adaptation in CMA-ES with intermediate or weighted multi-recombination.
% In larger dimension, TPA may generally be preferable.}{}%

The first component is concerned with the active covariance matrix update, which utilizes unsuccessful candidate solutions to actively shrink the sampling distribution in unpromising directions. We propose two ways to guarantee the positive definiteness of the covariance matrix by rescaling unsuccessful candidate solutions.

The second and main component is a diagonal acceleration of CMA by adaptive diagonal decoding, \DDCMA.
The covariance matrix adaptation is accelerated by adapting a coordinate-wise scaling separately from the positive definite symmetric covariance matrix.
This drastically accelerates the adaptation speed on high-dimensional functions with coordinate-wise ill-conditioning, whereas it does not significantly influence the adaptation of the covariance matrix on highly ill-conditioned non-separable functions without coordinate-wise scaling component.
%\del{An advantage with diagonal decoding over the plain CMA is observed not only in terms of the number of function evaluations, but also in terms of the success rate on multimodal functions with $\lambda \in O(n^{1.75})$.}{}

The last component is a set of improved default parameters for CMA.
% \del{\footnote{%
% \del{Further default parameters which may be targeted for future improvements are the cumulation factor \cs\ in CSA and the population size $\lambda$. The former could change from $\Theta(n^{-1})$ to $\Theta(n^{-1/2})$, the latter from $4 + 3\log(n)$ to $1 + 4\sqrt{n}$.
% }}}{}
The scaling of the learning rates are relaxed from $\Theta(1/n^2)$ to $\Theta(1/n^{1.75})$ for default CMA and from $\Theta(1/n)$ to $\Theta(1/n^{0.75})$ for separable CMA. This contributes to accelerate the learning in all investigated CMA variants on high dimensional problems, say $n \geq 100$.

Algorithm selection as well as hyper-parameter tuning of an algorithm is a troublesome issue in black-box optimization. For CMA-ES, we needed to make a decision whether we use the plain CMA or the separable CMA, based on the limited knowledge on a problem of interest. If we select the separable CMA but the objective function is non-separable and highly ill-conditioned, we will not obtain a reasonable solution within most given function evaluation budgets. Therefore, plain CMA is a safer choice, though it may be slower on nearly separable functions.
The \DDCMA{} automizes this decision and achieves faster adaptation speed on functions with ill-conditioned variables usually without compromising the performance on non-separable and highly ill-conditioned functions.
Moreover, the advantage of \DDCMA{} is not limited to separable problems.
We found a class of non-separable problems with variables of different sensitivity on which \DDCMA-ES decidedly outperforms CMA-ES \emph{and} separable CMA-ES (see Figure~\ref{fig:fes-ellcig}).
As \DDCMA{}-ES improves the performance on a relatively wide range of problems with mis-scaling of variables, it is a prime candidate to become the default CMA-ES variant.

\section{Acknowledgements}
  The authors thank the associate editor and the anonymous reviewers for their valuable comments and suggestions.
  This work is partly supported by JSPS KAKENHI Grant Number 19H04179.

{\small
\bibliographystyle{apalike}
\bibliography{paper}
%\bibliography{paper,../../../../alabgit_trunk/reference/ref,../../../../alabgit_trunk/publication/ref}
}

\end{document}